\documentclass[11pt]{article}
\usepackage{mathrsfs,amsmath,amsfonts,amssymb,bm,bbm,eufrak,dsfont,pifont,amscd,
stmaryrd,euscript,amsthm,appendix,color,epsfig,xr,yfonts,tikz,verbatim}

\usepackage[round]{natbib}
\usepackage{graphicx,color,authblk}
\usepackage{amsfonts}
\usepackage{bbm}
\usepackage{amsmath}
\usepackage{tikz}
\usepackage{hyperref}
\newtheorem{theorem}{Theorem}

\newtheorem{corollary}[theorem]{Corollary}
\newtheorem{proposition}[theorem]{Proposition}
\theoremstyle{remark}
\newtheorem{remark}{Remark}
\theoremstyle{example}

\newtheorem{appxthm}{Theorem}[section]
\newtheorem{appxlem}[appxthm]{Lemma}
\newtheorem{appxpro}[appxthm]{Proposition}
\newtheorem{appxcor}[appxthm]{Corollary}
\theoremstyle{definition}

\theoremstyle{remark}

\newcommand{\Ck}{\mathfrak{C}}
\newcommand{\Ch}{\widehat{\mathfrak{C}}}

\newcommand{\R}{\mathbb{R}}

\newcommand{\X}{\mathcal{X}}

\newcommand{\E}{\mathbb{E}}

\newcommand{\Y}{\mathcal{Y}}

\newcommand{\Hk}{\mathcal{H}}

\newcommand{\Pb}{\mathbb{P}}

\newcommand{\bb}{\mathbb}
\newcommand{\eu}{\EuScript}

\newcommand{\Cal}{\mathcal}
\newcommand{\A}{\mathfrak{A}}
\newcommand{\sig}{\Sigma}
\newcommand{\sigh}{\widehat{\Sigma}}
\newcommand{\lambdah}{\widehat{\lambda}}
\newcommand{\phih}{\widehat{\phi}}
\newcommand{\mh}{\widehat{m}}
\newcommand{\id}{\mathfrak{I}}
\newcommand{\ide}{\mathfrak{I}}
\newcommand{\lp}{L^2(\mathbb{P})}
\newcommand{\oh}{\otimes_{\mathcal{H}}}
\newcommand{\ohm}{\otimes_{\mathcal{H}_m}}
\newcommand{\ol}{\otimes_{L^2(\mathbb{P})}}

\newcommand{\var}{\text{Var}}
\newcommand{\intx}{\int_{\mathcal{X}}}
\newcommand{\intt}{\int_{\Theta}}
\newcommand{\op}{\mathcal{L}^\infty}
\newcommand{\HS}{\mathcal{L}^2}
\newcommand{\Tr}{\mathcal{L}^1}
\newcommand{\tS}{\mathfrak{A}}
\newcommand{\Mp}{m_{\mathbb{P}}}
\newcommand{\mph}{\widehat{m}_{\mathbb{P}}}
\newcommand{\mpm}{m_{\mathbb{P},m}}
\newcommand{\mpmh}{\widehat{m}_{\mathbb{P},m}}
\newcommand{\OPH}{{\mathcal{L}^\infty(\mathcal{H})}}
\newcommand{\OH}{{\mathcal{L}^\infty(H)}}
\newcommand{\HSh}{\mathcal{L}^2(\mathcal{H})}
\newcommand{\HSH}{\mathcal{L}^2(\mathcal{H})}
\newcommand{\HSM}{{\mathcal{L}^2(\mathcal{H}_m)}}
\newcommand{\HSl}{\mathcal{L}^2(L^2(\mathbb{P}))}
\newcommand{\HSL}{{\mathcal{L}^2(L^2(\mathbb{P}))}}

\newcommand{\TCH}{{\mathcal{L}^1(\mathcal{H})}}
\newcommand{\OPL}{{\mathcal{L}^\infty(L^2(\mathbb{P}))}}
\newcommand{\OPHm}{{\mathcal{L}^\infty(\mathcal{H}_m)}}

\newcommand{\Rm}{\mathbb{R}^m}
\newcommand{\QEDA}{\hfill\ensuremath{\blacksquare}}
\newcommand{\vp}{\varphi}
\newcommand{\inner}[2]{\left\langle #1,#2\right\rangle}
\newcommand{\norm}[1]{\left\lVert#1\right\rVert}
\newcommand{\kbar}{\overline{k}(\cdot,X)}
\newcommand{\Pg}{P_{>}}
\newcommand{\PgC}{P_{>,\mathcal{C}}}

\newcommand{\Pl}{P_{\le}}
\newcommand{\PlB}{P_{\le,\mathcal{B}}}
\newcommand{\PlBc}{P_{\le,\mathcal{B}^c}}
\newcommand{\Ag}{A_{>}}
\newcommand{\AgC}{A_{>,\mathcal{C}}}
\newcommand{\AgCc}{A_{>,\mathcal{C}^c}}
\newcommand{\Al}{A_{\le}}
\newcommand{\AlB}{A_{\le,\mathcal{B}}}
\newcommand{\AlBc}{A_{\le,\mathcal{B}^c}}
\newcommand{\hsh}{{\mathcal{L}^2(H)}}
\newcommand*\circled[1]{\tikz[baseline=(char.base)]{
            \node[shape=circle,draw,inner sep=2pt] (char) {#1};}}

\title{Approximate Kernel PCA Using Random Features: Computational vs.~Statistical Trade-off}

\author{Bharath Sriperumbudur}
\author{Nicholas Sterge}
\affil{Department of Statistics,  
Pennsylvania State University\\
University Park, PA 16802, USA.\\
\texttt{bks18@psu.edu, nsterge@us.flowtraders.com}}

\date{}
\begin{document}

\maketitle

%
%
%
%

\begin{abstract}
Kernel methods are powerful learning methodologies that allow to perform non-linear data analysis.
Despite their popularity, they suffer from 
poor scalability in big data scenarios. Various approximation methods, including random feature approximation, have been proposed to alleviate the problem. However, the statistical
consistency of most of these approximate kernel methods is not well understood except for kernel ridge regression wherein it has been shown that the random feature approximation is not only
computationally efficient but also statistically consistent with a minimax optimal rate of convergence. In this paper, we investigate the efficacy of random feature approximation in the
context of kernel principal component analysis (KPCA) by studying the trade-off between computational and statistical behaviors of approximate KPCA. We show that the approximate KPCA is both computationally and statistically efficient compared to KPCA in terms of the error associated with reconstructing a kernel function based on its projection onto the corresponding eigenspaces. The analysis hinges on Bernstein-type inequalities for the operator and Hilbert-Schmidt norms of a self-adjoint Hilbert-Schmidt operator-valued U-statistics, which are of independent interest. 
%

\end{abstract}
\setlength{\parskip}{4pt}

\section{Introduction}\label{Sec:Introduction}
Let $X$ be a random variable distributed according to a probability measure $\Pb$ defined on a measurable space $\X$. Principal component analysis (PCA) \citep{Jollife-86} deals with finding a direction $\bm{a}\in\X\,(= \R^d)$
with $\Vert \bm{a}\Vert_2=1$ such that $\text{Var} [\bm{a}^\top X]$ is maximized. More generally, it provides a low-dimensional representation that retains as much variance as possible of $X$ and is used as a popular statistical methodology for dimensionality reduction and feature extraction.
In fact, the low-dimensional representation is the orthogonal projection of $X$ onto the $\ell$-eigenspace, i.e., the span of eigenvectors associated with top 
$\ell$ eigenvalues of the covariance matrix $\bb{E}XX^\top-\bb{E}X\bb{E}X^\top$ where $\ell<d$, resulting in a $\ell$-dimensional representation. A non-linear generalization of PCA (called kernel PCA) was proposed by \citet{Scholkopf-98} which solves $\sup\{\text{Var}[f(X)]:\Vert f\Vert_\Hk=1\}$, where $\Hk$ is a reproducing kernel Hilbert space (RKHS) \citep{Aronszajn-50}, with reproducing kernel $k:\X\times\X\rightarrow\R$ (see 
Section~\ref{Sec:notation}
for definition). Similar to linear PCA, the solution turns out to be the eigenfunction corresponding to the top eigenvalue of the covariance operator, 
$$\Sigma=\bb{E}\Phi(X)\oh\Phi(X)-\bb{E}\Phi(X)\oh\bb{E}\Phi(X),$$ where $\Phi(x):=k(\cdot,x)$ is called the \emph{feature map}. More generally, kernel PCA provides a Euclidean representation for $X$ by projecting $\Phi(X)$
onto the $\ell$-eigenspace of $\Sigma$. Clearly, if $\Phi(x)=x,\,x\in\X=\R^d$ (which corresponds to the linear kernel, $k(x,y)=\langle x,y\rangle_2,\,x,y\in\R^d$), then kernel PCA reduces to
linear PCA. On the other hand, depending on the choice of $k$, higher order moments of $X$ are considered through the second order moment of $\Phi(X)$ to compute the $\ell$-dimensional
representation for $X$, resulting in a non-linear interpretation of dimensionality reduction. 
Due to this, KPCA is popular in applications such as image denoising \citep{Mika-99, Jade-03, Teixeira-08,Phophalia-17}, image/systems modeling \citep{Kim-05,Li-15}, 
novelty/fault detection \citep{Hoffmann-07, Samuel-16,deMoura-17}, feature extraction \citep{Chang-15}, and computer vision \citep{Lampert-09, Peter-19}. We refer the reader to Section~\ref{Sec:notation} for notation and Section~\ref{Sec:kpca} for preliminaries on KPCA and its variants.

Given $X_1,\ldots,X_n\stackrel{i.i.d.}{\sim}\Pb$ with $\Pb$ being unknown, empirical version of linear PCA computes the eigenvectors of the empirical covariance matrix onto which $(X_i)^n_{i=1}$ are 
projected to obtain a low-dimensional representation. Similarly, the empirical version of kernel PCA (we refer to it as EKPCA) involves finding the eigenfunctions of the empirical covariance operator 
$$\sigh=\frac{1}{2n(n-1)}\sum^n_{i\ne j}\left(\Phi(X_i)-\Phi(X_j)\right)\oh \left(\Phi(X_i)-\Phi(X_j)\right).$$ 
While this requires solving a possibly infinite dimensional 
eigenvalue problem, it can be shown that these eigenfunctions can be computed by only solving a finite dimensional eigenvalue problem (see Proposition~\ref{pro:eigsystem}). In particular,
it involves finding the eigenvectors of the Gram matrix, $[k(X_i,X_j)]^n_{i,j=1}$ which has a computational requirement of $O(n^2\ell)$ where $\ell$ is the number of eigenvectors of interest. 
In addition to KPCA, more generally, most of the kernel algorithms (see \citealp{Scholkopf-02}) have a space complexity requirement of $O(n^2)$ and time complexity 
requirement of $O(n^3)$ as in some sense,
all of them involve an eigen decomposition of the Gram matrix. However, in big data 
scenarios where $n$ is large, 
the kernel methods including KPCA suffer from large space and time complexities.

An elegant approach to address this computational issue is to approximate the feature map $\Phi$ by a finite-dimensional map $\Phi_m$, i.e., $\Phi_m(x)\in\Rm$ so that $\sigh$
is approximated as $$\sigh_m=\frac{1}{2n(n-1)}\sum^n_{i\ne j}\left(\Phi_m(X_i)-\Phi_m(X_j)\right)\left(\Phi_m(X_i)-\Phi_m(X_j)\right)^\top.$$ Clearly, this is
equivalent to performing linear PCA on the mapped data $(\Phi_m(X_i))^n_{i=1}$, which involves finding the eigensystem of $\sigh_m$.
Since this has a computational complexity 
of $O(m^2\ell+m^2n)$, the computational burden is reduced from $O(n^2\ell)$ if $m<\sqrt{n\ell}$. However, since this computational gain may be achieved at the cost of statistical performance, the goal of this paper 
is to investigate the trade-off between computational and statistical efficiency of approximate empirical KPCA (we refer to it as RF-EKPCA) using a random finite dimensional approximation of $\Phi(X)$.

In the following,
we briefly introduce the idea of random feature approximation introduced by \citet{Rahimi-08a}, which involves computing a finite dimensional feature map that approximates the kernel function. Suppose say $k$ is a continuous translation invariant kernel on $\R^d$, i.e., $k(x,y)=\psi(x-y),\,x,y\in\R^d$ 
where $\psi$ is a continuous positive definite function on $\R^d$. Bochner's theorem \citep[Theorem 6.6]{Wendland-05} states that $\psi$ is the Fourier transform of a finite non-negative
Borel measure $\Lambda$ on $\R^d$, i.e.,
\begin{equation}
k(x,y)=\int_{\R^d}e^{-\sqrt{-1}\langle x-y,\omega\rangle_2}\,d\Lambda(\omega)\stackrel{(\star)}{=}\int_{\R^d} \cos(\langle x-y,\omega\rangle_2)\,d\Lambda(\omega),\label{Eq:Bochner1} 
\end{equation}
where $\langle \cdot,\cdot\rangle_2$ denotes the usual Euclidean inner product and $(\star)$ follows from the fact that $\psi$ is real-valued
and symmetric. Since $\Lambda(\R^d)=\psi(0)$, we can write \eqref{Eq:Bochner1} as
$k(x,y)=\psi(0)\int_{\R^d} \cos(\langle x-y,\omega\rangle_2)\,d\frac{\Lambda}{\psi(0)}(\omega)$
where $\frac{\Lambda}{\psi(0)}$ is a probability measure on $\R^d$. Therefore, without loss of generality, throughout the paper we assume that $\Lambda$ is a probability measure.
\citet{Rahimi-08a} proposed a random approximation to $k$ by replacing the integral with Monte Carlo sums constructed from $(\omega_i)^m_{i=1}\stackrel{i.i.d.}{\sim}\Lambda$, i.e.,
$$ k_m(x,y)=\psi_m(x-y)=\frac{1}{m}\sum^m_{i=1}\cos(\langle x-y,\omega_i\rangle_2)\stackrel{(\dagger)}{=}\langle \Phi_m(x),\Phi_m(y)\rangle_2,$$
where $\Phi_m=\frac{1}{\sqrt{m}}(\cos\langle \cdot,\omega_1\rangle_2,\ldots,\cos\langle \cdot,\omega_m\rangle_2,\sin\langle \cdot,\omega_1\rangle_2,\ldots,\sin\langle \cdot,\omega_m\rangle_2)^\top$ 
and $(\dagger)$ holds based on the trigonometric identity: $\cos(a-b)=\cos a \cos b+\sin a \sin b$. This kind of random approximation to $k$ can be constructed for a more general class of kernels
of the form $$k(x,y)=\int_{\Theta} \vp(x,\theta) \vp(y,\theta)\,d\Lambda(\theta)$$
by using $$k_m(x,y)=\frac{1}{m}\sum^m_{i=1}\vp(x,\theta_i)\vp(y,\theta_i)=\langle \Phi_m(x),\Phi_m(y)\rangle_2,$$
where $\Phi_m=\frac{1}{\sqrt{m}}
(\vp(\cdot,\theta_1),\ldots,\vp(\cdot,\theta_m))^\top,$ $\vp(x,\cdot)\in L^2(\Theta,\Lambda)$ for all $x\in\Cal{X}$, $(\theta_i)^m_{i=1}\stackrel{i.i.d.}{\sim}\Lambda$,
with $\Cal{X}$ and $\Theta$ being measurable spaces. 
Based on this approximation, the question of interest is whether RF-EKPCA consistent and how should $m$ depend on $n$ for RF-EKPCA to have similar statistical behavior to that of EKPCA, while still maintaining the computational edge. The goal of this paper is to address these questions.
\subsection{Contributions}\label{subsec:contribution}
The main contributions of the paper are as follows:\vspace{2mm}\\
\emph{(i)} In Section~\ref{Sec:results}, we compare the performance of RF-EKPCA with EKPCA in terms of the reconstruction error of the associated $\ell$-eigenspace, i.e., 
the error involved in reconstructing $\Phi(X)$ based on its projections onto the corresponding $\ell$-eigenspace. Since the $\ell$-eigenspace associated with RF-EKPCA is a subspace 
of $\Rm$ in contrast to $\Hk$ as is the case with EKPCA, the notion of projecting $\Phi(X)\in\Hk$ onto a subspace of $\Rm$ is vacuous. To alleviate the problem, we define inclusion and approximation operators that embed both $\Hk$ and $\Rm$ as subspaces in $\lp$. This, however, results in two different notions of reconstruction error: First reconstructing in 
$\Cal{H}$ and $\bb{R}^m$ and then embedding the reconstructed functions in $L^2(\bb{P})$, which we refer to as \textit{Reconstruct and Embed} (R-E), in contrast to first embedding the functions into $L^2(\bb{P})$ and then reconstructing in $L^2(\bb{P})$, which we refer to as \textit{Embed and Reconstruct} (E-R). In Propositions~\ref{pro:solution}, \ref{pro:interpret}, and \ref{pro:interpret2}, we provide a new reformulation of KPCA, EKPCA, and RF-EKPCA as minimizers of appropriate E-R and R-E reconstruction errors. This reformulation provides a generalization error type interpretation which can be used to investigate the statistical behavior of EKPCA and RF-EKPCA. Since PCA is a special case of KPCA, this reformulation also provides a novel interpretation for classical PCA as a minimizer of covariance matrix weighted reconstruction error. \vspace{1mm}\\
\emph{(ii)} In Sections~\ref{subsec:ep} and \ref{subsec:pe}, 
we show that RF-EKPCA has better computational complexity with no loss in statistical performance than EKPCA as long as $m$, which grows monotonically with $\ell$, is large enough with $\ell$ not being too large (Theorems~\ref{thm:rff main thm metric 1} and \ref{thm:rff main thm metric 2}). In other words, the number of eigen functions, $\ell$ used in the reconstruction cannot grow too fast with the sample size $n$ while requiring enough number of random features $m$ so that the approximation error does not dominate the estimation error.
By specializing Theorems~\ref{thm:rff main thm metric 1} and \ref{thm:rff main thm metric 2} to the cases of polynomial and exponential decay rates of the eigenvalues of $\Sigma$, in Corollaries~\ref{rff poly decay corollary}, \ref{rff exp decay corollary} 
and \ref{rff poly decay corollary metric 2}, \ref{rff exp decay corollary metric 2}, 
respectively, we show 
R-E and E-R reconstruction errors to have \emph{different} statistical behaviors. However, under each of these reconstruction errors, as mentioned above, RF-EKPCA matches the statistical performance of EKPCA at better computational complexity.\vspace{1mm}\\
\emph{(iii)} In Section~\ref{subsecsec:othernorm}, we investigate a generalization of R-E (similar generalization also holds for E-R) 
based on certain weighted $L^2(\mathbb{P})$-norms that are weighted by $(\id\id^*)^{-s/2},\,\,s\le 1$ with $\id$ being the inclusion operator, wherein the choice of $s=1$ yields a reconstruction error that matches with the reconstruction error for KPCA in the $\Cal{H}$-norm and $s=0$ matches with the R-E reconstruction error considered in Section~\ref{subsec:ep}. In Proposition~\ref{pro:schatten}, we again provide a new reformulation of KPCA, EKPCA, and RF-EKPCA as minimizers of the generalized R-E reconstruction errors, using which 
we establish a similar result as aforementioned that RF-EKPCA has same statistical complexity and better computational complexity than EKPCA as long as $m$ is sufficiently large with $\ell$ being sufficiently small with respect to the growth of $n$ (see Theorem~\ref{thm:schatten} and Remark~\ref{rem:3}). 

All these results hinge on Bernstein-type inequalities for the operator and Hilbert-Schmidt norms of a self-adjoint Hilbert-Schmidt operator-valued U-statistics, which are of independent interest (see Theorem~\ref{thm:bernstein U-stat}). 
\subsection{Related work}\label{subsec:work}
To the best of our knowledge, not much investigation has been carried out on the 
statistical analysis of RF-EKPCA. \citet{Lopez-Paz-14} studied the quality of approximation of the Gram matrix by 
the approximate Gram matrix (using random Fourier features) in operator norm and 
showed a convergence rate of $n(\sqrt{(\log n)/m}+(\log n)/m)$. This 
approximation bound is too loose as we require $m$ to grow faster than $n$ to 
achieve convergence to zero, which defeats the purpose of random feature 
approximation. More recently, based on \citep{Blanchard-07} and an earlier version of this work \citep{Sriperumbudur-17}, using inclusion and approximation operators, \citet{Ullah-18} compared the $\ell$-eigenspaces (with $\ell$ fixed) of EKPCA and RF-EKPCA by comparing certain inner product of the uncentered covariance operator with the difference between the projection operators associated with $\ell$-eigenspaces of KPCA and EKPCA (\emph{resp.} RF-EKPCA), after embedding them all as Hilbert-Schmidt operators on $L^2(\bb{P})$. Through upper bounds on these  differences of inner products, they argued that $m=\sqrt{n}$ random features are sufficient for RF-EKPCA to have similar statistical behavior to that of EKPCA, thereby guaranteeing better computational complexity at no statistical loss. However, the work lacks on two fronts: (i) The comparison is made using only upper bounds on the performance criterion (i.e., difference of inner products) and no matching lower bounds are provided to establish their sharpness, which means the sufficiency of $\sqrt{n}$ random features is inconclusive, and (ii) the criterion used for comparison has no clear interpretation. In contrast, in this work, we use performance criteria which have a clear interpretation and establish matching upper and lower bounds on their statistical behavior.

On the other hand, statistical behavior of EKPCA is well 
understood. \citet{Shawe-Taylor-05} studied the statistical consistency of EKPCA 
in terms of the reconstruction error of the estimated $\ell$-eigenspace and 
obtained a convergence rate of $n^{-1/2}$. By taking into account the decay rate of the eigenvalues of the covariance operator, 
improved rates are obtained by \citet{Blanchard-07} and \citet{Rudi-13}.
However, unlike in this paper 
where the reconstruction error is defined in terms of convergence in $\lp$, 
these works consider convergence in $\Hk$. The question of convergence of 
$\ell$-eigenspaces associated with EKPCA was considered by \citet{Zwald-05} as 
convergence of orthogonal projection operators on $\Hk$ in Hilbert-Schmidt norm 
and obtained a convergence rate of $n^{-1/2}$.

In the discussion so far, we only considered random feature approximation to $\Phi$. At a broader level, to address the computational issues, various other approximation methods have been 
proposed and investigated in the kernel methods literature. Some of the popular approximation strategies include the incomplete Cholesky 
factorization \citep{Fine-01, Bach-05a}, Nystr\"{o}m method (e.g., see \citealp{Williams-01, Drineas-05}), sketching \citep{Yang-17}, sparse greedy approximation 
\citep{Smola-00}, etc. 
While it has been widely accepted that these approximate methods including random feature approximation provide significant computational advantages and has been empirically shown to provide learning algorithms or solutions
that do not suffer from significant deterioration in performance compared to those without approximation \citep{Rahimi-08a,Kumar-09,Yang-12,Yang-17}, until recently, the statistical consistency of 
these approximate methods is not well understood. 
In fact, over the last few years, the statistical behavior of these approximation schemes have been investigated only in the context of kernel ridge regression, wherein 
it has been shown \citep{Bach-13, Alaoui-15, Rudi-15, Yang-17, Rudi-17} that Nystr\"{o}m, 
random feature and sketching based approximate kernel ridge regression are 
consistent and achieve minimax rates of convergence as achieved by the 
exact methods 
but using fewer features than the sample size.
This means, these approximate 
kernel ridge regression algorithms are not only computationally efficient 
compared to their exact counterpart
but also statistically efficient, i.e., achieve the best possible convergence 
rate. On the other hand, the theoretical behavior of approximate kernel 
algorithms other than approximate kernel ridge regression is not well 
understood. 
This paper provides a theoretical understanding on the question of computational 
vs. statistical trade-off in random feature based approximate kernel 
PCA.

\section{Definitions \& Notation}\label{Sec:notation}

Define $\Vert
\bm{a}\Vert_2:=\sqrt{\sum^d_{i=1}a^2_i}$ and $\langle \bm{a},\bm{b}\rangle_2:=\sum^d_{i=1}a_ib_i$, where $\bm{a}:=(a_1,\ldots,a_d)\in\bb{R}^d$ and $\bm{b}:=(b_1,\ldots,b_d)\in\bb{R}^d$. 
$\bm{a}\otimes_2 \bm{b}:=\bm{a}\bm{b}^\top$ denotes the tensor product of $\bm{a}$ and $\bm{b}$. $\bm{I}_n$ denotes an $n\times n$ identity matrix. We define $\bm{1}_n:=(1,\stackrel{n}{\ldots},1)^\top$ and $\bm{H}_n:=\bm{I}_n-\frac{1}{n}\bm{1}_n\otimes_2\bm{1}_n$. $\delta_{ij}$ denotes the
Kronecker delta. $a\wedge b:=\min(a,b)$ and $a\vee b:=\max(a,b)$. $[n]:=\{1,\ldots,n\}$ for $n\in\bb{N}$. For constants $a$ and $b$, $a\lesssim b$ (\emph{resp.} $a\gtrsim b$) denotes that there exists a positive constant $c$ (\emph{resp.} $c'$) such that $a\le cb$ (\emph{resp.} $a\ge c'b$). For a random variable $A$ with law $P$ and a constant $b$, $A\lesssim_P b$ denotes that for any $\delta>0$, there exists a positive constant $c_\delta<\infty$ such that $P(A\le c_\delta b)\ge \delta$.

For a topological
space $\Cal{X}$, 
$M^b_+(\Cal{X})$ denotes the set of all finite non-negative
Borel measures on $\Cal{X}$. For $\mu\in M^b_+(\Cal{X})$, $L^r(\Cal{X},\mu)$
denotes the Banach space of $r$-power ($r\ge
1$) $\mu$-integrable functions. For $f\in L^r(\Cal{X},\mu)$, $\Vert
f\Vert_{L^r(\mu)}:=\left(\int_\Cal{X}|f|^r\,d\mu\right)^{1/r}$ denotes
the $L^r$-norm of $f$ for $1\le r<\infty$. $\mu^n:=\mu\times\stackrel{n}{\ldots}\times\mu$ is the $n$-fold product measure. 
$\Cal{H}$ denotes a reproducing kernel Hilbert space with a reproducing kernel $k:\Cal{X}\times\Cal{X}\rightarrow\bb{R}$. 

$S\in\Cal{L}(H)$
is called \emph{self-adjoint} if
$S^*=S$ and is called \emph{positive} if $\langle Sx,x\rangle_H\ge 0$ for all
$x\in H$, where $\Cal{L}(H)$ is the space of bounded linear operators on a Hilbert space $H$ and $S^*$ denotes the adjoint of $S$. $\alpha\in\bb{R}$ is called an \emph{eigenvalue} of $S\in\Cal{L}(H)$
if there exists an $x\ne 0$ such that $Sx=\alpha x$ and such an $x$ is called
the \emph{eigenvector}/\emph{eigenfunction} of $S$ and $\alpha$. An eigenvalue is said to be \emph{simple} if it has multiplicity one. $\Vert S\Vert_{\Cal{L}^r(H)}$ denotes the trace, Hilbert-Schmidt and operator norms of a self-adjoint operator $S\in\Cal{L}(H)$ when $r=1,\,2$ and $\infty$, respectively. 
For $x,y\in
H$, $x\otimes_{H} y$ is an element of the tensor product space
$H\otimes H$ which can also be seen as an operator from $H$ to $H$ as
$(x\otimes_{H} y)z=x\langle y,z\rangle_{H}$ for any $z\in H$.
%

\section{Variants of Kernel PCA: Population, Empirical and Approximate}
\label{Sec:kpca}

In this section, we review kernel PCA \citep{Scholkopf-98} in population and empirical settings and introduce approximate kernel PCA based on random features. This section not only provides preliminaries on kernel PCA but also fixes some notation that will be used throughout the paper. To start with, we assume the following for the rest of the paper:
\begin{itemize}
 \item[($A_1$)] $(\Cal{X},\Cal{B})$ is a second countable (i.e., completely separable) space endowed with Borel $\sigma$-algebra $\Cal{B}$. $(\Cal{H},k)$ is an RKHS of real-valued functions on $\Cal{X}$ with a bounded continuous strictly positive definite kernel $k$ satisfying $\sup_{x\in\Cal{X}}k(x,x)=:\kappa<\infty$.
\end{itemize}
The second countability of $\Cal{X}$ and $\Cal{B}$ being countably generated ensure that for any $\sigma$-finite measure $\mu$ defined on $\Cal{B}$, $L^r(\Cal{X},\mu)$ is separable for any $r\in [1,\infty)$ \citep[Proposition 3.4.5]{Cohn-13}. The second countability of $\Cal{X}$ and continuity of $k$ ensures $\Cal{H}$ is separable \citep[Lemma 4.33]{Steinwart-08}. The separability of $\Cal{H}$ and $k$ being bounded continuous ensures that $k(\cdot,x):\Cal{X}\rightarrow\Cal{H}$ is Bochner-measurable for all $x\in\Cal{X}$ \citep[Theorem 8 on p.5]{Dinculeanu-00}. The separability of $L^r(\Cal{X},\mu)$ and Bochner-measurability of $k(\cdot,x)$ will be crucial in our analysis.
 
\subsection{PCA in Reproducing Kernel Hilbert Space}\label{subsec:pca}

As mentioned in Section~\ref{Sec:Introduction}, 
kernel PCA extends the idea of PCA in $\bb{R}^d$ to an 
RKHS by finding a function $f\in\Cal{H}$ such that $\var[f(X)]$ is maximized, i.e.,
\begin{equation*}\sup\left\{\var[f(X)]:\Vert f\Vert_\Cal{H}=1\right\}=\sup\left\{\bb{E}\left[f(X)-\bb{E}\left[f(X)\right]\right]^2:\Vert f\Vert_\Cal{H}=1\right\}.\nonumber
\end{equation*}
Since $f\in\Cal{H}$, using the reproducing property $f(X)=\langle f,k(\cdot,X)\rangle_\Cal{H}$, we have $\var[f(X)]=\bb{E}\left[\langle f,k(\cdot,X)\rangle_\Cal{H}-\langle f,m_\bb{P}\rangle_\Cal{H}\right]^2$
where $m_\bb{P}:=\int_\Cal{X} k(\cdot,x)\,d\bb{P}(x)\in\Cal{H}$ is the \emph{mean element} of $\bb{P}$, defined as: for all $f\in \Cal{H}$, $\langle f,m_\bb{P}\rangle_\Cal{H}=\bb{E}\left[f(X)\right]$. The boundedness of $k$ guarantees that $\Mp$ is well-defined as it ensures 
$k(\cdot,X)$ is $\bb{P}$-integrable in the Bochner sense (see \citealp*[Definition 1 and Theorem 2]{Diestel-77}). Therefore,
\begin{eqnarray*}
\var[f(X)]=\bb{E}\left[\langle f,k(\cdot,X)-m_\bb{P}\rangle^2_\Cal{H}\right]
\stackrel{(\star)}{=}{} \langle f,\Sigma f\rangle_\Cal{H},\nonumber
\end{eqnarray*}
where $(\star)$ follows from the Riesz representation theorem and the boundedness of $k$, which combinedly guarantee the Bochner $\bb{P}$-integrability of 
$k(\cdot,X)\oh k(\cdot,X)$. Here
\begin{equation}\label{Eq:cov}
\Sigma:=\int_\Cal{X}(k(\cdot,x)-m_\bb{P})\oh (k(\cdot,x)-m_\bb{P})\,d\bb{P}(x)
\end{equation}
is the covariance operator on $\Cal{H}$ whose action on $f\in\Cal{H}$ is defined as
$\Sigma f=\int_\Cal{X} k(\cdot,x)f(x)\,d\bb{P}(x)-m_\bb{P}\int_\Cal{X}f(x)\,d\bb{P}(x).$
Therefore, the kernel PCA problem exactly resembles classical PCA where the goal is to find $f\in\Cal{H}$ that solves $\sup\left\{\langle f,\Sigma f\rangle_\Cal{H}:\Vert f\Vert_\Cal{H}=1\right\}$ 
with $\Sigma$ being defined as in \eqref{Eq:cov}. Since $k$ is bounded, it can be shown (see Proposition~\ref{pro:id}\emph{(iii)}) 
that $\Sigma$ is a trace-class operator and therefore Hilbert-Schmidt and compact.
Also it is obvious that $\Sigma$ is self-adjoint and positive and therefore by spectral theorem \citep[Theorems VI.16, VI.17]{Reed-80}, $\Sigma$ can be written as 
\begin{equation}\Sigma=\sum_{i\in I} \lambda_i \phi_i \oh \phi_i,\label{Eq:cov-eig}\end{equation}
where $(\lambda_i)_{i\in I}\subset\bb{R}^+$ are the eigenvalues and $(\phi_i)_{i\in I}$ are the orthonormal system of eigenfunctions of $\Sigma$ that span $\overline{\text{Ran}(\Sigma)}$ with the index set $I$ being either 
countable in which case $\lambda_i\rightarrow 0$ as $i\rightarrow\infty$
or finite. It is therefore obvious that the solution to KPCA 
is an eigenfunction of $\Sigma$ corresponding to the largest eigenvalue. 

Throughout the paper, we assume that 
\begin{itemize}
\item[($A_2$)] The eigenvalues $(\lambda_i)_{i\in I}$ of $\Sigma$ in \eqref{Eq:cov} are simple, positive and without any loss of generality, they satisfy a decreasing rearrangement, i.e., $\lambda_1>\lambda_2>\cdots$.
\end{itemize}
($A_2$) ensures that $(\phi_i)_{i\in I}$ form an orthonormal basis and the eigenspace corresponding
to $\lambda_i$ for any $i\in I$ is one-dimensional. This means, the orthogonal projection operator 
onto $\text{span}\{(\phi_i)^\ell_{i=1}\}$ is given by $P_\ell(\Sigma) =\sum^\ell_{i=1}\phi_i\oh \phi_i.$

\subsection{Empirical Kernel PCA}\label{subsec:empirical}
In practice, $\bb{P}$ is unknown and the knowledge of $\bb{P}$ is available only through random samples $(X_i)^n_{i=1}$ drawn i.i.d.~from it. The goal of empirical kernel PCA (EKPCA) is therefore to find
$f\in\Cal{H}$ such that 
$$\widehat{\var}[f(X)]:=\frac{1}{2n(n-1)}\sum^n_{i\ne j}\left(f(X_i)-f(X_j)\right)^2
,$$ 
i.e., the empirical variance, is maximized. Note that this is an estimate of 
\begin{eqnarray}
\var[f(X)]&{}={}&\bb{E}[f^2(X)]-\bb{E}^2[f(X)]
=\int f^2(x)\,d\bb{P}(x)
-\langle f,m_\bb{P}\rangle^2_\Cal{H}\nonumber
\end{eqnarray}
based on the $U$-statistic representation, although in the literature (e.g., \citealt{Scholkopf-98}), assuming $m_\bb{P}=0$, a $V$-statistic form, i.e., $\frac{1}{n}\sum^n_{i=1}f^2(X_i)$ is used. However, it is important to note that the assumption of $m_\bb{P}=0$ is not satisfied by many kernels including the Gaussian kernel, and if this assumption is relaxed, the corresponding $V$-statistic form is not unbiased. Since unbiasedness turns to be crucial in our analysis, we choose the above $U$-statistic form though from the point of view of methodology alone, the $V$-statistic can be equally used. 

Using the reproducing property,
it is easy to show that $\widehat{\var}[f(X)]=\langle f,\sigh f\rangle_\Cal{H}$ where $\sigh:\Cal{H}\rightarrow\Cal{H}$, 
\begin{equation}\sigh:=\frac{1}{2n(n-1)}\sum^n_{i\ne j}\left(k(\cdot,X_i)-k(\cdot,X_j)\right)\oh \left(k(\cdot,X_i)-k(\cdot,X_j)\right)\label{Eq:emp-cov}\end{equation}
is an unbiased estimator ($U$-statistic) of $\Sigma$, referred to 
as the empirical covariance operator. 
Since $\sigh$ is a self-adjoint operator on (a possibly infinite dimensional) $\Cal{H}$ with rank at most $n-1$ (therefore, compact), 
it follows from the spectral theorem \citep[Theorems VI.16, VI.17]{Reed-80} that
\begin{equation}\sigh=\sum^{n-1}_{i=1}\widehat{\lambda}_i\widehat{\phi}_i\oh \widehat{\phi}_i,\label{Eq:emp-sig}\end{equation}
where $(\widehat{\lambda}_i)^{n-1}_{i=1}$ and $(\widehat{\phi}_i)^{n-1}_{i=1}$ are the eigenvalues and eigenfunctions of $\sigh$. In fact, since $k$ is strictly positive definite, it can be shown that $\text{rank}(\sigh)=n-1$ $\bb{P}$-a.s., and therefore, similar to ($A_2$), we assume the following:
\begin{itemize}
\item[($A_3$)] The eigenvalues $(\widehat{\lambda}_i)^{n-1}_{i=1}$ of $\sigh$ in \eqref{Eq:emp-cov} are simple $\bb{P}$-a.s.~and without loss of generality, 
they satisfy a decreasing rearrangement, i.e., $\widehat{\lambda}_1>\widehat{\lambda}_2>\cdots$ $\bb{P}$-a.s.
\end{itemize}
We would like to mention that the simplicity of the eigenvalues of $\sigh$ is not really required for the results of this paper to hold. However, this assumption simplifies the notation and proofs, and therefore for the sake of simplicity and clarity, we resort to the above assumption.

Based on ($A_3$), a low-dimensional Euclidean representation of $X_i\in\X$ can be obtained as 
\begin{equation}
\left(\langle k(\cdot,X_i),\phih_1\rangle_{\Hk},\ldots,\langle k(\cdot,X_i),\phih_\ell\rangle_{\Hk}\right)^\top=\left(\phih_1(X_i),\ldots,\phih_\ell(X_i)\right)^\top
,\label{Eq:dim-red}
\end{equation}
where $\ell<n-1$ and $i\in[n]$. Clearly, the choice of $k(\cdot,x)=\langle\cdot,x\rangle_2$ for $x\in\bb{R}^d$ in \eqref{Eq:dim-red} reduces to the usual low-dimensional representation using linear PCA. Under ($A_3$), we denote the orthogonal projection operator
onto $\text{span}\{(\phih_i)^\ell_{i=1}\}$ as $P_\ell(\sigh)$, which is given by
$P_\ell(\sigh)=\sum^\ell_{i=1}\widehat{\phi}_i\oh \widehat{\phi}_i.$

Note that the Euclidean representation in \eqref{Eq:dim-red} requires the knowledge of $(\phih_i)^{n-1}_{i=1}$, which are not obvious to compute even though $\sigh$ has finite rank, as they are solution to a possibly infinite dimensional eigen problem. The following result (proved in Section~\ref{subsec:pro-eigsystem}) 
shows that the eigensystem $(\widehat{\lambda}_i,\widehat{\phi}_i)^{n-1}_{i=1}$ of $\sigh$ can be obtained by finding the eigensystem of a $n\times n$ matrix. 
This means the computation of $(\lambdah_i,\phih_i)^{\ell}_{i=1}$ for $\ell\le n$ has a space complexity of $O(n^2)$ and a time complexity of $O(n^2\ell)$--e.g., by partial SVD methods such as Krylov subspace method (see \citealp[Sections 3.3.2 \& 3.3.3]{Halko-11}).
\begin{proposition}\label{pro:eigsystem}
Let $(\lambdah_i,\phih_i)_i$ be the eigensystem of $\sigh$ in \eqref{Eq:emp-sig}. Define $\bm{K}=[k(X_i,X_j)]_{i,j\in[n]}$ and $\bm{H}_n:=\bm{I}_n-\frac{1}{n}\bm{1}_n\otimes_2\bm{1}_n$. Then 
$(\lambdah_i,\widehat{\bm{\alpha}}_i)_i$ are the eigenvalues and eigenvectors of $\frac{1}{n-1}\bm{KH}_n$ with 
$$\phih_i=\frac{1}{\lambdah_i\sqrt{n}}\sum^n_{j=1} \gamma_{i,j}k(\cdot,X_j),$$
where $\bm{\gamma}_i:=(\gamma_{i,1},\ldots,\gamma_{i,n})^\top=\frac{n}{n-1}\bm{H}_n\widehat{\bm{\alpha}}_i$ with $\widehat{\bm{\alpha}}_i\notin \Cal{N}(\bm{H}_n)$.
\end{proposition}
Using representer theorem \citep{Kimeldorf-71}, \citet{Scholkopf-98} have shown a similar result for EKPCA but with uncentered covariance operator (i.e., $\Sigma$ with $\Mp=0$) when $\bm{K}$ is invertible. Since $\Mp=0$ is not a valid assumption for many kernels, Proposition~\ref{pro:eigsystem} handles the U-statistic version of the centered covariance operator without using representer theorem and without requiring $\bm{K}$ to be invertible.
\subsection{Approximate Kernel PCA using Random Features}\label{subsec:rff}
In this section, we present approximate kernel PCA using random features, which we call as RF-KPCA.
Throughout this section, we assume the following:
\begin{itemize}
 \item[($A_4$)] $\Cal{H}$ is an RKHS with reproducing kernel $k$ of the form
 \begin{equation*}
k(x,y)=\intt \vp(x,\theta) \vp(y,\theta)\,d\Lambda(\theta)=\langle \vp(x,\cdot),\vp(y,\cdot)\rangle_{L^2(\Lambda)},\nonumber
\end{equation*}
where $\vp:\Cal{X}\times\Theta\rightarrow\bb{R}$ is continuous, $\sup_{\theta\in\Theta,x\in\Cal{X}}|\vp(x,\theta)|\le\sqrt{\kappa}$ and $\Lambda$ is a probability measure on a second countable space $(\Theta,\Cal{A})$ endowed with Borel $\sigma$-algebra $\Cal{A}$.
\end{itemize}
The assumption of $\Lambda$ being a probability measure on $\Theta$ is not restrictive as any $\Lambda\in M^b_+(\Theta)$ can be normalized to a probability measure. However, the 
uniform boundedness of $\vp$ over $\Cal{X}\times\Theta$ is somewhat restrictive as it is sufficient to assume $\vp(x,\cdot)\in L^2(\Cal{X},\Lambda),\,\forall\,x\in\Cal{X}$ for $k$ to be
well-defined. But the uniform boundedness of $\vp$ ensures that $k$ is bounded, as assumed in $(A_1)$. By sampling $(\theta_i)^m_{i=1}\stackrel{i.i.d.}{\sim}\Lambda$,
an approximation to $k$ can be constructed as
\begin{equation*}
 k_m(x,y)=\frac{1}{m}\sum^m_{i=1}\vp(x,\theta_i)\vp(y,\theta_i)=:\sum^m_{i=1}\vp_i(x)\vp_i(y)=\langle \Phi_m(x),\Phi_m(y)\rangle_{2},\nonumber
\end{equation*}
where $\vp_i:=\frac{1}{\sqrt{m}}\vp(\cdot,\theta_i)$ and $\Phi_m(x):=(\vp_1(x),\ldots,\vp_m(x))^\top\in\bb{R}^m$ is the random feature map. It is easy to verify that $k_m$ is the reproducing kernel of the
RKHS \begin{equation*}\Cal{H}_m=\left\{f:f=\sum^m_{i=1}\beta_i\vp_i,\,(\beta_i)^m_{i=1}\subset\bb{R}\right\}\nonumber\end{equation*}
w.r.t.~$\langle \cdot,\cdot\rangle_{\Cal{H}_m}$ defined as $\langle f,g\rangle_{\Cal{H}_m}:=\sum^m_{i=1}\alpha_i\beta_i$ where $g=\sum^m_{i=1}\alpha_i\vp_i$. 
Therefore $\Cal{H}_m$ is isometrically isomorphic to $\bb{R}^m$. We refer the reader to \cite[Appendix E]{Rudi-17} for examples of $\varphi$ that yield some widely used reproducing kernels.

Having obtained a random feature map, the idea of RF-KPCA is to perform linear PCA on $\Phi_m(X)$ where $X\sim\bb{P}$, i.e., RF-KPCA involves
finding a direction $\bm{\beta}\in\bb{R}^m$ such that the variance of $\langle\bm{\beta},\Phi_m(X)\rangle_2$ is maximized:
\begin{equation}
\sup\left\{\var[\langle \bm{\beta},\Phi_m(X)\rangle_2]:\Vert\bm{\beta}\Vert_2=1\right\}=\sup\left\{\langle\bm{\beta}, \Omega_m\bm{\beta}\rangle_2:\Vert\bm{\beta}\Vert_2=1\right\},\label{Eq:finite-dim-Rm}
\end{equation}
where $\Omega_m:=\text{Cov}[\Phi_m(X)]=\bb{E}[(\Phi_m(X)-\bb{E}[\Phi_m(X)])\otimes_2 (\Phi_m(X)-\bb{E}[\Phi_m(X)])]$
is a self-adjoint positive definite matrix. In fact, it is easy to verify that performing linear PCA on $\Phi_m(X)$ is same as performing KPCA in $\Cal{H}_m$ since
\begin{equation}\sup\left\{\var[f(X)]:\Vert f\Vert_{\Cal{H}_m}=1\right\}=\sup\left\{\var[\langle\bm{\beta},\Phi_m(X)\rangle_2]:\Vert\bm{\beta}\Vert_2=1\right\},
\label{Eq:Rm-Hm}
\end{equation}
which follows from $\Cal{H}_m$ being isometrically isomorphic to $\bb{R}^m$ and $f\in\Cal{H}_m$ has the form $f(x)=\langle \bm{\beta},\Phi_m(x)\rangle_2$. Note that 
\begin{equation}
\sup\left\{\var[f(X)]:\Vert f\Vert_{\Cal{H}_m}=1\right\}=\sup\left\{\langle f,\Sigma_m f\rangle_{\Cal{H}_m}:\Vert f\Vert_{\Cal{H}_m}=1\right\},\label{Eq:Hm-equal}
\end{equation}
where \begin{eqnarray}
\Sigma_m&{}={}&\int_{\Cal{X}}k_m(\cdot,x)\ohm k_m(\cdot,x)\,d\bb{P}(x)-\mpm\ohm\mpm,\nonumber
\end{eqnarray}
and $\mpm:=\int k_m(\cdot,x)\,d\bb{P}(x)$. It therefore follows from \eqref{Eq:finite-dim-Rm}--\eqref{Eq:Hm-equal} that the eigenvalues of $\Sigma_m$ and $\Omega_m$ coincide and the eigenfunctions, $(\phi_{m,i})^m_{i=1}$ of $\Sigma_m$ and eigenvectors, $(\bm{\beta}_{m,i})^m_{i=1}$ of $\Omega_m$ are related as $\phi_{m,i}(x)=\langle \bm{\beta}_{m,i},\Phi_m(x)\rangle_2$.

The empirical counterpart of 
RF-KPCA (we call it as RF-EKPCA) is obtained by solving 
$$\sup_{\Vert\bm{\beta}\Vert_2=1}\widehat{\text{Var}}[\langle\bm{\beta},\Phi_m(X)\rangle_2]=\sup_{\Vert\bm{\beta}\Vert_2=1}\langle\bm{\beta},\widehat{\Omega}_m\bm{\beta}\rangle_2=\sup_{\Vert f\Vert_{\Cal{H}_m}=1}\langle f,\sigh_m f\rangle_{\Cal{H}_m},$$
where 
\begin{eqnarray}
\sigh_m=\frac{1}{2n(n-1)}\sum^n_{i\ne j}\left(k_m(\cdot,X_i)-k_m(\cdot,X_j)\right)\ohm \left(k_m(\cdot,X_i)-k_m(\cdot,X_j)\right)\nonumber
\end{eqnarray}
is a self-adjoint positive definite operator on $\Cal{H}_m$ that is equivalent (in the above mentioned sense) to $\widehat{\Omega}_m$, which is a $U$-statistic estimator of $\Omega_m$. 
Since $\Sigma_m$ and $\sigh_m$ are trace-class (see Proposition~\ref{pro:approx}\emph{(iii)}) 
and self-adjoint, spectral theorem \citep[Theorems VI.16, VI.17]{Reed-80} yields that
\begin{equation}\Sigma_m=\sum^m_{i=1}\lambda_{m,i}\phi_{m,i}\ohm \phi_{m,i}\quad\text{and}\quad \sigh_m=\sum^m_{i=1}\lambdah_{m,i}\phih_{m,i}\ohm \phih_{m,i},\nonumber
\end{equation}
where $(\lambda_{m,i})^m_{i=1}\subset\bb{R}^+$ (\emph{resp.} $(\lambdah_{m,i})^m_{i=1}\subset\bb{R}^+$) and $(\phi_{m,i})^m_{i=1}$ (\emph{resp.} $(\phih_{m,i})^m_{i=1}$) are the 
eigenvalues and eigenvectors of $\Sigma_m$ (\emph{resp.} $\sigh_m$).
We will assume that
\begin{itemize}
\item[($A_5$)] The eigenvalues $(\lambda_{m,i})^m_{i=1}$ (\emph{resp.} $(\lambdah_{m,i})^m_{i=1}$) of $\Sigma_m$ (\emph{resp.} $\sigh_m$) are simple, positive
and without any loss of generality, they satisfy a decreasing rearrangement, i.e., $\lambda_{m,1}>\lambda_{m,2}>\cdots$ 
(\emph{resp}. $\lambdah_{m,1}>\lambdah_{m,2}>\cdots$) $\Lambda$-a.s. (\emph{resp.} $\Lambda\times \bb{P}$-a.s.).
\end{itemize}
Based on $(A_5)$, a low-dimensional representation of $X_i\in\X$ can be obtained as
$$(\phih_{m,1}(X_i),\ldots,\phih_{m,\ell}(X_i))^\top\in\bb{R}^\ell,$$
$\ell\le m$, $i\in[n]$. The orthogonal projection operators onto the $\ell$-eigenspaces of $\Sigma_m$ and $\sigh_m$ are given by $P_\ell(\Sigma_m)=\sum_{i=1}^\ell\phi_{m,i}\ohm\phi_{m,i}$ and $P_\ell(\sigh_m)=\sum_{i=1}^\ell\phih_{m,i}\ohm\phih_{m,i}$, respectively. 
Since $(\lambdah_{m,i},\phih_{m,i})^\ell_{i=1}$ for $\ell\le m$ is a subset of the eigensystem of $\sigh_m$ (which is equivalent to the $m\times m$ matrix $\widehat{\Omega}_m$), the associated time complexity of finding this set scales as $O(m^2\ell+m^2n)$, where $O(m^2n)$ is the complexity of computing $\widehat{\Omega}_m$. This implies that RF-EKPCA is computationally cheaper than EKPCA if $m<\sqrt{n\ell}$ for $\ell\le n$, i.e., $m=o(\sqrt{n\ell})$ as $n,\ell\rightarrow\infty$. 

\section{Computational vs.~Statistical Trade-off}
\label{Sec:results}
The main goal of this paper is to 
investigate whether the above mentioned computational saving achieved by RF-EKPCA is obtained at the cost of statistical ``efficiency" or not. 
To this end, we investigate this question by using the reconstruction error as a measure of statistical performance.
To elaborate, in linear PCA, the quality of reconstruction after projecting a random variable $X\in\R^d$ onto the span of top $\ell$ eigenvectors of $\Sigma$ is captured by the reconstruction error, given by
\begin{equation*}\bb{E}_{X\sim \Pb}\left\Vert (X-\mu)-\sum^\ell_{i=1}\langle (X-\mu),\phi_i\rangle_2\phi_i\right\Vert^2_2,\nonumber 
\end{equation*}
where $(\phi_i)_i$ are the eigenvectors of $\Sigma=\bb{E}[XX^\top]-\mu\mu^\top$ with $\mu:=\bb{E}[X]$. Since $(\phi_i)_i$ form an orthonormal basis in $\R^d$, the above consideration makes sense and clearly, the choice of $\ell=d$ yields zero error. Since KPCA generalizes linear PCA---the choice of $k(x,y)=\langle x,y\rangle_2$ reduces kernel PCA to linear PCA---, it is natural to consider the reconstruction error in KPCA and EKPCA to be
\begin{equation}
\bb{E}_{X\sim\Pb}\left\Vert \overline{k}(\cdot,X)-\sum^\ell_{i=1}\langle \xi(X),\zeta_i\rangle_\Hk\zeta_i\right\Vert^2_\Hk\label{Eq:rec-error-kpca}
\end{equation}
with $\xi(X)=\overline{k}(\cdot,X)$, $\zeta_i=\phi_i$ and $\xi(X)=\widetilde{k}(\cdot,X)$, $\zeta_i=\phih_i$ respectively, 
where $(\phi_i)_i$ and $(\phih_i)_i$ are the orthonormal eigenfunctions of $\Sigma$ and $\sigh$ given in \eqref{Eq:cov-eig} and \eqref{Eq:emp-sig}, corresponding to the eigenvalues $(\lambda_i)_i$ and $(\lambdah_i)_i$ satisfying $(A_2)$ and $(A_3)$ respectively.
Here for any $x\in\Cal{X}$, $$\overline{k}(\cdot,x)=k(\cdot,x)-\intx k(\cdot,x)\,d\Pb(x)\quad\text{and}\quad\widetilde{k}(\cdot,x)=k(\cdot,x)-\frac{1}{n}\sum^n_{i=1} k(\cdot,X_i).$$ Since empirical mean is used in empirical linear PCA to find principal components, we used $\widetilde{k}$ in \eqref{Eq:rec-error-kpca} to measure the performance of EKPCA. However, similar performance measure as in \eqref{Eq:rec-error-kpca} is not possible for RF-KPCA and RF-EKPCA as the orthonormal eigenvectors of $\Sigma_m$ and $\sigh_m$ belong to $\Cal{H}_m$ (isometrically isomorphic to $\R^m$) while $\overline{k}(\cdot,X)$ and $\widetilde{k}(\cdot,X)$ belong to $\Hk$, which means the notion of respectively projecting $\overline{k}(\cdot,X)$ and $\widetilde{k}(\cdot,X)$ onto $(\phi_{m,i})_i$ and $(\phih_{m,i})_i$ is vacuous. However, since both $\Hk$ and $\Cal{H}_m$ are subspaces of $L^2(\Pb)$, it is natural to consider the reconstruction error in $L^2(\Pb)$-norm so that the behaviors of EKPCA and RF-EKPCA can be compared.
%

To this end, define 
an inclusion operator (up to a constant)
\begin{equation*}\id:\Cal{H}\rightarrow \lp,\quad f\mapsto f-f_\bb{P},\nonumber
\end{equation*}
where $f_\bb{P}:=\intx f(x)\,d\bb{P}(x)$. It can be shown (see Proposition~\ref{pro:id}) 
that $\id^*:\lp\rightarrow \Cal{H},\,\, f\mapsto \intx k(\cdot,x)f(x)\,d\bb{P}(x)-m_\bb{P}f_\bb{P}$ and $\Sigma=\id^*\id.$
Similarly, we define an approximation operator \begin{equation*}\frak{A}:\Cal{H}_m\rightarrow \lp,\quad f=\sum^m_{i=1}\beta_i\varphi_i\mapsto \sum^m_{i=1}\beta_i(\vp_i-\vp_{i,\bb{P}})=f-f_\bb{P},\nonumber
\end{equation*}
where $\vp_{i,\bb{P}}:=\intx \vp_i(x)\,d\bb{P}(x).$ It can be shown (see Proposition~\ref{pro:approx}) that
$\frak{A}^*:\lp\rightarrow\Cal{H}_m,\,\,f\mapsto\sum_{i=1}^m(\inner{f}{\vp_i}_{\lp}-f_\Pb\vp_{i,\Pb})\vp_i$ and $\Sigma_m=\frak{A}^*\frak{A}.$
Based on these operators, we first consider alternate notions of reconstruction error in $\lp$ for KPCA, EKPCA, RF-KPCA and RF-EKPCA 
in Section~\ref{subsec:notions} 
and then present results comparing the statistical behavior of EKPCA and RF-EKPCA in Sections~\ref{subsec:ep} and \ref{subsec:pe}. 
%
\subsection{Alternate Notions of Reconstruction Error}\label{subsec:notions}
Let $(\psi_i)_i\subset\Cal{H}$ and $(\mu_i)_i\subset\lp$ be arbitrary collections of functions. We now define two notions of reconstruction error in $\lp$ as follows: \textit{Reconstruct and Embed (R-E)}
\begin{eqnarray}
\eu{R}(\psi_1,\ldots,\psi_\ell)&{}={}&\bb{E}_{X\sim\Pb}\norm{\id \overline{k}(\cdot,X)-\id\left(\sum_{i=1}^{\ell}\inner{\overline{k}(\cdot,X)}{\psi_i}_{\Hk}\psi_i\right)}^2_{L^2(\Pb)}\nonumber\\
&{}={}&\bb{E}_{X\sim\Pb}\norm{\id \overline{k}(\cdot,X)-\id\left(\sum_{i=1}^{\ell}\psi_i\otimes_\Cal{H}\psi_i\right)\overline{k}(\cdot,X)}^2_{L^2(\Pb)},\nonumber
\end{eqnarray}
where $\overline{k}(\cdot,X)$ is first \emph{projected}\footnote{In general, $P_\psi:=\sum^\ell_{i=1}\psi_i\otimes_\Cal{H}\psi_i$ is not a projection operator since $P^2_\psi\ne P_\psi$.} and reconstructed along $(\psi_i)_{i\in[\ell]}$ in $\Cal{H}$, and then embedded into $\lp$ through $\id$; and \textit{Embed and Reconstruct (E-R)}
$$\eu{S}(\mu_1,\ldots,\mu_\ell)=\bb{E}_{X\sim\Pb}\norm{\id \overline{k}(\cdot,X)-\sum_{i=1}^{\ell}\inner{\id \overline{k}(\cdot,X)}{\mu_i}_{\lp}\mu_i}^2_{L^2(\Pb)},$$
where $\overline{k}(\cdot,X)$ is first embedded into $\lp$ through $\id$ and then projected and reconstructed along $\left(\mu_i\right)_{i\in[\ell]}$ in $\lp$.  
The following result (proved in Section~\ref{subsec:solution}) 
shows that the minimizers of $\eu{R}$ and $\eu{S}$ over any $(\psi_i)_{i\in[\ell]}$ in $\Cal{H}$ and $(\mu_i)_{i\in[\ell]}$ in $\lp$ are precisely the PCA solutions in $\Cal{H}$ and $\lp$, respectively.
%
\begin{proposition}\label{pro:solution}
Suppose $(A_1)$ and $(A_2)$ hold. 
Then the following hold:\vspace{2mm}\\
(i) $(\phi_1,\ldots,\phi_\ell)=\arg\inf\{\eu{R}(\psi_1,\ldots,\psi_\ell):(\psi_i)^\ell_{i=1}\subset \Cal{H}\}$;\vspace{1mm}\\
(ii) $\left(\frac{\id\phi_1}{\sqrt{\lambda_1}},\ldots,\frac{\id\phi_\ell}{\sqrt{\lambda_\ell}}\right)=\arg\inf\{\eu{S}(\mu_1,\ldots,\mu_\ell):(\mu_i)^\ell_{i=1}\subset \lp\}$;\vspace{1mm}\\
(iii) $R_{\Sigma,\ell}:=\eu{R}(\phi_1,\ldots,\phi_\ell)=\eu{S}\left(\frac{\id\phi_1}{\sqrt{\lambda_i}},\ldots,\frac{\id\phi_\ell}{\sqrt{\lambda_\ell}}\right):=S_{\Sigma,\ell}$;
\vspace{1mm}\\
(iv) $R_{\Sigma,\ell}\le \Vert \Sigma\Vert_{\OPH}\bb{E}_{X\sim\bb{P}}\left\Vert \overline{k}(\cdot,X)-\sum_{i=1}^{\ell}\inner{\overline{k}(\cdot,X)}{\phi_i}_{\Hk}\phi_i\right\Vert^2_\Cal{H}$.
\end{proposition}
\begin{remark}\label{Rem:solution}
(i) Proposition~\ref{pro:solution}(i) shows that the minimizer of the R-E reconstruction error is precisely the KPCA solution which is obtained by minimizing $\bb{E}\Vert\overline{k}(\cdot,X)-\sum^\ell_{i=1}\inner{\overline{k}(\cdot,X)}{\psi_i}_\Cal{H}\psi_i\Vert^2_\Cal{H}$ over $(\psi_i)_{i\in[\ell]}\subset\Cal{H}$, i.e., R-E provides an alternate interpretation for KPCA.\vspace{1.5mm}\\
(ii) Since minimizing $\eu{S}$ is equivalent to performing PCA in $\lp$, it follows that minimizers of $\eu{S}$ are the eigenfunctions of $\id\id^*$, which are precisely $(\frac{\id\phi_i}{\sqrt{\lambda_i}})_{i\in[\ell]}$ with $(\lambda_i,\phi_i)_{i\in[\ell]}$ being the eigenpairs of $\Sigma=\id^*\id$ (see the proof for details).\vspace{1.5mm}\\
(iii) Proposition~\ref{pro:solution}(iv) implies that $R_{\Sigma,\ell}$ is a weaker measure of reconstruction error than the one defined in \eqref{Eq:rec-error-kpca}, with the latter matching with the reconstruction error of linear PCA when $k(x,y)=\langle x,y\rangle_2$. More precisely, in Theorems~\ref{thm:rff main thm metric 1} and \ref{thm:rff main thm metric 2}, we will show that $R_{\Sigma,\ell}=S_{\Sigma,\ell}=\sum_{i>\ell}\lambda^2_i$ while 
the reconstruction error in \eqref{Eq:rec-error-kpca} 
behaves as $\sum_{i>\ell}\lambda_i$, which clearly establishes $R_{\Sigma,\ell}$ to be weaker than the one in \eqref{Eq:rec-error-kpca}. Later, in Section~\ref{subsecsec:othernorm}, we will generalize $\eu{R}$ such that KPCA's reconstruction error also behaves as $\sum_{i>\ell}\lambda_i$.\QEDA
\end{remark}
Similar to KPCA, in the following result (proved in Section~\ref{subsec:interpret}), 
we present alternate interpretations for EKPCA, RF-KPCA and RF-EKPCA as minimization of appropriate R-E reconstruction errors. 
\begin{proposition}\label{pro:interpret}
Suppose $(A_1)-(A_5)$ hold. Define $$\eu{T}_{A,H}(P):=\norm{A^{1/2}(I-P)A^{1/2}}^2_{\Cal{L}^2(H)}$$ where $A:H\rightarrow H$ is a positive self-adjoint Hilbert-Schmidt operator on $H$ with $H$ being a separable Hilbert space. Then the following hold.\vspace{2mm}\\
(i) $\eu{R}(\psi_1,\ldots,\psi_\ell)=\eu{T}_{\Sigma,\Cal{H}}(P_\psi)$ where $P_\psi:=\sum^\ell_{i=1}\psi_i\oh\psi_i$;\vspace{1.5mm}\\
(ii) (KPCA) $(\phi_1,\ldots,\phi_\ell)=\arg\inf\{\eu{T}_{\Sigma,\Cal{H}}(P_\psi):(\psi_i)_{i\in[\ell]}\subset\Cal{H}\}$;\vspace{1.5mm}\\
(iii) (EKPCA) $(\phih_{1},\ldots,\phih_{\ell})=\arg\inf\{\eu{T}_{\sigh,\Cal{H}}(P_\psi):(\psi_i)_{i\in[\ell]}\subset\Cal{H}\}$;\vspace{1.5mm}\\
(iv) (RF-KPCA) $(\phi_{m,1},\ldots,\phi_{m,\ell})=\arg\inf\{\eu{T}_{\sig_m,\Cal{H}_m}(P_\tau):(\tau_i)_{i\in[\ell]}\subset\Cal{H}_m\}$ where $P_\tau=\sum^\ell_{i=1}\tau_i\ohm\tau_i$;\vspace{1.5mm}\\
(v) (RF-EKPCA) $(\phih_{m,1},\ldots,\phih_{m,\ell})=\arg\inf\{\eu{T}_{\sigh_m,\Cal{H}_m}(P_\tau):(\tau_i)_{i\in[\ell]}\subset\Cal{H}_m\}$.
\end{proposition}
\begin{remark}\label{Rem:solution1} (i) Proposition~\ref{pro:interpret}(i) provides an alternate expression for the R-E error in terms of the Hilbert-Schmidt norm of a certain self-adjoint operator on $\Cal{H}$ and Proposition~\ref{pro:interpret}(ii) is obvious from Proposition~\ref{pro:solution}(i). Since this alternate expression depends on $\Sigma$, by replacing $\Sigma$ with $\sigh$ (resp. $\Sigma_m$, $\sigh_m$), alternate formulation for EKPCA (resp. RF-KPCA, RF-EKPCA) can be provided as in Proposition~\ref{pro:interpret}(iii)-(v).\vspace{1.5mm}\\
(ii) Based on Proposition~\ref{pro:interpret}(iii), EKPCA can be interpreted as the minimizer of the following empirical R-E reconstruction error, 
\begin{eqnarray}
\widehat{\eu{R}}(\psi_1,\ldots,\psi_\ell)&{}={}&
\frac{1}{n}\sum^n_{j=1}\norm{\widehat{\id}\widetilde{k}(\cdot,X_j)-\widehat{\id}\left(\sum^\ell_{i=1}\inner{\widetilde{k}(\cdot,X_j)}{\psi_i}_\Cal{H}\psi_i\right)}^2_{L^2(\bb{P}_n)}\nonumber\\
&{}={}&\eu{T}_{\sigh,\Cal{H}}(P_\psi),\nonumber
\end{eqnarray}
where $\widehat{\id}$ and $L^2(\bb{P}_n)$ are defined in Proposition~\ref{pro:interpret2}(iii). Similar interpretation can be provided for RF-EKPCA (resp. RF-KPCA) by considering approximate empirical (resp. population) R-E reconstruction error which is defined by replacing $\widehat{\id}$ by $\widehat{\A}$ (resp. $\A$) and $\widetilde{k}$ by $\widetilde{k}_m$ (resp. $\overline{k}_m$), respectively.\vspace{1.5mm}\\
(iii) The minimal value of $\eu{T}_{A,H}$ for $(A,H)=(\sigh,\Cal{H})$, $(\Sigma_m,\Cal{H}_m)$ and $(\sigh_m,\Cal{H}_m)$ can be shown to be $\sum_{i>\ell}\lambdah^2_i$, $\sum_{i>\ell}\lambda^2_{m,i}$, and $\sum_{i>\ell}\lambdah^2_{m,i}$, respectively, all of which are closely related to $\sum_{i>\ell}\lambda^2_i$ which is the minimum value of $\eu{T}_{\Sigma,\Cal{H}}$.\QEDA
\end{remark}
The following result (proved in Section~\ref{subsec:interpret2}) 
is similar to Proposition~\ref{pro:interpret} and shows the relation between the EKPCA (\emph{resp.} RF-KPCA, RF-EKPCA) solution and the minimizer of appropriate empirical versions of the E-R reconstruction error.
\begin{proposition}\label{pro:interpret2}
Suppose $(A_1)-(A_5)$ hold. Define $$\eu{V}_{A,H}(P):=\norm{(I-P)^2A}^2_{\Cal{L}^2(H)}$$ where $A:H\rightarrow H$ is a positive self-adjoint Hilbert-Schmidt operator on $H$ with $H$ being a separable Hilbert space. Then the following hold.\vspace{2mm}\\
(i) $\eu{S}(\mu_1,\ldots,\mu_\ell)=\eu{V}_{\id\id^*,\lp}(P_\mu)$ where $P_\mu:=\sum^\ell_{i=1}\mu_i\ol\mu_i$;\vspace{1.5mm}\\
(ii) (KPCA) $\left(\frac{\id\phi_1}{\sqrt{\lambda_1}},\ldots,\frac{\id \phi_\ell}{\sqrt{\lambda_\ell}}\right)=\arg\inf\{\eu{V}_{\id\id^*,\lp}(P_\mu):(\mu_i)_{i\in[\ell]}\subset\lp\}$;\vspace{1.5mm}\\
(iii) (EKPCA) 
\begin{eqnarray}
\left(\widehat{\id}\phih_{1}/\sqrt{\lambdah_1},\ldots,\widehat{\id}\phih_{\ell}/\sqrt{\lambdah_\ell}\right)=\arg\inf\left\{
\eu{V}_{\widehat{\id}\widehat{\id}^*,L^2(\bb{P}_n)}\left(Q_{\mu,n}\right):(\mu_i)_{i\in[\ell]}\subset L^2(\bb{P}_n)\right\},
\nonumber
\end{eqnarray}
where $Q_{\mu,n}:=\sum^\ell_{i=1}\mu_i\otimes_{L^2(\bb{P}_n)}\mu_i$, $L^2(\bb{P}_n):=\left\{f:\frac{1}{n}\sum^n_{i=1}f^2(X_i)<\infty\right\}$, and $\widehat{\id}:\Cal{H}\rightarrow L^2(\bb{P}_n),\,\,\,f\mapsto \sqrt{\frac{n}{n-1}}\left(f-\frac{1}{n}\sum^n_{i=1}f(X_i)\right);$\vspace{1.5mm}\\
(iv) (RF-KPCA) \begin{eqnarray}
\left(\A\phi_{m,1}/\sqrt{\lambda_{m,1}},\ldots,\A\phi_{m,\ell}/\sqrt{\lambda_{m,\ell}}\right)&{}={}&\arg\inf
_{(\mu_i)_{i\in[\ell]}\subset L^2(\bb{P})}
\eu{V}_{\A\A^*,L^2(\bb{P})}\left(P_\mu\right);
\nonumber
\end{eqnarray}
\vspace{-1mm}
(v) (RF-EKPCA) \begin{eqnarray}
\left(\widehat{\A}\phih_{m,1}/\sqrt{\lambdah_{m,1}},\ldots,\widehat{\A}\phih_{m,\ell}/\sqrt{\lambdah_{m,\ell}}\right)&{}={}&\arg\inf_{(\mu_i)_{i\in[\ell]}\subset L^2(\bb{P}_n)}
\eu{V}_{\widehat{\A}\widehat{\A}^*,L^2(\bb{P}_n)}\left(Q_{\mu,n}\right);
\nonumber
\end{eqnarray}
where $\widehat{\A}:\Cal{H}_m\rightarrow L^2(\bb{P}_n)$, $f\mapsto \sqrt{\frac{n}{n-1}}\left(f-\frac{1}{n}\sum^n_{i=1}f(X_i)\right)$.
\end{proposition}
%
\begin{remark}\label{Rem:solution2}
(i) Proposition~\ref{pro:interpret2}(ii) is obvious from Proposition~\ref{pro:interpret2}(i) and Proposition~\ref{pro:solution}(ii). Using Proposition~\ref{pro:interpret2}(ii), alternate interpretation for EKPCA (resp. RF-EKPCA) can be obtained by replacing $\id$ with its empirical version $\widehat{\id}$ (resp. $\widehat{\A}$) and $\bb{P}$ with its empirical version $\bb{P}_n$.\vspace{1.5mm}\\
(ii) $\eu{V}_{\widehat{\id}\widehat{\id}^*,L^2(\bb{P}_n)}(\sum^\ell_{i=1}\mu_i\otimes_{L^2(\bb{P}_n)}\mu_i)$ in Proposition~\ref{pro:interpret2}(iii) can be shown to be equal to the following empirical E-R reconstruction error, 
$$\widehat{\eu{S}}(\mu_1,\ldots,\mu_\ell)=\frac{1}{n}\sum^n_{j=1}\norm{\widehat{\id}\widetilde{k}(\cdot,X_j)-\sum^\ell_{i=1}\inner{\widehat{\id}\widetilde{k}(\cdot,X_j)}{\mu_i}_{L^2(\bb{P}_n)}\mu_i}^2_{L^2(\bb{P}_n)}.$$ Similar interpretation can be provided for RF-EKPCA by replacing $\widehat{\id}$ and $\widetilde{k}$ by $\widehat{\A}$ and $\widetilde{k}_m$, respectively.\vspace{1.5mm}\\
(iii) Similar to $\eu{T}_{A,H}$, the minimal value of $\eu{V}_{A,H}$ for $(A,H)=(\widehat{\id}\widehat{\id}^*,L^2(\bb{P}_n))$, $(\A\A^*,L^2(\bb{P}))$ and $(\widehat{\A}\widehat{\A}^*,L^2(\bb{P}_n))$ can be shown to be $\sum_{i>\ell}\lambdah^2_i$, $\sum_{i>\ell}\lambda^2_{m,i}$, and $\sum_{i>\ell}\lambdah^2_{m,i}$, respectively, which matches with their R-E counterparts.\QEDA
\end{remark}
Using these alternate interpretations of KPCA and its variants, in the following sections, we will investigate and compare the performances of EKPCA and RF-EKPCA in R-E and E-R settings.
\begin{subsection}{Reconstruct and Embed (R-E)}\label{subsec:ep}
The performance of EKPCA and RF-EKPCA can be measured by how well the output of these algorithms (i.e., the reconstructed functions) approximate $\id k(\cdot,X)$ in $\lp$ in expectation.  
To elaborate, since the principal components outputted by EKPCA and RF-EKPCA are $\langle\widetilde{k}(\cdot,X),\phih_i\rangle_{\Cal{H}}$ and $\langle\widetilde{k}_m(\cdot,X),\phih_{m,i}\rangle_{\Cal{H}_m}$, $i\in[\ell]$ along the directions $(\phih_i)^\ell_{i=1}$ and $(\phih_{m,i})^\ell_{i=1}$ respectively, we compare their $\lp$-embedded versions to $\id\overline{k}(\cdot,X)$, as defined below:
\begin{align}
R_{\sigh,\ell}&=\bb{E}_{X\sim\Pb}\norm{\id \overline{k}(\cdot,X)-\id\left(\sum_{i=1}^{\ell}\inner{\widetilde{k}(\cdot,X)}{\phih_i}_{\Hk}\phih_i\right)}^2_{L^2(\Pb)},\label{Eq:emp-error-pe}
\end{align}
\begin{align}
R_{\Sigma_m,\ell}&=\bb{E}_{X\sim\Pb}\norm{\id \overline{k}(\cdot,X)-\frak{A}\left(\sum_{i=1}^{\ell}\inner{\overline{k}_m(\cdot,X)}{\phi_{m,i}}_{\Hk_m}\phi_{m,i}\right)}^2_{L^2(\Pb)},\label{Eq:pop-rf-error-pe}
\end{align}
and
\begin{align}
R_{\sigh_m,\ell}&=\bb{E}_{X\sim\Pb}\norm{\id \overline{k}(\cdot,X)-\frak{A}\left(\sum_{i=1}^{\ell}\inner{\widetilde{k}_m(\cdot,X)}{\phih_{m,i}}_{\Hk_m}\phih_{m,i}\right)}^2_{L^2(\Pb)}.\label{Eq:emp-rf-error-pe}
\end{align}
Note that $\eu{T}_{\Sigma,\Cal{H}}(P_\ell(\Sigma))$ 
measures the performance of KPCA where $P_\ell(\Sigma)=\sum^\ell_{i=1}\phi_i\oh\phi_i$. Therefore $\eu{T}_{\Sigma,\Cal{H}}(P_\ell(\sigh))$ can be considered as a performance measure of EKPCA with $P_\ell(\sigh)=\sum^\ell_{i=1}\phih_i\oh\phih_i$. Similarly, $\eu{T}_{\Sigma_m,\Cal{H}_m}(P_\ell(\Sigma_m))$ and $\eu{T}_{\Sigma_m,\Cal{H}_m}(P_\ell(\sigh_m))$ can be used as performance measures of RF-KPCA and RF-EKPCA, respectively. However, in $\eu{T}_{\Sigma_m,\Cal{H}_m}$, the reconstructed function is $\A \overline{k}_m(\cdot,X)$ in contrast to $\id \overline{k}(\cdot,X)$ as in $\eu{T}_{\Sigma,\Cal{H}}$. In order to reconstruct the same function in all the algorithms, we consider the reconstruction errors defined in \eqref{Eq:emp-error-pe}-\eqref{Eq:emp-rf-error-pe}. The following result (proved in Section~\ref{subsec:relation}) 
shows that the reconstruction errors defined in \eqref{Eq:emp-error-pe}--\eqref{Eq:emp-rf-error-pe} are statistically equivalent to $\eu{T}_{\Sigma,\Cal{H}}(P_\ell(\sigh))$, $\eu{T}_{\Sigma_m,\Cal{H}_m}(P_\ell(\Sigma_m))$ and $\eu{T}_{\Sigma_m,\Cal{H}_m}(P_\ell(\sigh_m))$.
\begin{theorem}\label{pro:relation}
Under the assumptions $(A_1)-(A_5)$, the following hold:\vspace{2mm}\\
(i) $R_{\sigh,\ell}\lesssim_{\bb{P}^n} \eu{T}_{\Sigma,\Cal{H}}(P_\ell(\sigh))\lesssim_{\bb{P}^n}  R_{\sigh,\ell}+\frac{1}{n}$;\vspace{2mm}\\
(ii) $\left(\sqrt{R_{\Sigma_m,\ell}}-\frac{1}{\sqrt{m}}\right)^2\lesssim_{\Lambda^m} \eu{T}_{\Sigma_m,\Cal{H}_m}(P_\ell(\Sigma_m))\lesssim_{\Lambda^m}R_{\Sigma_m,\ell}+\frac{1}{m};$\vspace{2mm}\\
(ii) $\left(\sqrt{R_{\sigh_m,\ell}}-\frac{1}{\sqrt{m}}\right)^2\lesssim_{\bb{P}^n\times\Lambda^m} \eu{T}_{\sigh_m,\Cal{H}_m}(P_\ell(\sigh_m))\lesssim_{\bb{P}^n\times\Lambda^m}R_{\sigh_m,\ell}+\frac{1}{n}+\frac{1}{m}.$
\end{theorem}
In the above result, the term $\frac{1}{n}$ is the error incurred by centering $k(\cdot,X)$ (\emph{resp.} $k_m(\cdot,X)$) around $\mph$ (\emph{resp.} $\mpmh$) instead of $\Mp$ (\emph{resp.} $\mpm$). The term $\frac{1}{m}$ is the approximation error incurred by approximating $\id\overline{k}(\cdot,X)$ by $\A\overline{k}_m(\cdot,X)$, i.e., $\bb{E}\norm{\id\overline{k}(\cdot,X)-\A\overline{k}_m(\cdot,X)}^2_{\lp}$. 
The following result, which is proved in Section~\ref{subsec:thm-bound}, 
provides a finite-sample bound on \eqref{Eq:emp-error-pe}--\eqref{Eq:emp-rf-error-pe}, using which convergence rates can be obtained.
\begin{theorem}\label{thm:rff main thm metric 1}
Suppose $(A_1)-(A_5)$ hold. For any $t>0$, define $\Cal{N}_\Sigma(t)=\emph{tr}(\Sigma(\Sigma+tI)^{-1})$. Then the following hold:\vspace{2mm}\\
(i)    \begin{equation*}
        R_{\Sigma,\ell}=\sum_{i>\ell}\lambda_i^2.\nonumber
    \end{equation*}
  (ii)  For any $\delta>0$ with $n\ge2\log\frac{2}{\delta}$ and $\frac{140\kappa}{n}\log\frac{16\kappa n}{\delta}\le t\le\norm{\Sigma}_\OPH$, 
    \begin{equation*}
    \Pb^n\left\{(X_i)_{i=1}^n:\sum_{i>\ell}\lambda_i^2\le R_{\sigh,\ell}\le 9\Cal{N}_\Sigma(t)(\lambda_{\ell+1}+t)^2+\frac{64\kappa^2\log\frac{2}{\delta}}{n}\right\}\ge1-3\delta.\nonumber
    \end{equation*}
    (iii) For any $\delta>0$ with $m\ge \left(2\vee \frac{1024\kappa^2}{\sum_{i>\ell}\lambda^2_i}\right)\log\frac{2}{\delta}$,
    \begin{equation*}
        \Lambda^m\left\{(\theta_i)_{i=1}^m:\frac{1}{4}\sum_{i>\ell}\lambda^2_{i}\le R_{\Sigma_m,\ell}\le4\sum_{i>\ell}\lambda_i^2+\frac{256\kappa^2\log\frac{2}{\delta}}{m}\right\}\ge1-6\delta.\nonumber
    \end{equation*}
    (iv) For any $\delta>0$ with $n\ge 2\log\frac{2}{\delta}$, $m\ge \left(2\vee \frac{1024\kappa^2}{\sum_{i>\ell}\lambda^2_i}\right)\log\frac{2}{\delta}$ and $\frac{140\kappa}{n}\log\frac{16\kappa n}{\delta}$ $\vee\frac{86\kappa}{m}\log\frac{16\kappa m}{\delta}\le t\le\frac{\norm{\Sigma}_{\OPH}}{3}$, with probability at least $1-12\delta$ over the choice of $\left((X_i)_{i=1}^n,(\theta_j)_{j=1}^m\right)$:
    \begin{equation*}
    \frac{1}{4}\sum_{i>\ell}\lambda^2_{i}\le R_{\sigh_m,\ell}\le162\Cal{A}_1(t) (\lambda_{\ell+1}+t)^2+\frac{640\kappa^2\log\frac{2}{\delta}}{3n}+\frac{256\kappa^2\log\frac{2}{\delta}}{m},\nonumber
\end{equation*}
where $\Cal{A}_1(t):=\Cal{N}_\Sigma(t)+\frac{16\kappa\log\frac{2}{\delta}}{tm}+\sqrt{\frac{8\kappa\Cal{N}_\Sigma(t)\log\frac{2}{\delta}}{tm}}$.
\end{theorem}
Since $\Sigma$ is trace class it is obvious that $\lambda_\ell\rightarrow 0$ and $\sum_{i>\ell}\lambda_i^2\rightarrow 0$ as $\ell\rightarrow\infty$. It therefore follows that $R_{\Sigma,\ell}\rightarrow 0$ and $R_{\Sigma_m,\ell}\rightarrow 0$ as $\ell,m\rightarrow\infty$. Further, by assuming a decay rate on $(\lambda_i)_i$, a convergence rate for $R_{\Sigma,\ell}$ and $R_{\Sigma_m,\ell}$ may be obtained. Note that up to constants, $R_{\Sigma,\ell}$ and $R_{\Sigma_m,\ell}$ will have the same statistical behavior if $m$ is chosen to be large enough that $\sum_{i>\ell}\lambda_i^2$ dominates $\frac{1}{m}$. As in Theorem \ref{thm:rff main thm metric 1}, the behavior of the empirical varieties depend on $t$ and $\Cal{N}_\Sigma(t)$. $\Cal{N}_\Sigma(t)$ is referred to as the effective dimension or degrees of freedom \citep{Caponnetto-07}, and captures the complexity of $\Hk$. Since $\Cal{N}_\Sigma(t)\lesssim \frac{1}{t}$ (better bound can be obtained if a certain decay rate for $(\lambda_i)_i$ is assumed), it is easy to see that $R_{\sigh,\ell}\rightarrow 0$ if $\ell,n\rightarrow 0$ and $n\lambda^2_\ell \rightarrow\infty$, and $R_{\sigh_m,\ell}\rightarrow 0$ if $\ell,m,n\rightarrow 0$ and $\lambda^2_\ell (m\wedge n)\rightarrow\infty$. However, in order to properly compare the behavior of EKPCA and RF-EKPCA to each other, as well as to their population counterparts, an assumption on the decay rate of $(\lambda_i)_i$ must be made, and the trade-off between $t$, $\lambda_\ell$ and $\Cal{N}_\Sigma(t)$ must be explored. The following corollary (proved in Section~\ref{sec:rff poly decay proof}) 
to Theorem \ref{thm:rff main thm metric 1} investigates the statistical behavior of EKPCA and RF-EKPCA in detail under the polynomial decay condition (exponential decay condition is analyzed in Corollary~\ref{rff exp decay corollary}) 
on the eigenvalues of $\Sigma$.
\begin{corollary}[Polynomial decay of eigenvalues]\label{rff poly decay corollary}
Suppose $\underbar{A}i^{-\alpha}\le\lambda_i\le\bar{A}i^{-\alpha}$ for $\alpha>1$ and $\underbar{A},\bar{A}\in(0,\infty)$. Let $\ell=n^{\frac{\theta}{\alpha}}$, $0<\theta\le\alpha$.
Then
\\\\
(i) $$n^{-2\theta(1-\frac{1}{2\alpha})}\lesssim R_{\Sigma,\ell}\lesssim n^{-2\theta(1-\frac{1}{2\alpha})}.$$
There exists $\tilde{n}\in\bb{N}$ such that for all $n>\tilde{n}$, the following hold:\vspace{2mm}\\
(ii) \[ n^{-2\theta(1-\frac{1}{2\alpha})}\lesssim R_{\sigh,\ell}\lesssim_{\Pb^n}\begin{cases} 
      n^{-2\theta(1-\frac{1}{2\alpha})},\qquad\hspace{2.5mm}\theta\le\frac{\alpha}{2\alpha-1} \\
     \frac{1}{n},\qquad\,\,\,\qquad\qquad\theta\ge\frac{\alpha}{2\alpha-1}
   \end{cases};
\]
(iii) For $0<\gamma\le1$ and $m=n^\gamma$, 
\begin{eqnarray*} n^{-2\theta(1-\frac{1}{2\alpha})}\mathds{1}_{\left\{\gamma\ge \theta\left(2-\frac{1}{\alpha}\right)\right\}}&{}\lesssim_{\Lambda^m}{}& R_{\Sigma_m,\ell}\lesssim_{\Lambda^m}\\
&{}{}&\begin{cases} 
      n^{-2\theta(1-\frac{1}{2\alpha})},\qquad\gamma\ge \theta\left(2-\frac{1}{\alpha}\right),\,\theta\le\frac{\alpha}{2\alpha-1} \\
      n^{-\gamma},\qquad\qquad\quad\gamma\le 1\wedge\theta\left(2-\frac{1}{\alpha}\right)
   \end{cases};
\end{eqnarray*}
(iv) For $0<\gamma\le1$ and $m=n^\gamma$, 
\begin{eqnarray*} n^{-2\theta(1-\frac{1}{2\alpha})}\mathds{1}_{\left\{\gamma\ge \theta\left(2-\frac{1}{\alpha}\right)\right\}}&{}\lesssim_{\Lambda^m}{}& R_{\sigh_m,\ell}\lesssim_{\Pb^n\times\Lambda^m}\\
&{}{}&\begin{cases} 
      n^{-2\theta(1-\frac{1}{2\alpha})},\qquad\gamma\ge  \theta\left(2-\frac{1}{\alpha}\right),\,\theta\le\frac{\alpha}{2\alpha-1} \\
      n^{-\gamma},\qquad\qquad\quad\gamma\le 1\wedge\theta\left(2-\frac{1}{\alpha}\right)
   \end{cases}.
\end{eqnarray*}
\end{corollary}

\begin{remark}\label{rff:rem1}
$(i)$  The condition $\alpha>1$ is required to ensure that $\Sigma$ is trace class.  Comparing the behavior of $R_{\sigh,\ell}$ to that of $R_{\Sigma,\ell}$ it is clear that EKPCA recovers optimal convergence rates (compared to that of KPCA) if $\ell$ grows to infinity not faster than $n^{1/(2\alpha-1)}$. Since the reconstruction error is based on $\ell$ eigenfunctions, the computational complexity of EKPCA behaves as $O(n^2\ell)=O(n^{2+\frac{\theta}{\alpha}})$. It is important to note that $0<\theta\le\frac{\alpha}{2\alpha-1}$ is the only useful region both computationally and statistically as $\theta>\frac{\alpha}{2\alpha-1}$ does not improve the statistical rates of EKPCA (than that achieved at $\theta=\frac{\alpha}{2\alpha-1}$) but increases its computational complexity.
%
%
\vspace{1mm}\\
$(ii)$ Comparing $R_{\sigh_m,\ell}$ with $R_{\sigh,\ell}$ it is clear that if $\ell$ grows to infinity not faster than $n^{1/(2\alpha-1)}$ and the number of random features $m$ grows sufficiently fast, then RF-EKPCA and EKPCA enjoy the same statistical behavior.  The rate at which the number of random features must grow depends on the growth of $\ell$ through $\theta$ and $\alpha$; the choice of $1\ge\gamma\ge\theta\left(2-\frac{1}{\alpha}\right)$ yields the same statistical behavior for RF-EKPCA, EKPCA, and KPCA. \vspace{1mm}\\
$(iii)$ The computational complexity of RF-EKPCA is given by $O(m^2\ell+m^2n)=O(n^{2\gamma+1})$ which is better than that of EKPCA if $\gamma<\frac{1}{2}+\frac{\theta}{2\alpha}$. This means, RF-EKPCA has a lower computational complexity with similar statistical behavior to that of EKPCA if $2\theta-\frac{\theta}{\alpha}\le\gamma<\frac{1}{2}+\frac{\theta}{2\alpha}$ and $\theta\le \frac{\alpha}{2\alpha-1}$ respectively, which implies $\theta< \frac{\alpha}{4\alpha-3}$. In other words, if $\ell$ grows at a lower order than $n^{1/(4\alpha-3)}$ and the number of random features are larger than $n^{\theta(2-\frac{1}{\alpha})}$, then RF-EKPCA enjoys computational superiority with no loss in statistical performance over that of EKPCA. 
On the other hand, if $\ell$ grows at an order faster than $n^{1/(4\alpha-3)}$ but not faster than $n^{1/(2\alpha-1)}$, it results in loss of computational advantage for RF-EKPCA while retaining the same statistical behavior to that of EKPCA---in fact, the rate in this regime is faster than in the previous regime of $\theta<\frac{\alpha}{4\alpha-3}$.
%
%
%
\end{remark}

\end{subsection}

\begin{subsection}{Embed and Reconstruct (E-R)}\label{subsec:pe}
In this section, we compare EKPCA and RF-EKPCA in terms of E-R reconstruction error. Since $(\widehat{\id}\phih_i/\sqrt{\lambdah_i})^\ell_{i=1}$ and $(\widehat{\A}\phih_{m,i}/\sqrt{\lambdah_{m,i}})^\ell_{i=1}$ are the minimizers of the empirical E-R reconstruction error (see Proposition~\ref{pro:interpret2}\emph{(iii),(v)}), it is natural to compare $$(\star)\quad\eu{S}(\id\phih_1/\sqrt{\lambdah_1},\ldots,\id\phih_\ell/\sqrt{\lambdah_\ell})\,\,\,\text{and}\,\,\,\eu{S}(\A\phih_{m,1}/\sqrt{\lambdah_{m,1}},\ldots,\A\phih_{m,\ell}/\sqrt{\lambdah_{m,\ell}}).$$
Instead, like in Section~\ref{subsec:ep}, we consider the following reconstruction errors since both EKPCA and RF-EKPCA provide the principal components for the empirical kernel functions, i.e., $\widetilde{k}(\cdot,X)$ and $\widetilde{k}_m(\cdot,X)$:
\begin{equation}
S_{\sigh,\ell}=\bb{E}_{X\sim\Pb}\norm{\id \overline{k}(\cdot,X)-\sum_{i=1}^{\ell}\inner{\id \widetilde{k}(\cdot,X)}{\frac{\id\phih_i}{\sqrt{\lambdah_i}}}_{\lp}\frac{\id\phih_i}{\sqrt{\lambdah_i}}}^2_{L^2(\Pb)},\label{Eq:emp-error-ep}
\end{equation}
\begin{equation}
S_{\Sigma_m,\ell}=\bb{E}_{X\sim\Pb}\norm{\id \overline{k}(\cdot,X)-\sum_{i=1}^{\ell}\inner{\A \overline{k}_m(\cdot,X)}{\frac{\A\phi_{m,i}}{\sqrt{\lambda_{m,i}}}}_{\lp}\frac{\A\phi_{m,i}}{\sqrt{\lambda_{m,i}}}}^2_{L^2(\Pb)},\label{Eq:pop-rf-error-ep}
\end{equation}
and
\begin{equation}
S_{\sigh_m,\ell}=\bb{E}_{X\sim\Pb}\norm{\id \overline{k}(\cdot,X)-\sum_{i=1}^{\ell}\inner{\A \widetilde{k}_m(\cdot,X)}{\frac{\A\phih_{m,i}}{\sqrt{\lambdah_{m,i}}}}_{\lp}\frac{\A\phih_{m,i}}{\sqrt{\lambdah_{m,i}}}}^2_{L^2(\Pb)}.\label{Eq:approx-emp-error-ep}
\end{equation}
Note that unlike in R-E where the reconstructed functions (before being embedded) are computable (see \eqref{Eq:emp-error-pe}--\eqref{Eq:emp-rf-error-pe}), in \eqref{Eq:emp-error-ep}--\eqref{Eq:approx-emp-error-ep} and $(\star)$, the reconstructed functions are not computable because of their dependence on unknown $\bb{P}$. 
First, in Theorem~\ref{thm:rff main thm metric 2} (proved in Section~\ref{subsec:proof:metric2}), 
we provide a finite-sample bound on the behavior of \eqref{Eq:emp-error-ep}--\eqref{Eq:approx-emp-error-ep}, which is then used in Theorem~\ref{Thm:variant-ER} (proved in Section~\ref{Subsec:proof-ER-variant}), 
to show that $(\star)$ is statistically equivalent to \eqref{Eq:emp-error-ep} and \eqref{Eq:approx-emp-error-ep}. 

\begin{theorem}\label{thm:rff main thm metric 2}
Suppose $(A_1)-(A_5)$ hold. For any $t>0$, define $\Cal{N}_\Sigma(t)=\emph{tr}(\Sigma(\Sigma+tI)^{-1})$. Then the following hold:\vspace{2mm}\\
    (i) \begin{equation*}
        S_{\Sigma,\ell}=\sum_{i>\ell}\lambda_i^2.\nonumber
    \end{equation*}
    (ii) For any $\delta>0$ with $n\ge2\log\frac{2}{\delta}$ and $\frac{140\kappa}{n}\log\frac{16\kappa n}{\delta}\le t\le\frac{\lambda_\ell}{3}$, with probability at least $1-11\delta$ over the choice of $(X_i)^n_{i=1}$,
    \begin{eqnarray*}
\sum_{i>\ell}\lambda_i^2\le S_{\sigh,\ell}&{}\lesssim{}& \Cal{N}_\Sigma(t)(\lambda_{\ell+1}+t)^2+\kappa^{5/2}\log\frac{2}{\delta}\left[\frac{\Cal{N}_\Sigma(t)}{n\sqrt{t}}\wedge\frac{\kappa^{3/2}}{nt}\right]\\
&{}{}&\qquad+\frac{\kappa^3(\kappa\wedge 1)\log^2\frac{3}{\delta}}{n^2t}+\frac{\kappa^2\log\frac{2}{\delta}}{n}.\nonumber
%
    \end{eqnarray*}
    (iii) For any $\delta>0$ with $m\ge \left(2\vee \frac{1024\kappa^2}{\sum_{i>\ell}\lambda^2_i}\right)\log\frac{2}{\delta}$,
    \begin{equation*}
        \Lambda^m\left\{(\theta_i)_{i=1}^m:\frac{1}{4}\sum_{i>\ell}\lambda_i^2\le S_{\Sigma_m,\ell}\le 4\sum_{i>\ell}\lambda_i^2+\frac{256\kappa^2\log\frac{2}{\delta}}{m}\right\}\ge1-6\delta.\nonumber
    \end{equation*}
    (iv) For any $\delta>0$ with $n\ge 2\log\frac{2}{\delta}$, $m\ge \left(2\vee \frac{1024\kappa^2}{\sum_{i>\ell}\lambda^2_i}\right)\log\frac{2}{\delta}$ and $\frac{140\kappa}{n}\log\frac{16\kappa n}{\delta}$ $\vee\frac{86\kappa}{m}\log\frac{16\kappa m}{\delta}\le t\le\frac{\lambda_\ell}{9}$, with probability at least $1-26\delta$ over the choice of $\left((X_i)_{i=1}^n,(\theta_j)_{j=1}^m\right)$:
    \begin{eqnarray*}
    \frac{1}{4}\sum_{i>\ell}\lambda_i^2\le S_{\sigh_m,\ell}&{}\lesssim{}&\Cal{A}_2(t) (\lambda_{\ell+1}+t)^2+\kappa^{5/2}\log\frac{2}{\delta}\left[\frac{\Cal{A}_2(t)}{n\sqrt{t}}\wedge \frac{\kappa^{3/2}}{nt}\right]\nonumber\\
    &{}{}&\qquad\qquad+\frac{\kappa^3(1\wedge \kappa)\log^2\frac{3}{\delta}}{n^2t}+\frac{\kappa^2\log\frac{2}{\delta}}{n}+\frac{\kappa^2\log\frac{2}{\delta}}{m},\nonumber
%
%
\end{eqnarray*}
where $\Cal{A}_2(t):=\frac{\kappa\log\frac{2}{\delta}}{tm}+\sqrt{\frac{\kappa\Cal{N}_\Sigma(t)\log\frac{2}{\delta}}{tm}}+\Cal{N}_\Sigma(t)$.
\end{theorem}
\begin{theorem}\label{Thm:variant-ER}
Under the assumptions $(A_1)-(A_5)$, the following hold:\vspace{2mm}\\
    (i) $\sum_{i>\ell}\lambda^2_i\le \eu{S}\left(\frac{\id\phih_1}{\sqrt{\lambdah_1}},\ldots,\frac{\id\phih_\ell}{\sqrt{\lambdah_\ell}}\right)\lesssim_{\bb{P}^n} S_{\sigh,\ell}+\frac{1}{n};$\vspace{1.5mm}\\
    (ii) For $m\gtrsim\frac{1}{\sum_{i>\ell}\lambda^2_i}$, $\sum_{i>\ell}\lambda^2_i\lesssim_{\Lambda^m} \eu{S}\left(\frac{\A\phi_{m,1}}{\sqrt{\lambda_{m,1}}},\ldots,\frac{\A\phi_{m,\ell}}{\sqrt{\lambda_{m,\ell}}}\right)\lesssim_{\Lambda^m} S_{\Sigma_m,\ell}+\frac{1}{m};$\vspace{1.5mm}\\
    (iii) For $m\gtrsim\frac{1}{\sum_{i>\ell}\lambda^2_i}$, $\sum_{i>\ell}\lambda^2_i\lesssim_{\Lambda^m} \eu{S}\left(\frac{\A\phih_{m,1}}{\sqrt{\lambdah_{m,1}}},\ldots,\frac{\A\phih_{m,\ell}}{\sqrt{\lambdah_{m,\ell}}}\right)\lesssim_{\bb{P}^n\times\Lambda^m} S_{\sigh_m,\ell}+\frac{1}{n}+\frac{1}{m}.$
\end{theorem}
\begin{remark}\label{Rem:Embed-Project}
(i)
A key difference (which may be an artifact of the proof technique) between Theorems~\ref{thm:rff main thm metric 1} and \ref{thm:rff main thm metric 2} is the upper bound on $t$. While $t$ is upper bounded by a constant in Theorem~\ref{thm:rff main thm metric 1}, $t$ is upper bounded by $\lambda_\ell$ (up to constants) in Theorem~\ref{thm:rff main thm metric 2}, which enforces a lower bound on $\lambda_\ell$. Since $\lambda_\ell\rightarrow 0$ as $\ell\rightarrow\infty$ and the lower bound on $\lambda_\ell$ converges to zero as $n\rightarrow\infty$, this enforces a constraint on $\lambda_\ell$ to not converge to zero too fast. In other words, it imposes a condition on $\ell$ to not grow too fast with $n$.\vspace{1mm}\\
(ii) The explicit universal constants, which are suppressed in (ii) and (iv) of Theorem~\ref{thm:rff main thm metric 2} for brevity, are provided in the proof. It is clear from (ii) and (iv) of Theorem~\ref{thm:rff main thm metric 2} that for $m$ large enough, both EKPCA and RF-EKPCA have similar statistical behavior---a similar observation was made in Theorem~\ref{thm:rff main thm metric 1}.\QEDA
\end{remark}
The following corollary (proved in Sections~\ref{sec:rff poly decay proof2}) 
to Theorem \ref{thm:rff main thm metric 2} investigates the statistical behavior of EKPCA and RF-EKPCA under the polynomial decay condition (exponential decay condition is analyzed in Corollary~\ref{rff exp decay corollary metric 2}) 
on the eigenvalues of $\Sigma$.
\begin{corollary}[Polynomial decay of eigenvalues]\label{rff poly decay corollary metric 2}
Suppose $\underbar{A}i^{-\alpha}\le\lambda_i\le\bar{A}i^{-\alpha}$ for $\alpha>1$ and $\underbar{A},\bar{A}\in(0,\infty)$. Let $\ell=n^{\frac{\theta}{\alpha}}$, $0<\theta\le\alpha$. Define $\frac{1}{\alpha'}:=\left(\frac{1}{\alpha}+\frac{1}{2}\right)\wedge 1$ and $\beta:=\frac{1}{2+\frac{1}{\alpha'}-\frac{1}{\alpha}}$. 
Then
\\\\
(i) $$n^{-2\theta(1-\frac{1}{2\alpha})}\lesssim S_{\Sigma,\ell}\lesssim n^{-2\theta(1-\frac{1}{2\alpha})}.$$
There exists $\tilde{n}\in\bb{N}$ such that for all $n>\tilde{n}$, the following hold:\vspace{2mm}\\
(ii) \[ n^{-2\theta(1-\frac{1}{2\alpha})} \lesssim S_{\sigh,\ell}\lesssim_{\Pb^n}\begin{cases} 
      n^{-2\theta(1-\frac{1}{2\alpha})},\qquad\theta\le\beta \\
     n^{-\left(1-\frac{\theta}{\alpha'}\right)},\qquad\beta\le\theta<1
   \end{cases};
\]
(iii) For $0<\gamma\le1$ and $m=n^\gamma$, 
\begin{eqnarray*}
 n^{-2\theta(1-\frac{1}{2\alpha})}\mathds{1}_{\left\{\gamma\ge \theta\left(2-\frac{1}{\alpha}\right)\right\}}&{}\lesssim_{\Lambda^m}{}& S_{\Sigma_m,\ell}\lesssim_{\Lambda^m}\\
 &{}{}&\begin{cases} 
      n^{-2\theta(1-\frac{1}{2\alpha})},\,\,\gamma\ge \theta\left(2-\frac{1}{\alpha}\right),\,\theta\le\frac{\alpha}{2\alpha-1} \\
      n^{-\gamma},\,\,\qquad\quad\gamma\le 1\wedge\theta\left(2-\frac{1}{\alpha}\right)
   \end{cases};
\end{eqnarray*}
(iv) For $0<\gamma\le1$ and $m=n^\gamma$, 
\begin{eqnarray*} 
&{}{}&n^{-2\theta(1-\frac{1}{2\alpha})}\mathds{1}_{\left\{\gamma\ge \theta\left(2-\frac{1}{\alpha}\right)\right\}}\lesssim_{\Lambda^m} S_{\sigh_m,\ell}\lesssim_{\Pb^n\times\Lambda^m}\\
&{}{}&\qquad\qquad\qquad\qquad\qquad\begin{cases} 
      n^{-2\theta(1-\frac{1}{2\alpha})},\quad\,\gamma\ge  \theta\left(2-\frac{1}{\alpha}\right),\,\theta\le\beta \\
      n^{-\left(1-\frac{\theta}{\alpha'}\right)},\quad \gamma\ge 1-\frac{\theta}{\alpha'},\,\gamma>\theta,\,\beta\le\theta<1\\
      n^{-\gamma},\quad\quad\,\,\,\theta<\gamma\le \left[1\wedge\theta\left(2-\frac{1}{\alpha}\right)\wedge \left(1-\frac{\theta}{\alpha'}\right)\right]
   \end{cases}.
\end{eqnarray*}
\end{corollary}
\begin{remark}\label{rff:rem:cor3}
(i) In Corollary~\ref{rff poly decay corollary metric 2}, note that $\beta=\frac{\alpha}{3\alpha-1}$ for $1<\alpha\le 2$ and $\beta=\frac{2}{5}$ for $\alpha\ge 2$ with the best convergence rate for $S_{\sigh,\ell}$ being attained at $\theta=\beta$ as the rate is a convex function of $\theta$. Clearly, from both computational and statistical view points, only the range of $0<\theta\le\beta$ is interesting and useful as $\theta>\beta$ yields similar/slower convergence rates with more computational complexity. This means, optimal convergence rates are obtained for EKPCA and RF-EKPCA for $\ell$ not growing faster than $n^{\theta/\alpha}$, $\theta\le\beta$ and $m\ge n^{\theta(2-\frac{1}{\alpha})}$. 
Arguing as in Remark~\ref{rff:rem1}$(iii)$, it can be shown that the computational complexity of RF-EKPCA is better than of EKPCA and with no loss in statistical performance if $\theta<\frac{\alpha}{4\alpha-3}\wedge\beta$ and $\gamma\ge \theta(2-\frac{1}{\alpha})$. 
\vspace{1mm}\\
(ii) More interesting observations can be made by comparing Corollaries~\ref{rff poly decay corollary} (\emph{resp.} Corollary~\ref{rff exp decay corollary}) and \ref{rff poly decay corollary metric 2} (\emph{resp.} Corollary~\ref{rff exp decay corollary metric 2}). 
First, EKPCA and RF-EKPCA have different upper asymptotic behaviors in R-E (Corollary~\ref{rff poly decay corollary} and Corollary~\ref{rff exp decay corollary}) 
and E-R (Corollary~\ref{rff poly decay corollary metric 2} and Corollary~\ref{rff exp decay corollary metric 2}). 
Particularly, while the reconstruction error rate improves with increase in $\theta$ in both the cases, it saturates beyond a certain $\theta$ in the case of R-E while it decreases in the case of E-R. This latter behavior is due to the inverse of empirical eigenvalues that appear in $S_{\sigh,\ell}$ and $S_{\sigh_m,\ell}$---as $\ell$ becomes large, then inverse of the empirical eigenvalues makes large contributions to the error, resulting in slower convergence rates. In the regimes of $\theta$ where R-E and E-R behave similarly (for both EKPCA and RF-EKPCA), we note that $\theta$ has a larger upper bound (i.e., $\ell$ can have faster growth) in R-E than in E-R, which again relates to the above mentioned issue of the inverse of empirical eigenvalues. Particularly, in the case of E-R, for $\alpha\le 2$, RF-EKPCA has better computational behavior than EKPCA if $\theta<\beta$ and $\gamma\ge \theta(2-\frac{1}{\alpha})$ while such a result holds for R-E for a wider range of $\theta$, i.e., $\theta<\frac{\alpha}{4\alpha-3}$, which means faster growth for $\ell$ is allowable for R-E without losing computational or statistical efficiency. On the other hand, for $\alpha\ge 2$, R-E and E-R behave similarly for $0<\theta\le\frac{\alpha}{4\alpha-3}$. Finally, we would like to highlight that the number of random features ($m$) needed in E-R and R-E so that their statistical behavior match that of KPCA is only a sufficient condition based on the upper bounds in Theorems~\ref{thm:rff main thm metric 1} and \ref{thm:rff main thm metric 2}. It is not clear whether this requirement on $m$ is sharp.\QEDA
\end{remark}
\end{subsection}
\begin{subsection}{Schatten norms}\label{subsecsec:othernorm}
So far, we have seen that the population reconstruction error in E-R and R-E, i.e., $R_{\Sigma,\ell}$ and $S_{\Sigma,\ell}$ behave as $\sum_{i>\ell}\lambda^2_i$, which is the squared $\ell_2$-norm of $\bm{\lambda}_\ell:=(\lambda_{\ell+1},\lambda_{\ell+2},\ldots)$. Of course, if we use the population reconstruction error defined in \eqref{Eq:rec-error-kpca}, it is easy to show that it behaves as $\sum_{i>\ell}\lambda_i$, which is the $\ell_1$-norm of $\bm{\lambda}_\ell$. But the reconstruction error defined in \eqref{Eq:rec-error-kpca} is not useful for our purpose because of the aforementioned technical issues and that is why we introduced E-R and R-E in Sections~\ref{subsec:ep} and \ref{subsec:pe}. In this section, we explore an extension of $R_{\Sigma,\ell}$ (similar extension holds for $S_{\Sigma,\ell}$ as well) which yields different norms of $\bm{\lambda}_\ell$. To this end, define
\begin{equation}
\eu{R}_s(\psi_1,\ldots,\psi_\ell)=\bb{E}\norm{\left(\id\id^*\right)^{-s/2}\left[\id \overline{k}(\cdot,X)-\id\left(\sum_{i=1}^{\ell}\inner{\overline{k}(\cdot,X)}{\psi_i}_{\Hk}\psi_i\right)\right]}^2_{\lp}\nonumber 
\end{equation}
for any $(\psi_1,\ldots,\psi_\ell)\subset\Cal{H}$ and $s\le 1$,
which is a weighted version of $\eu{R}$ with the error being weighted by $(\id\id^*)^{-s/2}$. Here $(\id\id^*)^{-1}$ is treated as the inverse of $\id\id^*$ restricted to $L^2(\bb{P})\backslash \text{Ker}(\id\id^*)$ where $\text{Ker}(\id\id^*)=\text{Ker}(\id^*)=\{f\in L^2(\bb{P}):f\,\,\text{is a constant a.s.--}\bb{P}\}$. The following result (similar to Propositions~\ref{pro:solution} and \ref{pro:interpret}), which is proved in Section~\ref{subsec:pro-schatten} 
shows that KPCA solution is the minimizer of $\eu{R}_s$ with the minimum value being $\sum_{i>\ell}\lambda^{2-s}_i$, i.e., $(2-s)$-Schatten norm of $\bm{\lambda}_\ell$ and that EKPCA, RF-KPCA and RF-EKPCA solutions are minimizers of appropriate empirical versions of $\eu{R}_s$.
\begin{proposition}\label{pro:schatten}
Suppose $(A_1)-(A_5)$ hold. For $s\le 1$, define $$\eu{T}_{s,A,H}(P):=\norm{A^{(1-s)/2}(I-P)A^{1/2}}^2_{\Cal{L}^2(H)},$$ where $A:H\rightarrow H$ is a positive self-adjoint Hilbert-Schmidt operator on $H$ with $H$ being a separable Hilbert space. Then the following hold.\vspace{2mm}\\
(i) $\eu{R}_s(\psi_1,\ldots,\psi_\ell)=\eu{T}_{s,\Sigma,\Cal{H}}(P_\psi)$ where $P_\psi:=\sum^\ell_{i=1}\psi_i\oh\psi_i$;\vspace{1mm}\\
(ii) (KPCA) $(\phi_1,\ldots,\phi_\ell)=\arg\inf\{\eu{T}_{s,\Sigma,\Cal{H}}(P_\psi):(\psi_i)_{i\in[\ell]}\subset\Cal{H}\}$;\vspace{1.5mm}\\
(iii) (EKPCA) $(\phih_{1},\ldots,\phih_{\ell})=\arg\inf\{\eu{T}_{s,\sigh,\Cal{H}}(P_\psi):(\psi_i)_{i\in[\ell]}\subset\Cal{H}\}$;\vspace{1.5mm}\\
(iv) (RF-KPCA) $(\phi_{m,1},\ldots,\phi_{m,\ell})=\arg\inf\{\eu{T}_{s,\sig_m,\Cal{H}_m}(P_\tau):(\tau_i)_{i\in[\ell]}\subset\Cal{H}_m\}$ where $P_\tau=\sum^\ell_{i=1}\tau_i\ohm\tau_i$;\vspace{1.5mm}\\
(v) (RF-EKPCA) $(\phih_{m,1},\ldots,\phih_{m,\ell})=\arg\inf\{\eu{T}_{s,\sigh_m,\Cal{H}_m}(P_\tau):(\tau_i)_{i\in[\ell]}\subset\Cal{H}_m\}$.
\end{proposition}
Proposition~\ref{pro:schatten} implies that the case of $s=1$ exactly recovers the original KPCA problem (see \eqref{Eq:rec-error-kpca}) and it follows from Lemma~\ref{lem:optima} 
that $\eu{R}_s(\phi_1,\ldots,\phi_\ell)=\sum_{i>\ell}\lambda^{2-s}_i$. With the intuition gained from Proposition~\ref{pro:schatten} and for the same reasons mentioned in Section~\ref{subsec:ep}, we measure the performance of EKPCA, RF-KPCA and RF-EKPCA as follows: Define $R_{\Sigma,\ell,s}:=\eu{R}_s(\phi_1,\ldots,\phi_\ell)$. 
$R_{\sigh,\ell,s}$ is the performance measure of EKPCA which is defined by replacing $\overline{k}$ and $\psi_i$ with $\widetilde{k}$ and $\phih_i$ respectively in $\eu{R}_s$. Similarly, the reconstruction error of RF-KPCA can be defined as 
\begin{equation*}
 R_{\Sigma_m,\ell,s}=\bb{E}_{X\sim\Pb}\norm{\left(\id\id^*\right)^{-s/2}\id \overline{k}(\cdot,X)-\left(\A\A^*\right)^{-s/2}\A P_\ell(\Sigma_m)\overline{k}_m(\cdot,X)
 }^2_{L^2(\Pb)}
\end{equation*}
with the performance of RF-EKPCA being measured by $R_{\sigh_m,\ell,s}$, which is defined by replacing $\overline{k}_m$ with $\widetilde{k}_m$ and $P_\ell(\Sigma_m)$ with $P_\ell(\sigh_m)$, where $P_\ell(\Sigma_m)=\sum^\ell_{i=1}\phi_{m,i}\ohm\phi_{m,i}$ and $P_\ell(\sigh_m)=\sum^\ell_{i=1}\phih_{m,i}\ohm\phih_{m,i}$. The following result provides the probabilistic behavior of these generalized reconstruction errors.
\begin{theorem}\label{thm:schatten}
Suppose $(A_1)-(A_5)$ hold. For any $t>0$, define $\Cal{N}_\Sigma(t)=\emph{tr}(\Sigma(\Sigma+tI)^{-1})$. Then the following hold:\vspace{2mm}\\
    (i) For any $s\le 1$, \begin{equation}
        R_{\Sigma,\ell,s}=\sum_{i>\ell}\lambda^{2-s}_i.\nonumber
    \end{equation}
    (ii) For $\frac{\log n}{n} \lesssim t\lesssim\norm{\Sigma}_\OPH$ and any $s\le 1$, 
    \begin{equation}
    \sum_{i>\ell}\lambda_i^{2-s}\le R_{\sigh,\ell,s}\lesssim_{\bb{P}^n} \frac{\Cal{N}_\Sigma(t)(\lambda_{\ell+1}+t)^2}{t^s}+\frac{1}{n}.\nonumber
    \end{equation}
    (iii) For $m\gtrsim \left[\left(\sum_{i>\ell}\lambda^{2-s}_{i}\right)^{\frac{2}{s-2}}\vee \left(\sum_{i>\ell}\lambda^{2-s}_{i}\right)^{\frac{2}{s}}\right]\mathds{1}_{[-2,0)}(s)+\frac{\mathds{1}_{\{0\}}(s)}{\sum_{i>\ell}\lambda^{2}_{i}}$, 
    \begin{equation}
        \sum_{i>\ell}\lambda^{2-s}_{i}\lesssim_{\Lambda^m} R_{\Sigma_m,\ell,s}\lesssim_{\Lambda^m} \sum_{i>\ell}\lambda_i^{2-s}+m^{s/2}\mathds{1}_{[-2,0)}(s)+\frac{1}{m}\mathds{1}_{\{0\}}(s).\nonumber
    \end{equation}
    (iv) For $m\gtrsim \left[\left(\sum_{i>\ell}\lambda^{2-s}_{i}\right)^{\frac{2}{s-2}}\vee \left(\sum_{i>\ell}\lambda^{2-s}_{i}\right)^{\frac{2}{s}}\right]\mathds{1}_{[-2,0)}(s)+\frac{\mathds{1}_{\{0\}}(s)}{\sum_{i>\ell}\lambda^{2}_{i}}$ and $\frac{\log n}{n}\vee\frac{\log m}{m}\lesssim t\lesssim\norm{\Sigma}_{\OPH}$,
    \begin{equation}
    \sum_{i>\ell}\lambda^{2-s}_{i}\lesssim_{\Lambda^m} R_{\sigh_m,\ell,s}\lesssim_{\Lambda^m\times\bb{P}^n} \frac{\Cal{A}_1(t) (\lambda_{\ell+1}+t)^2}{t^s}+\frac{1}{n}+\frac{\mathds{1}_{[-2,0)}(s)}{m^{-s/2}}+\frac{\mathds{1}_{\{0\}}(s)}{m},\nonumber
\end{equation}
where $\Cal{A}_1(t):=\Cal{N}_\Sigma(t)+\frac{1}{tm}+\sqrt{\frac{\Cal{N}_\Sigma(t)}{tm}}$.
\end{theorem}
\begin{remark}\label{rem:3}
(i) The restriction of $s\le 1$ for $R_{\Sigma,\ell,s}$ appears because $\Sigma$ is a trace class. On the other hand, the bounds for $R_{\Sigma_m,\ell,s}$ and $R_{\sigh_m,\ell,s}$ hold only for $s\in[-2,0]$. This could be an artifact of the analysis as the proof of these bounds involve bounding $\Vert (\id\id^*)^{-s/2}-(\tS\tS^*)^{-s/2}\Vert_{\Cal{L}^\infty(\lp)}$ by $\Vert \id\id^*-\tS\tS^*\Vert^{-s/2}_{\Cal{L}^\infty(\lp)}$ by using the operator monotonicity of the map $x\mapsto x^{-s/2}$ for $0\le -s/2\le 1$, i.e., $s\in[-2,0]$. Note that this range of $s$ yields weaker notions of reconstruction error as it corresponds to Schatten norms of order greater than 2, with the more interesting range being $[0,1]$. However, the current bounds for $R_{\Sigma_m,\ell,s}$ and $R_{\sigh_m,\ell,s}$ do not hold for $s$ taking values in this interesting range. Also note that $s=0$ exactly reduces to Theorem~\ref{thm:rff main thm metric 1}.\vspace{1mm}\\
(ii) As before, assuming $i^{-\alpha}\lesssim \lambda_i\lesssim i^{-\alpha}$ for $\alpha>1$ and $\ell=n^{\frac{\theta}{\alpha}}$, $0<\theta\le \alpha$, it follows that $$n^{-2\theta\left(1-\frac{1}{2\alpha}-\frac{s}{2}\right)}\lesssim R_{\Sigma,\ell,s}\lesssim n^{-2\theta\left(1-\frac{1}{2\alpha}-\frac{s}{2}\right)}$$ and $$n^{-2\theta\left(1-\frac{1}{2\alpha}-\frac{s}{2}\right)}\lesssim R_{\sigh,\ell,s}\lesssim_{\Pb^n}\begin{cases} 
      n^{-2\theta(1-\frac{1}{2\alpha}-\frac{s}{2})},\quad\,\,\,\theta\le\frac{\alpha}{2\alpha-1-\alpha s} \\
     \frac{1}{n},\qquad\,\,\,\qquad\qquad\theta\ge\frac{\alpha}{2\alpha-1-\alpha s}
   \end{cases},
$$
which matches with Corollary~\ref{rff poly decay corollary}(i,ii) for $s=0$. These show the convergence rate to be slower for $0<s\le 1$ compared to $s=0$, which is expected as the former corresponds to a stronger notion of reconstruction error. Also $0<\theta\le\frac{\alpha}{2\alpha-1-\alpha s}$ is the only useful range both computationally and statistically as $\theta>\frac{\alpha}{2\alpha-1-\alpha s}$ does not improve the statistical rates but increases the computational complexity. 

On the other hand for $-2\le s\le 0$, it follows that 
$$n^{-2\theta\left(1-\frac{1}{2\alpha}-\frac{s}{2}\right)}\lesssim_{\Lambda^m} R_{\sigh_m,\ell,s}\lesssim_{\Lambda^m\times \Pb^n}
       n^{-2\theta(1-\frac{1}{2\alpha}-\frac{s}{2})}$$
for $m\gtrsim n^{\frac{4\theta}{2-s}(1-\frac{1}{2\alpha}-\frac{s}{2})}$ and $\theta\le\frac{\alpha}{2\alpha-1-\alpha s}$, which implies that for sufficiently large $m$, EKPCA and RF-EKPCA have similar statistical behavior as long as $\theta\le \frac{\alpha}{2\alpha-1-\alpha s}$. However, RF-EKPCA is computationally better than EKPCA only when $0<\theta<\frac{(2-s)\alpha}{8\alpha-6+s-4s\alpha}$. Also note that for $\theta=\frac{\alpha}{2\alpha-1-\alpha s}$, which is where the best rate of $\frac{1}{n}$ is achieved for any $s\in[-2,0]$, we obtain $m\gtrsim n^{\frac{2}{2-s}}$, i.e., the requirement on the number of random features is monotonically increasing w.r.t.~$s\in[-2,0]$. Since $\theta$ is an increasing function of $s$, it implies fewer $\ell$ is sufficient for optimal rates for smaller $s$. To elaborate, at the chosen value of $\theta$, statistical optimality is conserved for RF-EKPCA at $s=0$ if $m\gtrsim n$ while only $m\gtrsim \sqrt{n}$ is required at $s=-2$. This is understandable as smaller values of $s$ result in weaker notions of reconstruction error as explained above.\QEDA 
\end{remark}
\end{subsection}
\section{Discussion} \label{discussion}

To summarize, we investigated the computational vs. statistical trade-off in the problem of approximating kernel PCA using random features. While it is obvious that approximate kernel PCA using $m$ random features has lower computational complexity than kernel PCA when $m<n$ with $n$ being the number of samples, it is not obvious that this computational gain is not achieved at the cost of statistical efficiency. Through inclusion and approximation operators, we explored various notions of reconstructing a kernel function using $\ell$ eigenfunctions, wherein we showed that approximate kernel PCA has computational advantage with no loss in statistical optimality as long as $m$ is large enough (but still $m<n$) and $\ell$ is small enough with $m$ depending on the number of eigenfunctions $\ell$ being considered. If $\ell$ is large, then more features are needed to maintain the statistical behavior, thereby resulting in the loss of computational advantage.

There are few open questions in this topic which may be of interest to address. (i) In contrast to the setting of this paper where $\ell$ grows with $n$, it may be of interest to consider asymptotics when $\ell$ is fixed but $n\rightarrow\infty$. In such a setting, one may investigate E-R, R-E and their variations/generalizations. For example, in R-E, we can compare EKPCA and RF-EKPCA by comparing $R_{\sigh,\ell}-R_{\Sigma,\ell}$ and $R_{\sigh_m,\ell}-R_{\Sigma,\ell}$. While Theorems~\ref{thm:rff main thm metric 1}, \ref{thm:rff main thm metric 2} and \ref{thm:schatten} do not directly specialize to the setting of fixed $\ell$, using ideas employed in their proofs, upper bounds can be derived on $R_{\sigh,\ell}-R_{\Sigma,\ell}$ and $R_{\sigh_m,\ell}-R_{\Sigma,\ell}$. However, lower bounds are needed to establish the sharpness of these upper bounds so that these excess errors can be matched for a certain choice of $m$. (ii) Apart from reconstruction error, one may compare $\ell$-eigenspaces (for fixed $\ell$) associated with EKPCA and RF-EKPCA by comparing the corresponding projection operators through their embeddings as bounded operators on $L^2(\bb{P})$. \citet{Ullah-18} investigated this direction by comparing certain inner product of the uncentered covariance operator with the difference between the projection operators associated with $\ell$-eigenspaces of KPCA and EKPCA (\emph{resp.} RF-EKPCA). Different meaningful notions of comparing the projection operators can be explored and upper convergence rates can be derived using the perturbation theory for self-adjoint operators (see \citealp{Sriperumbudur-18} for some preliminary results). However, as above, developing lower bounds will be critical to establish the sharpness of the upper bounds, thereby facilitating a meaningful comparison of the statistical performances of EKPCA and RF-EKPCA.



\section{Proofs}\label{Sec:proofs}
In this section we present the proofs of the results in Sections~\ref{Sec:kpca} and \ref{Sec:results}.
\subsection{Proof of Proposition~\ref{pro:eigsystem}}\label{subsec:pro-eigsystem}
Define the sampling operator 
$$S:\Cal{H}\rightarrow\bb{R}^n,\qquad f\mapsto \frac{1}{\sqrt{n}}(f(X_1),\ldots,f(X_n))^\top$$
whose adjoint, called the reconstruction operator can be shown (see Proposition~\ref{pro:sampling}\emph{(i)}) to be 
$$S^*:\bb{R}^n\rightarrow \Cal{H},\qquad \bm{\alpha}\mapsto \frac{1}{\sqrt{n}}\sum^n_{i=1}\alpha_i k(\cdot,X_i),$$
where $\bm{\alpha}:=(\alpha_1,\ldots,\alpha_n)^\top$. Define $\tilde{\bm{H}}_n=\frac{n}{n-1}\bm{H}_n$. It follows from Proposition~\ref{pro:sampling}\emph{(ii)} that $\sigh=S^*\tilde{\bm{H}}_nS$, which implies
$(\phih_i)_i$ satisfy \begin{equation}S^*\tilde{\bm{H}}_nS\phih_i=\lambdah_i\phih_i,\label{Eq:emp-eig-alternate}\end{equation}
where $\lambdah_i\ge 0$. Multiplying both sides of \eqref{Eq:emp-eig-alternate} on the left by $S$, we obtain that $(\widehat{\bm{\alpha}}_i)_i$, 
$\widehat{\bm{\alpha}}_i:=S\phih_i,\,i\in[n]$ are eigenvectors of $SS^*\tilde{\bm{H}}_n=\frac{1}{n}\bm{K}\tilde{\bm{H}}_n,$
i.e., they satisfy the finite dimensional linear system, \begin{equation}\bm{K}\tilde{\bm{H}}_n\widehat{\bm{\alpha}}_i=n\lambdah_i\widehat{\bm{\alpha}}_i,
\label{Eq:finite-eig}\end{equation}
where $\bm{K}$ is the Gram matrix, i.e., $(\bm{K})_{ij}=k(X_i,X_j),\,\,i,j\in[n]$ and the fact that $\bm{K}=nSS^*$ follows from Proposition~\ref{pro:sampling}\emph{(iii)}. It is important to note that $(\widehat{\bm{\alpha}}_i)_i$
do not form an orthogonal system in the usual Euclidean inner product but in the weighted inner product where the weighting matrix is $\tilde{\bm{H}}_n$. 
Indeed, it is easy to verify that $$\left\langle \widehat{\bm{\alpha}}_i,\tilde{\bm{H}}_n \widehat{\bm{\alpha}}_j\right\rangle_{2}=
\left\langle S\phih_i,\tilde{\bm{H}}_nS\phih_j\right\rangle_2=\langle\phih_i,\sigh \phih_j\rangle_\Cal{H}=\lambdah_j\langle \phih_i,\phih_j\rangle_\Cal{H}=\lambdah_j\delta_{ij},$$
where $\delta_{ij}$ is the Kronecker delta. 
Having obtained $(\widehat{\bm{\alpha}}_i)_i$ from \eqref{Eq:finite-eig}, the eigenfunctions of $\sigh$ are obtained from \eqref{Eq:emp-eig-alternate} as
\begin{equation*}\phih_i=\frac{1}{\lambdah_i}S^*\tilde{\bm{H}}_n\widehat{\bm{\alpha}}_i,\nonumber
\end{equation*}
and the result follows.
\subsection{Proof of Proposition~\ref{pro:solution}}\label{subsec:solution}
\emph{(i)} Define $P_\psi:=\sum^\ell_{i=1}\psi_i\oh\psi_i$. Therefore,
\begin{eqnarray}
\eu{R}((\psi_i)_{i\in [\ell]})=\bb{E}\norm{\id(I-P_\psi)\overline{k}(\cdot,X)}^2_{\lp}&{}\stackrel{(*)}{=}{}&\norm{\Sigma^{1/2}(I-P_\psi)\Sigma^{1/2}}^2_{\HSH}\nonumber\\
&{}\stackrel{(\dagger)}{=}{}&\Cal{R}^\Sigma_{0,1,1}(P_\psi),\nonumber
\end{eqnarray}
where we used Lemma~\ref{rf cov proj rewrite} in $(*)$ since $\Sigma=\id^*\id$ and $k$ is continuous and bounded (therefore, Bochner integrable), and the definition in \eqref{Eq:def} in $(\dagger)$. The result therefore follows from Lemma~\ref{lem:optima} that $(\phi_i)_{i\in[\ell]}$ is the unique minimizer of $\eu{R}$.\vspace{2mm}\\
\emph{(ii)} Define $P_\mu:=\sum^\ell_{i=1}\mu_i\ol\mu_i$. Therefore,
\begin{eqnarray}
\eu{S}((\mu_i)_{i\in [\ell]})&{}={}&\bb{E}\norm{(I-P_\mu)\id\overline{k}(\cdot,X)}^2_{\lp}=\inner{\id^*(I-P_\mu)^2\id}{\Sigma}_{\HSH}\nonumber\\
&{}={}&\text{Tr}\left[\id^*(I-P_\mu)^2\id\id^*\id\right]=\text{Tr}[T(I-P_\mu)^2T]\nonumber\\
&{}={}&\norm{(I-P_\mu)T}^2_{\HSL},\nonumber
\end{eqnarray}
where $T:=\id\id^*$. 
It follows from Lemma~\ref{lem:optima}\emph{(ii)} that 
$(\chi_i)_{i\in[\ell]}$ is the minimizer of $\eu{S}$ where $(\chi_i)_{i\in[\ell]}$ are the eigenfunctions of $T$ that correspond to the eigenvalues $(\lambda_i)_{i\in[\ell]}$. 
We now show that $\chi_i=\frac{\id\phi_i}{\sqrt{\lambda_i}}$. Since $\Sigma\phi_i=\lambda_i\phi_i$, we have $\id^*\id\phi_i=\lambda_i\phi_i$, which implies $\id\id^*\id\phi_i=\lambda_i\id\phi_i$, i.e., $T\left(\frac{\id\phi_i}{\sqrt{\lambda_i}}\right)=\lambda_i\left(\frac{\id\phi_i}{\sqrt{\lambda_i}}\right)$. Therefore $\left(\frac{\id\phi_i}{\sqrt{\lambda_i}}\right)_i$ are eigenfunctions of $T$ and so $\left(\frac{\id\phi_i}{\sqrt{\lambda_i}}\right)_{i\in[\ell]}$ is the minimizer of $\eu{S}$.\vspace{2mm}\\
\emph{(iii)} It follows from \emph{(i)}, \emph{(ii)} and Lemma~\ref{lem:optima} that $R_{\Sigma,\ell}=\Cal{R}^\Sigma_{0,1,1}(P_\ell(\Sigma))=\sum_{i>\ell}\lambda^2_i$ and $S_{\Sigma,\ell}=\Cal{R}^T_{0,0,2}(\sum^\ell_{i=1}\chi_i\ol\chi_i)=\sum_{i>\ell}\lambda^2_i$.\vspace{2mm}\\
\emph{(iv)} Note that for any $(\psi_i)_{i\in[\ell]}$,
\begin{eqnarray}
\eu{R}((\psi_i)_{i\in [\ell]})&{}={}&\bb{E}\norm{\id(I-P_\psi)\overline{k}(\cdot,X)}^2_{\lp}=\bb{E}\norm{\Sigma^{1/2}(I-P_\psi)\overline{k}(\cdot,X)}^2_{\Cal{H}}\nonumber\\
&{}\le{}& \Vert\Sigma^{1/2}\Vert^2_{\OPH}\bb{E}\norm{(I-P_\psi)\overline{k}(\cdot,X)}^2_{\Cal{H}}.\nonumber
\end{eqnarray}
The result therefore follows by applying the above inequality for $(\phi_i)_{i\in[\ell]}$.
\subsection{Proof of Proposition~\ref{pro:interpret}}\label{subsec:interpret}
\emph{(i), (ii)} Refer to the proof of Proposition~\ref{pro:solution}\emph{(i)}.\vspace{1.5mm}\\
\emph{(iii), (iv), (v)} Note that $\eu{T}_{\sigh,\Cal{H}}(P_\psi)=\Cal{R}^{\sigh}_{0,1,1}(P_\psi)$, $\eu{T}_{\Sigma_m,\Cal{H}_m}(P_\tau)=\Cal{R}^{\Sigma_m}_{0,1,1}(P_\tau)$ and $\eu{T}_{\sigh_m,\Cal{H}_m}(P_\tau)=\Cal{R}^{\sigh_m}_{0,1,1}(P_\tau)$, following the notation defined in \eqref{Eq:def}. The result therefore follows from Lemma~\ref{lem:optima}.
\subsection{Proof of Proposition~\ref{pro:interpret2}}\label{subsec:interpret2}
\emph{(i), (ii)} Refer to the proof of Proposition~\ref{pro:solution}\emph{(ii)}.\vspace{1.5mm}\\
\emph{(iii)} It follows from Lemma~\ref{lem:optima} that the eigenfunctions corresponding to the top $\ell$ eigenvalues of $\widehat{\id}\widehat{\id}^*$ is the minimizer of $\eu{V}_{\widehat{\id}\widehat{\id}^*,L^2(\bb{P}_n)}$. In the following, we will show that $\widehat{\id}^*\widehat{\id}=\sigh$ which will then imply $\widehat{\id}\widehat{\id}^*\widehat{\id}\phih_i=\widehat{\id}\sigh\phih_i=\lambdah_i\widehat{\id}\phih_i$. The result is completed by noting that 
$$\left\langle\frac{\widehat{\id}\phih_i}{\sqrt{\lambdah_i}},\frac{\widehat{\id}\phih_j}{\sqrt{\lambdah_j}} \right\rangle_{L^2(\bb{P}_n)}=\frac{1}{\sqrt{\lambdah_i\lambdah_j}}\left\langle\sigh\phih_i,\phih_j\right\rangle_{\Cal{H}}=\delta_{ij}$$
and $(\lambdah_i,\widehat{\id}\phih_i/\sqrt{\lambdah_i})_i$ is the eigensystem of $\widehat{\id}\widehat{\id}^*$. 

To show $\widehat{\id}^*\widehat{\id}=\sigh$, for any $f\in\Cal{H}$, consider
\begin{eqnarray}
\inner{f}{\widehat{\id}^*\widehat{\id}f}_{\Cal{H}}&{}={}&\norm{\widehat{\id}f}^2_{L^2(\bb{P}_n)}=\frac{1}{n}\sum^n_{i=1}(\widehat{\id}f)^2(X_i)
=\frac{1}{n-1}\sum^n_{i=1}\left(f(X_i)-\frac{1}{n}\sum^n_{j=1}f(X_j)\right)^2\nonumber\\
&{}={}&\frac{1}{n-1}\sum^n_{i=1}f^2(X_i)-\frac{1}{n(n-1)}\left(\sum^n_{j=1}f(X_j)\right)^2
\stackrel{\eqref{Eq:temmm}}{=}\inner{f}{\sigh f}_{\Cal{H}}\nonumber
\end{eqnarray}
and the claim follows.\vspace{1.5mm}\\
\emph{(iv), (v)} The proof is exactly same as that of \emph{(iii)} by replacing $\widehat{\id}$, $\sigh$, $\phih_i$ and $\lambdah_i$ by $\A$, $\Sigma_m$, $\phi_{m,i}$ and $\lambda_{m,i}$ (\emph{resp.} $\widehat{\A}$, $\sigh_m$, $\phih_{m,i}$ and $\lambdah_{m,i}$), respectively.
\subsection{Proof of Theorem~\ref{pro:relation}}\label{subsec:relation} 
Define $$\Mp:=\intx k(\cdot,x)\,d\bb{P}(x),\,\,\mph:=\frac{1}{n}\sum^n_{i=1}k(\cdot,X_i),\,\,\mpm:=\intx k_m(\cdot,x)\,d\bb{P}(x)$$ and $$\mpmh:=\frac{1}{n}\sum^n_{i=1}k_m(\cdot,X_i).$$
\emph{(i)} From Proposition~\ref{pro:interpret}\emph{(i)}, we have
\begin{eqnarray}
\eu{T}_{\Sigma,\Cal{H}}(P_\ell(\sigh))&{}={}&\bb{E}\norm{\id \overline{k}(\cdot,X)-\id P_\ell(\sigh))\overline{k}(\cdot,X)}^2_{\lp}\nonumber\\
&{}\stackrel{(\dagger)}{=}{}&\bb{E}\norm{\id \overline{k}(\cdot,X)-\id P_\ell(\sigh)\widetilde{k}(\cdot,X)+\id P_\ell(\sigh)(\Mp-\mph)}^2_{\lp}\nonumber\\
&{}\stackrel{(*)}{=}{}&R_{\sigh,\ell}+\norm{\id P_\ell(\sigh)(\Mp-\mph)}^2_{\lp}
\stackrel{(\ddagger)}{\lesssim_{\bb{P}^n}} R_{\sigh,\ell}+\frac{1}{n},\nonumber
\end{eqnarray}
where expanding the squares in $(\dagger)$ and noting that the expectation of the inner product is zero, yields $(*)$. $(\ddagger)$ follows from \eqref{eq:rff ekpca B bnd} and Lemma~\ref{lem:kernel mean bnd}. The lower bound is obtained by noting that $\norm{\id P_\ell(\sigh)(\Mp-\mph)}^2_{\lp}\ge 0$.\vspace{1.5mm}\\
\emph{(ii)} From Proposition~\ref{pro:interpret}\emph{(iv)}, we have
\begin{eqnarray}
\eu{T}_{\Sigma_m,\Cal{H}_m}(P_\ell(\Sigma_m))&{}={}&\bb{E}\norm{\A \overline{k}_m(\cdot,X)-\A P_\ell(\Sigma_m))\overline{k}_m(\cdot,X)}^2_{\lp}\nonumber\\
&{}={}&\bb{E}\norm{\A \overline{k}_m(\cdot,X)-\id \overline{k}(\cdot,X)+\id \overline{k}(\cdot,X)-\A P_\ell(\Sigma_m)\overline{k}_m(\cdot,X)}^2_{\lp},\nonumber
\end{eqnarray}
and the result follows from Lemmas~\ref{lem:rf kernel mean bnd} and \ref{lem:supp}. \vspace{1.5mm}\\
\emph{(ii)} From Proposition~\ref{pro:interpret}\emph{(v)}, we have
\begin{eqnarray}
\eu{T}_{\Sigma_m,\Cal{H}_m}(P_\ell(\sigh_m))&{}={}&\bb{E}\norm{\A \overline{k}_m(\cdot,X)-\A P_\ell(\sigh_m))\overline{k}_m(\cdot,X)}^2_{\lp}\nonumber\\
&{}={}&\bb{E}\norm{\A \overline{k}_m(\cdot,X)-\A P_\ell(\sigh_m)\widetilde{k}_m(\cdot,X)+\A P_\ell(\sigh_m)(\mpm-\mpmh)}^2_{\lp}\nonumber\\
&{}={}&\bb{E}\norm{\A \overline{k}_m(\cdot,X)-\A P_\ell(\sigh_m)\widetilde{k}_m(\cdot,X)}^2_{\lp}+\bb{E}\norm{\A P_\ell(\sigh_m)(\mpm-\mpmh)}^2_{\lp}\nonumber\\
&{}={}&\bb{E}\norm{\A \overline{k}_m(\cdot,X)-\id \overline{k}(\cdot,X)+\id \overline{k}(\cdot,X)-\A P_\ell(\sigh_m)\widetilde{k}_m(\cdot,X)}^2_{\lp}\nonumber\\
&&\qquad\qquad+\bb{E}\norm{\A P_\ell(\sigh_m)(\mpm-\mpmh)}^2_{\lp}\nonumber
\end{eqnarray}
and the result follows from Lemmas~\ref{lem:kernel mean bnd}, \ref{lem:rf kernel mean bnd} and \ref{lem:supp}.
\subsection{Proof of Theorem~\ref{thm:rff main thm metric 1}}\label{subsec:thm-bound}
\emph{(i)} The result follows from the proof of Proposition~\ref{pro:solution}\emph{(iii)}.
\vspace{1mm}
\\
\emph{(ii)} \underline{\emph{Upper bound:}} Note that $$R_{\sigh,\ell}=\E\norm{\id\overline{k}(\cdot,X)-\id P_\ell(\sigh)(k(\cdot,X)-\mph)}_{\lp}^2.$$ Therefore, adding and subtracting $\id P_\ell(\sigh)m_\bb{P}$ and expanding squares, we obtain
\begin{eqnarray}
R_{\sigh,\ell}
&{}={}&\circled{\small{1}}+\circled{\small{2}}\label{eq:rff ekpca decomp}
\end{eqnarray}
with the inner product being zero since $\bb{E}[\overline{k}(\cdot,X)]=0$. Here $$\circled{\small{1}}:=
\E\norm{\id(I-P_\ell(\sigh))\overline{k}(\cdot,X)}_{\lp}^2,\,\,\text{and}\quad\circled{\small{2}}:= \norm{\id P_\ell(\sigh)(m_\Pb-\mph)}_{\lp}^2.$$ It follows from Lemma \ref{rf cov proj rewrite} that 
\begin{equation}
    \circled{\small{1}}=\E\norm{\id(I-P_\ell(\sigh))\overline{k}(\cdot,X)}_{\lp}^2=\norm{\Sigma^{1/2}(I-P_\ell(\sigh))\Sigma^{1/2}}_{\HSH}^2.\nonumber
\end{equation}
For any $t>0$, we have
\begin{eqnarray}
&{}{}&\norm{\Sigma^{1/2}(I-P_\ell(\sigh))\Sigma^{1/2}}_{\HSH}^2\nonumber
\\
&{}={}&\left\Vert\Sigma^{1/2}(\Sigma+tI)^{-1/2}(\Sigma+tI)^{1/2}(I-P_\ell(\sigh))(\Sigma+tI)^{1/2}(\Sigma+tI)^{-1/2}\Sigma^{1/2}\right\Vert_{\HSH}^2\nonumber
\\
&{}\le{}&\norm{\Sigma^{1/2}(\Sigma+tI)^{-1/2}}_{\HSH}^2\norm{\Sigma^{1/2}(\Sigma+tI)^{-1/2}}_{\OPH}^2\norm{(\Sigma+tI)^{1/2}(I-P_\ell(\sigh))(\Sigma+tI)^{1/2}}_{\OPH}^2\nonumber
\\
&{}\stackrel{(*)}{\le}{}&\Cal{N}_\Sigma(t)\norm{(\Sigma+tI)^{1/2}(I-P_\ell(\sigh))(\Sigma+tI)^{1/2}}_{\OPH}^2\nonumber
\\
&{}\le{}&\Cal{N}_\Sigma(t)\norm{(\Sigma+tI)^{1/2}(\sigh+tI)^{-1/2}}_{\OPH}^4\norm{(\sigh+tI)^{1/2}(I-P_\ell(\sigh))(\sigh+tI)^{1/2}}_{\OPH}^2,\nonumber\\
&{}\le{}&\Cal{N}_\Sigma(t)(\lambdah_{\ell+1}+t)^2\norm{(\Sigma+tI)^{1/2}(\sigh+tI)^{-1/2}}_{\OPH}^4,\label{Eq:tmep}
\end{eqnarray}
where in $(*)$, we have used $\norm{\Sigma^{1/2}(\Sigma+tI)^{-1/2}}_{\HSH}^2=\text{tr}((\Sigma+tI)^{-1/2}\Sigma(\Sigma+tI)^{-1/2})=\text{tr}(\Sigma(\Sigma+tI)^{-1})=\Cal{N}_\Sigma(t)$ and $\norm{\Sigma^{1/2}(\Sigma+tI)^{-1/2}}_{\OPH}^2\le 1$. Applying Lemma~\ref{lem:cent 1}\emph{(ii, iv)} to \eqref{Eq:tmep}, we obtain that for any $\delta>0$ such that $\frac{140\kappa}{n}\log\frac{16\kappa n}{\delta}\le t\le\norm{\Sigma}_{\OPH}$, with probability at least $1-2\delta$ over the choice of $(X_i)^n_{i=1}$,
\begin{eqnarray}
\circled{\small{1}}=\norm{\Sigma^{1/2}(I-P_\ell(\sigh))\Sigma^{1/2}}_{\HSH}^2\le 4\Cal{N}_\Sigma(t)(\lambdah_{\ell+1}+t)^2\le 9\Cal{N}_\Sigma(t)(\lambda_{\ell+1}+t)^2.\label{Eq:pe-2}
\end{eqnarray}
%
%
We now bound $\circled{\small{2}}$ as follows.
\begin{eqnarray}
\circled{\small{2}}&{}={}&\inner{\id P_\ell(\sigh)(m_\Pb-\mph)}{\id P_\ell(\sigh)(m_\Pb-\mph)}_{\lp}\nonumber
\\
&{}={}&\inner{\Sigma^{1/2}P_\ell(\sigh)(m_\Pb-\mph)}{\Sigma^{1/2}P_\ell(\sigh)(m_\Pb-\mph)}_\Hk\nonumber\\
&{}={}&\norm{\Sigma^{1/2}P_\ell(\sigh)(m_\Pb-\mph)}_\Hk^2
\le\norm{\Sigma^{1/2}}_\OPH^2\norm{P_\ell(\sigh)}_\OPH^2\norm{m_\Pb-\mph}_\Hk^2\nonumber\\
&{}={}&\Vert\Sigma\Vert_\OPH\norm{m_\Pb-\mph}_\Hk^2,\label{eq:rff ekpca B bnd}
\end{eqnarray}
where the last equality uses $\norm{P_\ell(\sigh)}_\OPH=1$.  The result follows by applying Lemma \ref{lem:kernel mean bnd}\emph{(i)} to (\ref{eq:rff ekpca B bnd}), combining it with \eqref{Eq:pe-2} in (\ref{eq:rff ekpca decomp}), and using $\norm{\Sigma}_\OPH\le 2\kappa$.\vspace{1mm}\\
\underline{\emph{Lower bound:}} It follows from \eqref{eq:rff ekpca decomp} that $R_{\sigh,\ell}\ge \circled{\small{1}}$. Lemma~\ref{lem:optima} implies that $\circled{\small{1}} =\Cal{R}^{\Sigma}_{0,1,1}(P_\ell(\sigh))\ge \sum_{i>\ell}\lambda^2_i=R_{\Sigma,\ell}$ and the result follows by combining these bounds.
\vspace{1.5mm}
\\
\emph{(iii)} \underline{\emph{Upper bound:}} Since $R_{\Sigma_m,\ell}=\E\norm{\id\overline{k}(\cdot,X)-\tS P_\ell(\Sigma_m)\overline{k}_m(\cdot,X)}_{\lp}^2$, by \eqref{Eq:supp}, we have
\begin{eqnarray*}
R_{\Sigma_m,\ell}&{}\le{}& 2\left(\circled{\small{3}}+\circled{\small{4}}\right),\nonumber
\end{eqnarray*}
where $\circled{\small{3}}:=\E\norm{\id\overline{k}(\cdot,X)-\tS\overline{k}_m(\cdot,X)}_{\lp}^2$ and $\circled{\small{4}}:=\E\norm{\tS(I-P_\ell(\Sigma_m))\overline{k}_m(\cdot,X)}_{\lp}^2.$
It follows from Lemma~\ref{rf cov proj rewrite} that 
\begin{eqnarray}
\circled{\small{4}}&{}={}&\norm{\Sigma_m^{1/2}(I-P_\ell(\Sigma_m))\Sigma_m^{1/2}}_\HSM^2=
\sum_{i=\ell+1}^m\lambda_{m,i}^2=\sum_{i=\ell+1}^m(\lambda_{m,i}-\lambda_i+\lambda_i)^2\nonumber\\
&{}\le{}& 2\sum_{i=\ell+1}^m(\lambda_{m,i}-\lambda_i)^2+2\sum_{i>\ell}\lambda_i^2\stackrel{(**)}{\le} 2\norm{\id\id^*-\tS\tS^*}_\HSL^2+2\sum_{i>\ell}\lambda_i^2,\label{eq:rf-kpca F1.1}
\end{eqnarray}
%
where in $(**)$, we used Hoffman-Wielandt inequality \citep{Bhatia-97} along with the fact that 
$\Sigma=\id^*\id$ and $\id\id^*$ have same eigenvalues, and similarly $\Sigma_m=\tS^*\tS$ and $\tS\tS^*$.
The result follows by applying Lemma~\ref{lem:perturb} to (\ref{eq:rf-kpca F1.1}) and Lemma~\ref{lem:rf kernel mean bnd} to \circled{\small{3}}.\vspace{1mm}
\\
\underline{\emph{Lower bound:}} Note that 
$R_{\Sigma_m,\ell}
\stackrel{\eqref{Eq:supp}}{\ge}\left(\sqrt{\circled{\small{4}}}-\sqrt{\circled{\small{3}}}\right)^2,
$ where \begin{eqnarray}\sqrt{\circled{\small{4}}}-\sqrt{\circled{\small{3}}}&{}={}&\sqrt{\sum_{i>\ell}\lambda^2_{m,i}}-\sqrt{\circled{\small{3}}}
\ge \left|\sqrt{\sum_{i>\ell}\lambda^2_{i}}-\sqrt{\sum_{i>\ell}\left(\lambda_{ i}-\lambda_{m,i}\right)^2}\right|-\sqrt{\circled{\small{3}}}\nonumber\\
&{}\ge{}& \sqrt{\sum_{i>\ell}\lambda^2_{i}}-\norm{\id\id^*-\tS\tS^*}_{\HSL}-\sqrt{\circled{\small{3}}}\ge \sqrt{\sum_{i>\ell}\lambda^2_{i}}-4\kappa\sqrt{\frac{2\log\frac{2}{\delta}}{m}}-8\kappa\sqrt{\frac{\log\frac{2}{\delta}}{m}}\nonumber
\end{eqnarray}
\begin{eqnarray}
&{}\ge{}& \sqrt{\sum_{i>\ell}\lambda^2_{i}}-16\kappa\sqrt{\frac{\log\frac{2}{\delta}}{m}}\ge\frac{1}{2}\sqrt{\sum_{i>\ell}\lambda^2_{i}}
\label{Eq:qq}
\end{eqnarray}
as $\frac{1}{2}\sqrt{\sum_{i>\ell}\lambda^2_{i}}\ge 16\kappa\sqrt{\frac{\log\frac{2}{\delta}}{m}}$, with \eqref{Eq:qq} holding with probability at least $1-3\delta$ over the choice of $(\theta_i)^m_{i=1}$.
\vspace{1mm}\\
\emph{(iv)} \underline{\emph{Upper bound:}} $R_{\sigh_m,\ell}$ can be alternately written as $$R_{\sigh_m,\ell}=\E\norm{\id\overline{k}(\cdot,X)-\tS P_\ell(\sigh_m)(k_m(\cdot,X)-\mpmh)}_{L^2(\Pb)}^2,$$ which can be bounded as 
\begin{eqnarray}
R_{\sigh_m,\ell}
    &{}\le{}& 2\E\norm{\id\overline{k}(\cdot,X)-\tS P_\ell(\sigh_m)\overline{k}_m(\cdot,X)}_{\lp}^2+2 \E\norm{\tS P_\ell(\sigh_m)(m_{\Pb,m}-\mpmh)}_{\lp}^2\nonumber\\
    &{}\le{}& 4\left(\circled{\small{3}}+\circled{\small{5}}\right)+2\left(\circled{\small{6}}\right),\label{Eq:pe-emp}
\end{eqnarray}
where $\circled{\small{5}}:=
\E\norm{\tS(I-P_\ell(\sigh_m)\overline{k}_m(\cdot,X)}_{\lp}^2$ and $\circled{\small{6}}:=
\norm{\tS P_\ell(\sigh_m)(m_{\Pb,m}-\mpmh)}_{\lp}^2.$ Using Lemma~\ref{rf cov proj rewrite}, for $t>0$, we bound \circled{\small{5}} as
\begin{eqnarray}
\circled{\small{5}}&{}={}&\norm{\Sigma_m^{1/2}(I-P_\ell(\sigh_m))\Sigma_m^{1/2}}_{\HSM}^2\nonumber
\\
&{}={}&\left\Vert\Sigma_m^{1/2}(\Sigma_m+tI)^{-1/2}(\Sigma_m+tI)^{1/2}(I-P_\ell(\sigh_m))(\Sigma_m+tI)^{1/2}(\Sigma_m+tI)^{-1/2}\Sigma_m^{1/2}\right\Vert_{\HSM}^2\nonumber
\\
&{}\le{}&\norm{\Sigma_m^{1/2}(\Sigma_m+tI)^{-1/2}}_{\HSM}^2\norm{(\Sigma_m+tI)^{-1/2}\Sigma_m^{1/2}}_{\OPHm}^2\nonumber\\
&{}{}&\qquad\qquad\times\norm{(\Sigma_m+tI)^{1/2}(I-P_\ell(\sigh_m))(\Sigma_m+tI)^{1/2}}_{\OPHm}^2\nonumber\\
&{}\stackrel{(*)}{\le}&{}\Cal{N}_{\Sigma_m}(t)\norm{(\Sigma_m+tI)^{1/2}(I-P_\ell(\sigh_m))(\Sigma_m+tI)^{1/2}}_{\OPHm}^2\nonumber\\
&{}\le{}&\Cal{N}_{\Sigma_m}(t)\norm{(\Sigma_m+tI)^{1/2}(\sigh_m+tI)^{-1/2}}_\OPHm^4
\nonumber\\
&{}{}&\qquad\qquad\times\norm{(\sigh_m+tI)^{1/2}(I-P_\ell(\sigh_m))(\sigh_m+tI)^{1/2}}_\OPHm^2\nonumber
\\
&{}\le{}&\Cal{N}_{\Sigma_m}(t)(\lambdah_{m,\ell+1}+t)^2\norm{(\Sigma_m+tI)^{1/2}(\sigh_m+tI)^{-1/2}}_\OPHm^4,\label{eq:rf-ekpca G2.1}
\end{eqnarray}
where we used $\norm{\Sigma_m^{1/2}(\Sigma_m+tI)^{-1/2}}_\HSM^2=\text{tr}\left(\Sigma_m^{1/2}(\Sigma_m+tI)^{-1}\Sigma_m^{1/2}\right)=\text{tr}\left(\Sigma_m(\Sigma_m+tI)^{-1}\right)=:\Cal{N}_{\Sigma_m}(t),$ and $\norm{\Sigma_m^{1/2}(\Sigma_m+tI)^{-1/2}}_\OPHm^2\le 1$ in $(*)$. Conditioning on $(\theta_i)^m_{i=1}$ and applying Lemma~\ref{lem:cent 1}\emph{(ii, iv)} to \eqref{eq:rf-ekpca G2.1}, we obtain that for any $\delta>0$ and 
\begin{equation}
\frac{140\kappa}{n}\log\frac{16\kappa n}{\delta}\le t\le\norm{\Sigma_m}_{\OPH},\label{eq:verify}
\end{equation}
\begin{equation}\bb{P}^n_{|(\theta_i)^m_{i=1}}\left\{(X_i)^n_{i=1}:\circled{\small{5}}\le 9\Cal{N}_{\Sigma_m}(t)(\lambda_{m,\ell+1}+t)^2\right\}\ge 1-2\delta.\label{Eq:temp5}\end{equation}
Now, unconditioning w.r.t.~$(\theta_i)^m_{i=1}$ and applying Lemma~\ref{lem:1 rff}\emph{(ii, iv)} in \eqref{Eq:temp5}, we obtain that for any $\delta>0$ and $\frac{86\kappa}{m}\log\frac{16\kappa m}{\delta}\le t\le\norm{\Sigma}_{\OPH}$, 
\begin{eqnarray}&{}{}&\Lambda^m\times\bb{P}^n\left\{\left((\theta_i)^m_{i=1},(X_i)^n_{i=1}\right):\circled{\small{5}}\le \frac{81}{4}\left[\frac{32\kappa\log\frac{2}{\delta}}{tm}+\sqrt{\frac{32\kappa\Cal{N}_\Sigma(t)\log\frac{2}{\delta}}{tm}}+2\Cal{N}_\Sigma(t)\right](\lambda_{\ell+1}+t)^2\right\}\nonumber\\
&{}{}&\qquad\qquad\qquad\qquad\qquad\ge 1-5\delta.\label{Eq:final5}
\end{eqnarray}
Note that the upper bound in \eqref{eq:verify} holds because we assumed that $t\le \frac{1}{3}\Vert\Sigma\Vert_{\OPH}$ which is equivalent to $t\le \frac{1}{2}(\norm{\Sigma}_{\OPH}-t)\le \norm{\Sigma_m}_{\OPH}$ where we used $\frac{1}{2}(\lambda_{1}+t)\le \lambda_{m,1}+t$ from Lemma~\ref{lem:1 rff}\emph{(iii)}.

%
We now bound \circled{\small{6}} as 
\begin{eqnarray}
\qquad\quad\circled{\small{6}}&{}={}&\inner{\tS P_\ell(\sigh_m)(m_{\Pb,m}-\mpmh)}{\tS P_\ell(\sigh_m)(m_{\Pb,m}-\mpmh)}_{\lp}\nonumber
\\
&{}={}&\inner{\Sigma_m^{1/2}P_\ell(\sigh_m)(m_{\Pb,m}-\mpmh)}{\Sigma_m^{1/2}P_\ell(\sigh_m)(m_{\Pb,m}-\mpmh)}_{\Hk_m}\nonumber
\\
&{}={}&\norm{\Sigma_m^{1/2}P_\ell(\sigh_m)(m_{\Pb,m}-\mpmh)}_{\Hk_m}^2\nonumber\\
&{}\le{}&\norm{\Sigma_m^{1/2}}_\OPHm^2\norm{P_\ell(\sigh_m)}_\OPHm^2\norm{m_{\Pb,m}-\mpmh}_{\Hk_m}^2\nonumber\\
&{}\le{}&\lambda_{m,1}\norm{m_{\Pb,m}-\mpmh}_{\Hk_m}^2.
\label{eq:rf-ekpca H1}
\end{eqnarray}
Applying Lemma~\ref{lem:kernel mean bnd}\emph{(ii)} (by conditioning w.r.t.~$(\theta_i)^m_{i=1}$) and Lemma~\ref{lem:1 rff}\emph{(ii)} (to uncondition $(\theta_i)^m_{i=1}$) to (\ref{eq:rf-ekpca H1}), for any $\delta>0$ and $n\ge 2\log\frac{2}{\delta}$, we obtain
\begin{equation}\Lambda^m\times\bb{P}^n\left\{\left((\theta_i)^m_{i=1},(X_i)^n_{i=1}\right):\circled{\small{6}}\le \frac{320\kappa^2\log\frac{2}{\delta}}{3n}\right\}\ge 1-2\delta,\label{Eq:final6}
\end{equation}
where we used $\lambda_{m,1}\le \frac{3\lambda_1+t}{2}$ from Lemma~\ref{lem:1 rff}\emph{(ii)} and $t\le \frac{\lambda_1}{3}$ (as per our assumption), resulting in 
$\lambda_{m,1}\le \frac{5}{3}\norm{\Sigma}_{\OPH}\le \frac{10\kappa}{3}$. The result therefore follows by applying Lemma~\ref{lem:rf kernel mean bnd} to \circled{\small{3}} and combining it with \eqref{Eq:final5} and \eqref{Eq:final6} in \eqref{Eq:pe-emp}.\vspace{1mm}\\
\underline{\emph{Lower bound:}} As carried out in the proof of the lower bound of \emph{(iii)}, it can be shown that
\begin{eqnarray*}
R_{\sigh_m,\ell}&{}={}&\E\norm{\id\overline{k}(\cdot,X)-\tS P_\ell(\sigh_m)\widetilde{k}_m(\cdot,X)}_{L^2(\Pb)}^2
\ge \left(\sqrt{\circled{\small{7}}}-\sqrt{\circled{\small{3}}}\right)^2,\nonumber
\end{eqnarray*}
where
\begin{eqnarray*}
\circled{\small{7}}&{}:={}&\E\norm{\tS \overline{k}_m(\cdot,X)-\tS P_\ell(\sigh_m)\widetilde{k}_m(\cdot,X)}_{L^2(\Pb)}^2\nonumber\\
&{}={}&\E\norm{\tS(I-P_\ell(\sigh_m))\overline{k}_m(\cdot,X)}^2_{L^2(\Pb)}+\norm{\tS P_\ell(\sigh_m)(\mpm-\mpmh)}^2_{L^2(\Pb)}\nonumber\\
&{}{}&\qquad-2\E\left\langle \tS(I-P_\ell(\sigh_m))\overline{k}_m(\cdot,X),\tS P_\ell(\sigh_m)(\mpm-\mpmh)\right\rangle_{L^2(\Pb)}\nonumber\\
&{}\stackrel{(\dagger)}{=}{}&\circled{\small{5}}+\circled{\small{6}}\stackrel{(\ddagger)}{\ge} \sum_{i>\ell}\lambda^2_{m,i}+\circled{\small{6}}\ge \sum_{i>\ell}\lambda^2_{m,i},\nonumber
\end{eqnarray*}
where $(\dagger)$ follows by noting that the term involving the inner product is zero and $(\ddagger)$ follows by applying Lemma~\ref{lem:optima} to $\circled{\small{5}}$.
%
The result, therefore, follows from \eqref{Eq:qq}.
%
\subsection{Proof of Corollary \ref{rff poly decay corollary}}\label{sec:rff poly decay proof}
\emph{(i)}  From Theorem \ref{thm:rff main thm metric 1}\emph{(i)} we have 
$$R_{\Sigma,\ell}=
\sum_{i>\ell}\lambda_i^2\lesssim\sum_{i>\ell} i^{-2\alpha}\lesssim\int_\ell^\infty x^{-2\alpha}dx\lesssim\ell^{1-2\alpha}=n^{-2\theta(1-\frac{1}{2\alpha})}.$$
Similarly,
$$
R_{\Sigma,\ell}=\sum_{i>\ell}\lambda_i^2\gtrsim\sum_{i>\ell} i^{-2\alpha}\gtrsim\int_{\ell+1}^\infty x^{-2\alpha}dx\gtrsim(\ell+1)^{1-2\alpha}\gtrsim n^{-2\theta(1-\frac{1}{2\alpha})}.$$ 
\\
\emph{(ii)} From Theorem \ref{thm:rff main thm metric 1}\emph{(ii)} we have
$$R_{\sigh,\ell}\lesssim_{\Pb^n}\frac{1}{n}+\Cal{N}_\Sigma(t)(\lambda_\ell+t)^2,$$
with $\frac{\log n}{n}\lesssim t\le\frac{\lambda_1}{3}$. Using $\Cal{N}_\Sigma(t)\lesssim t^{-1/\alpha}$ from Lemma~\ref{N(t)}\emph{(i)}, it follows that
\begin{eqnarray*}
R_{\sigh,\ell}&{}\lesssim_{\Pb^n}{}&\inf\left\{t^{-1/\alpha}(n^{-\theta}+t)^2+n^{-1}:\frac{\log n}{n}\lesssim t\le\frac{\lambda_1}{3}\right\}\nonumber\\
&{}\lesssim{}&
\begin{cases}n^{-2\theta(1-\frac{1}{2\alpha})}+\frac{1}{n},\qquad \theta<1\\
\left(\frac{\log n}{n}\right)^{2-\frac{1}{\alpha}}+\frac{1}{n},\qquad \theta\ge1
\end{cases}.
\end{eqnarray*}
Of course, $\left(\frac{\log n}{n}\right)^{2-\frac{1}{\alpha}}\le\frac{1}{n}$ always holds, and $n^{-2\theta(1-\frac{1}{2\alpha})}\le\frac{1}{n}$ for $\theta\ge\frac{\alpha}{2\alpha-1}$, yielding the result.
\\
\emph{(iii)}  From Theorem \ref{thm:rff main thm metric 1}\emph{(iii)} we have
$$R_{\Sigma_m,\ell}\lesssim_{\Lambda_m}\frac{1}{m}+\sum_{i>\ell}\lambda_i^2.$$
From \emph{(i)} we have $\sum_{i>\ell}\lambda_i^2\lesssim n^{-2\theta\left(1-\frac{1}{2\alpha}\right)}$, and the result follows.
\\
\emph{(iv)} From Theorem \ref{thm:rff main thm metric 1}\emph{(iv)} we have
$$R_{\sigh_m,\ell}\lesssim_{\Pb^n\times\Lambda^m}\frac{1}{n}+\frac{1}{m}+\left(\Cal{N}_\Sigma(t)+\sqrt{\frac{\Cal{N}_\Sigma(t)}{tm}}+\frac{1}{tm}\right)(\lambda_\ell+t)^2,$$
for $\frac{\log n}{n}\vee\frac{\log m}{m}\lesssim t\lesssim\lambda_1$.  Note that $\Cal{N}_\Sigma(t)+\frac{1}{tm}+\sqrt{\frac{\Cal{N}_\Sigma(t)}{tm}}\lesssim t^{-1/\alpha}$, which follows from $\Cal{N}_\Sigma(t) \lesssim t^{-1/\alpha}$ (Lemma~\ref{N(t)}\emph{(i)}), $\frac{1}{tm}<t^{-1/\alpha}$ and $\sqrt{\frac{t^{-(1+1/\alpha)}}{m}}\lesssim t^{-1/\alpha}$ since $\frac{1}{m}<\left(\frac{\log n}{n}\vee\frac{\log m}{m}\right)^{1-\frac{1}{\alpha}}\lesssim t^{1-\frac{1}{\alpha}}$. Therefore, we have
$$R_{\sigh_m,\ell}\lesssim_{\Pb^n\times\Lambda^m}n^{-\gamma}+t^{-1/\alpha}(n^{-\theta}+t)^2,$$
for $\frac{\log n}{n^\gamma}\lesssim t\lesssim\lambda_1$, where we have used $m=n^\gamma$ with $\gamma<1$ and $\ell=n^{\frac{\theta}{\alpha}}$. This implies
\begin{eqnarray*}
R_{\sigh_m,\ell}&{}\lesssim_{\Pb^n\times\Lambda^m}{}&\inf\left\{n^{-\gamma}+t^{-1/\alpha}(n^{-\theta}+t)^2:\frac{\log n}{n^\gamma}\lesssim t\lesssim\lambda_1\right\}\nonumber\\
&{}\lesssim{}&
\begin{cases}
n^{-\gamma}+n^{\frac{\theta}{\alpha}-2\theta},\qquad\qquad\, \theta<\gamma\\                          
n^{-\gamma}+\left(\frac{\log n}{n}\right)^{2-\frac{1}{\alpha}},\qquad \theta\ge\gamma
\end{cases}
\end{eqnarray*}
and the result follows by considering the cases of $\gamma\ge \theta\left(2-\frac{1}{\alpha}\right)$ and $\gamma< \theta\left(2-\frac{1}{\alpha}\right)$.
\subsection{Proof of Theorem \ref{thm:rff main thm metric 2}}\label{subsec:proof:metric2}
\emph{(i)} The result follows from the proof of Proposition~\ref{pro:solution}\emph{(iii)}.
\vspace{1mm}\\
\emph{(ii)} \underline{\emph{Upper bound:}} By defining $\sigh_\ell^{-1}:=\sum_{i=1}^\ell\frac{1}{\lambdah_i}\phih_i\oh\phih_i,$
we have,
\begin{eqnarray}
S_{\sigh,\ell}&{}={}&\bb{E}\norm{\id \overline{k}(\cdot,X)-\id\sigh_\ell^{-1}\id^*\id(k(\cdot,X)-\widehat{m}_\Pb)}^2_{\lp}\nonumber
\\
&{}\stackrel{(\star)}{=}{}&\E\norm{\id(I-\sigh_\ell^{-1}\Sigma)\overline{k}(\cdot,X)}^2_{\lp}+\norm{\id\sigh_\ell^{-1}\Sigma(m_\Pb-\widehat{m}_\Pb)}_{\lp}^2\nonumber
\\
&{}\stackrel{(\dag)}{=}{}&\norm{\Sigma^{1/2}(I-\sigh_\ell^{-1}\Sigma)\Sigma^{1/2}}_{\HSH}^2+\norm{\Sigma^{1/2}\sigh_\ell^{-1}\Sigma(m_\Pb-\widehat{m}_\Pb)}_{\Hk}^2,\label{eq:ekpca metric 2 decomp.1}
\end{eqnarray}
where we have used $\id^*\id=\Sigma$ (see Proposition~\ref{pro:id}\emph{(iii)}) in $(\star)$ and Lemma~\ref{rf cov proj rewrite} in $(\dag)$. We can decompose the first term of (\ref{eq:ekpca metric 2 decomp.1}) as
\begin{eqnarray}
\norm{\Sigma^{1/2}(I-\sigh_\ell^{-1}\Sigma)\Sigma^{1/2}}_{\HSH}^2&{}\le{}&2\left(\circled{\small{1}}+\circled{\small{b}}\right),\label{eq:ekpca metric 2 decomp.2}
\end{eqnarray}
where 
$\circled{\small{b}}:=\norm{\Sigma^{1/2}(P_\ell(\sigh)-\sigh_\ell^{-1}\Sigma)\Sigma^{1/2}}_{\HSH}^2.$ 
Therefore,
\begin{equation}\label{eq:ekpca metric 2 A bnd}
\Pb^n\left\{(X_i)_{i=1}^n:\circled{\small{a}}\le9\Cal{N}_\Sigma(t)(\lambda_{\ell+1}+t)^2\right\}\ge1-2\delta,
\end{equation}
where $\frac{140\kappa}{n}\log\frac{16\kappa n}{\delta}\le t\le\norm{\Sigma}_\OPH$. For \circled{\small{b}} we write, 
\begin{eqnarray}
\circled{\small{b}}&{}={}&\norm{\Sigma^{1/2}\sigh_\ell^{-1}(\sigh-\Sigma)\Sigma^{1/2}}_{\HSH}^2\nonumber\\
&{}\le{}&\norm{\Sigma^{1/2}(\Sigma+tI)^{-1/2}}^2_{\OPH}\norm{(\Sigma+tI)^{1/2}(\sigh+tI)^{-1/2}}^2_{\OPH}\nonumber\\
&{}{}&\qquad\times\norm{(\sigh+tI)^{1/2}\sigh_\ell^{-1}(\sigh+tI)^{1/2}}_\OPH^2\norm{(\sigh+tI)^{-1/2}(\Sigma+tI)^{1/2}}^2_{\OPH}\nonumber\\
&{}{}&\qquad\qquad\times\norm{(\Sigma+tI)^{-1/2}(\sigh-\Sigma)\Sigma^{1/2}}^2_{\HSH}\nonumber\\
&{}\stackrel{(*)}{\le}{}&4\sup_{i\le\ell}\left(\frac{\lambdah_i+t}{\lambdah_i}\right)^2\times\circled{\small{c}}=4\left(\frac{\lambdah_\ell+t}{\lambdah_\ell}\right)^2\times\circled{\small{c}}\nonumber\\
&{}\stackrel{(\ddagger)}{\le}{}&36\left(\frac{\lambda_\ell+t}{\lambda_\ell-t}\right)^2\times\circled{\small{c}}\stackrel{(\star)}{\le} 144\times\circled{\small{c}},\label{Eq:bc}
\end{eqnarray}
which holds with probability at least $1-2\delta$ over the choice of $(X_i)^n_{i=1}$, where we used Lemma~\ref{lem:cent 1}\emph{(ii)} in $(*)$, Lemma~\ref{lem:cent 1}\emph{(iv, v)} in $(\ddagger)$ and the assumption that $t\le \frac{1}{3}\lambda_\ell$ in $(\star)$, with $$\circled{\small{c}}:=\norm{(\Sigma+tI)^{-1/2}(\sigh-\Sigma)\Sigma^{1/2}}^2_{\HSH}.$$ In the following, we will obtain two different bounds for \circled{\small{c}} based on different decompositions, which we then combine by choosing the minimum of them. Applying \eqref{Eq:hs-op} to \circled{\small{c}} yields
\begin{equation}
 \Pb^n\left\{(X_i)_{i=1}^n:\circled{\small{c}}\le\frac{128\kappa^{5/2}\Cal{N}_\Sigma(t)\log\frac{2}{\delta}}{n\sqrt{t}}+\frac{4096\kappa^{3}\log^2\frac{3}{\delta}}{n^2t}
\right\}\ge1-2\delta.\label{Eq:c1}
\end{equation}
\circled{\small{c}} can be alternately bounded as $$\circled{\small{c}}\le \norm{(\Sigma+tI)^{-1/2}}^2_{\OPH}\norm{\sigh-\Sigma}^2_{\HSH}\norm{\Sigma}^2_{\OPH},$$ yielding
\begin{equation}
\Pb^n\left\{(X_i)_{i=1}^n:\circled{\small{c}}\le\frac{256\kappa^4\log\frac{2}{\delta}}{nt}+\frac{8192\kappa^4\log^2\frac{3}{\delta}}{n^2t}\right\}\ge 1-2\delta \label{Eq:c2} 
\end{equation}
through an application of Theorem~\ref{thm:bernstein U-stat}\emph{(ii)}. Combining \eqref{Eq:c1} and \eqref{Eq:c2} provides
\begin{eqnarray*}
&{}{}&\Pb^n\left\{(X_i)_{i=1}^n:\circled{\small{c}}\le 256\kappa^{5/2}\log\frac{2}{\delta}\left[\frac{\Cal{N}_\Sigma(t)}{n\sqrt{t}}\wedge\frac{\kappa^{3/2}}{nt}\right]+\frac{8192\kappa^3(\kappa\wedge 1)\log^2\frac{3}{\delta}}{n^2t}\right\}\ge 1-4\delta,\nonumber  
\end{eqnarray*}
using which in \eqref{Eq:bc} yields
\begin{eqnarray}
&{}{}&\Pb^n\left\{(X_i)_{i=1}^n:\circled{\small{b}}\le 144\left[256\kappa^{5/2}\log\frac{2}{\delta}\left[\frac{\Cal{N}_\Sigma(t)}{n\sqrt{t}}\wedge\frac{\kappa^{3/2}}{nt}\right]+\frac{8192\kappa^3(\kappa\wedge 1)\log^2\frac{3}{\delta}}{n^2t}\right]\right\}\nonumber\\
&{}{}&\qquad\qquad\qquad\qquad\ge1-6\delta.\label{Eq:b}
\end{eqnarray}
To bound the second term of (\ref{eq:ekpca metric 2 decomp.1}) we have
\begin{eqnarray}
\norm{\Sigma^{1/2}\sigh_\ell^{-1}\Sigma(m_\Pb-\widehat{m}_\Pb)}_\Hk^2&{}\le{}&\norm{\Sigma^{1/2}\sigh_\ell^{-1}\Sigma^{1/2}}_\OPH^2\norm{\Sigma^{1/2}(m_\Pb-\widehat{m}_\Pb)}_\Hk^2\nonumber
\\
\qquad\quad&{}\le{}&\norm{\Sigma^{1/2}(\Sigma+tI)^{-1/2}}_\OPH^4\norm{(\Sigma+tI)^{1/2}(\sigh+tI)^{-1/2}}_\OPH^4\nonumber
\\
\qquad\quad&{}{}&\qquad\times \norm{(\sigh+tI)^{1/2}\sigh_\ell^{-1}(\sigh+tI)^{1/2}}_\OPH^2\norm{\Sigma^{1/2}(m_\Pb-\widehat{m}_\Pb)}_\Hk^2\nonumber
\\
\qquad\quad&{}\stackrel{(\diamond)}{\le}{}&144\Vert\Sigma\Vert_{\Cal{L}^\infty(\Cal{H})}\norm{m_\Pb-\widehat{m}_\Pb}_\Hk^2\le 288\kappa\norm{m_\Pb-\widehat{m}_\Pb}_\Hk^2,\label{Eq:2ndterm}
\end{eqnarray}
which holds with probability at least $1-2\delta$ over the choice of $(X_i)^n_{i=1}$, wherein we have used Lemma~\ref{lem:cent 1}\emph{(ii)} and the bound in \eqref{Eq:bc} on $\norm{(\sigh+tI)^{1/2}\sigh_\ell^{-1}(\sigh+tI)^{1/2}}_\OPH^2$ in $(\diamond)$. Applying Lemma~\ref{lem:kernel mean bnd}\emph{(i)} to \eqref{Eq:2ndterm}, we obtain
\begin{equation}
\Pb^n\left\{(X_i)_{i=1}^n:\norm{\Sigma^{1/2}\sigh_\ell^{-1}\Sigma(m_\Pb-\widehat{m}_\Pb)}_\Hk^2\le \frac{9216\kappa^2\log\frac{2}{\delta}}{n}\right\}\ge 1-3\delta.\label{Eq:2ndfinal} 
\end{equation}
Combining \eqref{eq:ekpca metric 2 decomp.1}--\eqref{eq:ekpca metric 2 A bnd}, \eqref{Eq:b} and \eqref{Eq:2ndfinal}, yields the result.\vspace{1mm}
%
%
%
\\
\underline{\emph{Lower bound:}} It follows from \eqref{eq:ekpca metric 2 decomp.1} that $$S_{\sigh,\ell}\ge \left\Vert \Sigma^{1/2}(I-\sigh^{-1}_\ell\Sigma)\Sigma^{1/2}\right\Vert^2_{\HSH}=\Cal{R}^{\Sigma}_{1,1,1}(\sigh^{-1}_\ell),$$
where the equality follows from the definition in Lemma~\ref{lem:optima}. The result follows from Lemma~\ref{lem:optima} by noting that $\Cal{R}^{\Sigma}_{1,1,1}(\sigh^{-1}_\ell)\ge \Cal{R}^{\Sigma}_{1,1,1}(\Sigma^{-1}_\ell)=S_{\Sigma,\ell}=\sum_{i>\ell}\lambda^2_i$.
\vspace{1mm}\\
\emph{(iii)}  Define $\Sigma_{m,\ell}^{-1}=\sum_{i=1}^\ell\frac{1}{\lambda_{i,m}}\phi_{i,m}\ohm\phi_{i,m}$. Then
\begin{eqnarray*}
S_{\Sigma_m,\ell}
&{}={}&\E\norm{\id\overline{k}(\cdot,X)-\tS\Sigma_{m,\ell}^{-1}\tS^*\tS\overline{k}_m(\cdot,X)}_{\lp}^2=\E\norm{\id\overline{k}(\cdot,X)-\tS\Sigma_{m,\ell}^{-1}\Sigma_m\overline{k}_m(\cdot,X)}_{\lp}^2\nonumber
\\
&{}={}&\E\norm{\id\overline{k}(\cdot,X)-\tS P_\ell(\Sigma_m)\overline{k}_m(\cdot,X)}_{\lp}^2=R_{\Sigma_m,\ell}\nonumber
\end{eqnarray*}
and the result follows from Theorem~\ref{thm:rff main thm metric 1}\emph{(iii)}.
\vspace{1mm}
%
%
\\
\emph{(iv)} \underline{\emph{Upper bound:}} Define 
$\sigh_{m,\ell}^{-1}=\sum_{i=1}^\ell\frac{1}{\lambdah_{i,m}}\phih_{i,m}\ohm\phih_{i,m}$. Then
\begin{eqnarray}
S_{\sigh_m,\ell}
&{}={}&\E\norm{\id\overline{k}(\cdot,X)-\tS\sigh_{m,\ell}^{-1}\tS^*\tS(k_m(\cdot,X)-\widehat{m}_{\Pb,m})}_{\lp}^2\le3\left(\circled{\small{3}}+\circled{\small{e}}+\circled{\small{f}}\right),\label{eq:rf-ekpca metric 2 decomp}
\end{eqnarray}
where 
$\circled{\small{e}}:=\E\norm{(I-\tS\sigh_{m,\ell}^{-1}\tS^*)\tS\overline{k}_m(\cdot,X)}_{\lp}^2$ and $\circled{\small{f}}:=\norm{\tS\sigh_{m,\ell}^{-1}\tS^*\tS(m_{\Pb,m}-\widehat{m}_{\Pb,m})}_{\lp}^2.$
%
Note that $\circled{\small{3}}$ which can be bounded using Lemma~\ref{lem:rf kernel mean bnd} and  
%
%
%
\begin{eqnarray}
\circled{\small{e}}&{}={}&\E\norm{\tS(I-\sigh_{m,\ell}^{-1}\Sigma_m)\overline{k}_m(\cdot,X)}_{\lp}^2\nonumber\\&{}\stackrel{(\clubsuit)}{=}{}&\norm{\Sigma_m^{1/2}(I-\sigh_{m,\ell}^{-1}\Sigma_m)\Sigma_m^{1/2}}_\HSM^2
\le2\left(\circled{\small{5}}+\circled{\small{e2}}\right)
,\label{eq:rf-ekpca metric 2 E decomp}
\end{eqnarray}
where we have used Lemma~\ref{rf cov proj rewrite} in $(\clubsuit)$ and $\circled{\small{e2}}:=\norm{\Sigma_m^{1/2}(P_\ell(\sigh_m)-\sigh_{m,\ell}^{-1}\Sigma_m)\Sigma_m^{1/2}}_\HSM^2.$ 
Note that the bound on $\circled{\small{5}}$, follows from \eqref{Eq:final5}. 
%
%
%
By handling \circled{\small{e2}} in a similar manner as \circled{\small{b}} in \emph{(ii)}, by conditioning on $(\theta_i)^m_{i=1}$, for $\frac{140\kappa}{n}\log\frac{16\kappa n}{\delta}\le t\le \Vert\Sigma_m\Vert_{\Cal{L}^\infty(\Cal{H}_m)}$, with probability at least $1-4\delta$ over the choice of $(X_i)^n_{i=1}$, we obtain
\begin{eqnarray}
\circled{\small{e2}}&{}={}&\norm{\Sigma_m^{1/2}\sigh_{m,\ell}^{-1}(\sigh_m-\Sigma_m)\Sigma_m^{1/2}}_\HSM^2\nonumber\\
&{}\le{}&36\left(\frac{\lambda_{m,\ell}+t}{\lambda_{m,\ell}-t}\right)^2\left[\frac{128\kappa^{5/2}\Cal{N}_{\Sigma_m}(t)\log\frac{2}{\delta}}{n\sqrt{t}}+\frac{4096\kappa^{3}\log^2\frac{3}{\delta}}{n^2t}\right].\label{Eq:interim}
%
\end{eqnarray}
By unconditioning w.r.t.~$(\theta_i)^m_{i=1}$ in the above inequality, for $\frac{86\kappa}{m}\log\frac{16\kappa m}{\delta}\le t\le\norm{\Sigma}_\OPH$, with probability at least $1-7\delta$ jointly over the choice of $((\theta_i)^m_{i=1},(X_i)^n_{i=1})$, we obtain
\begin{eqnarray}
\circled{\small{e2}}&{}\le{}& 36\left(\frac{3\lambda_{\ell}+3t}{\lambda_{\ell}-3t}\right)^2\left[\frac{4096\kappa^{3}\log^2\frac{3}{\delta}}{n^2t}+\frac{256\kappa^{5/2}\Cal{D}(t)\log\frac{2}{\delta}}{n\sqrt{t}}\right]\nonumber\\
&{}\le{}& 900\left[\frac{4096\kappa^{3}\log^2\frac{3}{\delta}}{n^2t}+\frac{256\kappa^{5/2}\Cal{D}(t)\log\frac{2}{\delta}}{n\sqrt{t}}\right],\label{Eq:e22}
\end{eqnarray}
where the first inequality follows from applying Lemma~\ref{lem:1 rff}\emph{(ii)--(iv)} to \eqref{Eq:interim}, the second inequality follows by using $t\le \frac{\lambda_\ell}{9}$ and we used $\Cal{D}(t):=\frac{16\kappa\log\frac{2}{\delta}}{tm}+\sqrt{\frac{8\kappa\Cal{N}_\Sigma(t)\log\frac{2}{\delta}}{tm}}+\Cal{N}_\Sigma(t)$.
\circled{\small{e2}} can be alternately bounded as follows. By conditioning on $(\theta_i)^m_{i=1}$, for $\frac{140\kappa}{n}\log\frac{16\kappa n}{\delta}\le t\le \Vert\Sigma_m\Vert_{\Cal{L}^\infty(\Cal{H}_m)}$, with probability at least $1-4\delta$ over the choice of $(X_i)^n_{i=1}$, we obtain
\begin{equation*}
\circled{\small{e2}}
\le\frac{36\lambda^2_{m,1}}{t}\left(\frac{\lambda_{m,\ell}+t}{\lambda_{m,\ell}-t}\right)^2\left[\frac{64\kappa^2\log\frac{2}{\delta}}{n}+\frac{2048\kappa^2\log^2\frac{3}{\delta}}{n^2}\right].
\end{equation*}
By unconditioning w.r.t.~$(\theta_i)^m_{i=1}$ in the above inequality, for $\frac{86\kappa}{m}\log\frac{16\kappa m}{\delta}\le t\le\norm{\Sigma}_\OPH$, with probability at least $1-5\delta$ jointly over the choice of $((\theta_i)^m_{i=1},(X_i)^n_{i=1})$, we obtain
\begin{eqnarray}
\circled{\small{e2}}
&{}\le{}&\frac{400\kappa^4}{t}\left(\frac{3\lambda_\ell+3t}{\lambda_\ell-3t}\right)^2\left[\frac{64\log\frac{2}{\delta}}{n}+\frac{2048\log^2\frac{3}{\delta}}{n^2}\right]\nonumber\\&{}\le{}& 10^4\kappa^4\left[\frac{64\log\frac{2}{\delta}}{nt}+\frac{2048\log^2\frac{3}{\delta}}{n^2t}\right].\label{Eq:e22-1}
\end{eqnarray}
Combining \eqref{Eq:e22} and \eqref{Eq:e22-1}, with probability at least $1-12\delta$ jointly over the choice of $((\theta_i)^m_{i=1},(X_i)^n_{i=1})$, we have
\begin{eqnarray}
\circled{\small{e2}}&{}\le{}& 24\times 10^4\left[\left(\frac{\kappa^4\log\frac{2}{\delta}}{nt}\right)\wedge \left(\frac{\kappa^{5/2}\Cal{D}(t)\log\frac{2}{\delta}}{n\sqrt{t}}\right)\right]+\frac{8\times 10^6\kappa^3(1\wedge k)\log^2\frac{3}{\delta}}{n^2t}.
\label{Eq:e22-final}
%
\end{eqnarray}
\circled{\small{f}} can be bounded as 
\begin{eqnarray}
\circled{f}&{}={}&\norm{\Sigma_m^{1/2}\sigh_{m,\ell}^{-1}\Sigma_m(m_{\Pb,m}-\widehat{m}_{\Pb,m})}_{\Hk_m}^2\nonumber
\\
&{}\le{}&\norm{\Sigma_m^{1/2}\sigh_{m,\ell}^{-1}\Sigma_m^{1/2}}_\OPHm^2\norm{\Sigma_m^{1/2}(m_{\Pb,m}-\widehat{m}_{\Pb,m})}_{\Hk_m}^2\nonumber
\\
&{}\le{}&\lambda_{m,1}\norm{\Sigma_m^{1/2}(\Sigma_m+tI)^{-1/2}}_\OPHm^4\norm{m_{\Pb,m}-\widehat{m}_{\Pb,m}}_{\Hk_m}^2\norm{(\Sigma_m+tI)^{1/2}(\sigh_m+tI)^{-1/2}}_\OPHm^4\nonumber\\
&{}{}&\qquad\times\norm{(\sigh_m+tI)^{1/2}\sigh_{m,\ell}^{-1}(\sigh_m+tI)^{1/2}}_\OPHm^2.\label{eq:rf-ekpca metric 2 F decomp 2}
\end{eqnarray}
By conditioning on $(\theta_i)^m_{i=1}$, for $\frac{140\kappa}{n}\log\frac{16\kappa n}{\delta}\le t\le \Vert\Sigma_m\Vert_{\Cal{L}^\infty(\Cal{H}_m)}$, with probability at least $1-3\delta$ over the choice of $(X_i)^n_{i=1}$, we obtain
\begin{equation*}
\circled{f}\le \frac{3840\kappa^2\log\frac{2}{\delta}}{n} \left(\frac{\lambda_{m,\ell}+t}{\lambda_{m,\ell}-t}\right)^2,\nonumber
\end{equation*}
where we used the fact that $\lambda_{m,1}\le\frac{10\kappa}{3}$ (see the proof of Theorem~\ref{thm:rff main thm metric 1}\emph{(iv)}) and employed Lemmas~\ref{lem:cent 1}\emph{(ii, iv, v)} and \ref{lem:kernel mean bnd}\emph{(ii)}. By unconditioning w.r.t. $(\theta_i)^m_{i=1}$ in the above inequality, with probability at least $1-4\delta$ jointly over the choice of $((\theta_i)^m_{i=1},(X_i)^n_{i=1})$, we obtain 
\begin{equation}
\circled{f}\le \frac{3840\kappa^2\log\frac{2}{\delta}}{n} \left(\frac{3\lambda_{\ell}+3t}{\lambda_{\ell}-3t}\right)^2\le \frac{96000\kappa^2\log\frac{2}{\delta}}{n},\label{Eq:f}
\end{equation}
where we applied Lemma~\ref{lem:1 rff}\emph{(ii, iii)} and $\frac{86\kappa}{m}\log\frac{16\kappa m}{\delta}\le t\le\norm{\Sigma}_\OPH$. The result therefore follows by combining \eqref{eq:rf-ekpca metric 2 decomp}, \eqref{eq:rf-ekpca metric 2 E decomp}, \eqref{Eq:e22-final} and \eqref{Eq:f} under the condition that $\frac{140\kappa}{n}\log\frac{8n}{\delta}\vee\frac{86\kappa}{m}\log\frac{8m}{\delta}\le t\le\frac{\lambda_\ell}{9}$ and $n\wedge m\ge8\log\frac{1}{\delta}$.\vspace{1mm}\\
\underline{\emph{Lower bound:}} As carried out in the proof of the lower bound of \emph{(iii)}, it can be shown that
\begin{eqnarray*}
S_{\sigh_m,\ell}&{}={}&\E\norm{\id\overline{k}(\cdot,X)-\tS \sigh^{-1}_{m,\ell}\Sigma_m\widetilde{k}_m(\cdot,X)}_{L^2(\Pb)}^2\nonumber\\
&{}\ge{}& \left(\sqrt{\circled{\small{e}}+\circled{\small{f}}}-\sqrt{\circled{\small{3}}}\right)^2\ge \left(\sqrt{\circled{\small{e}}}-\sqrt{\circled{\small{3}}}\right)^2\nonumber\\
&{}\stackrel{(*)}{\ge}{}& \left(\sqrt{\sum_{i>\ell}\lambda^2_{m,i}}-\sqrt{\circled{\small{3}}}\right)^2,\nonumber
\end{eqnarray*}
where $(*)$ follows by applying Lemma~\ref{lem:optima} to \circled{\small{e}}. 
The result, therefore, follows from \eqref{Eq:qq}.
\subsection{Proof of Theorem~\ref{Thm:variant-ER}}\label{Subsec:proof-ER-variant}
Define $T_{\sigh,\ell}=\eu{S}\left(\frac{\id\phih_1}{\sqrt{\lambdah_1}},\ldots,\frac{\id\phih_\ell}{\sqrt{\lambdah_\ell}}\right)$, $T_{\Sigma_m,\ell}$ $=\eu{S}\left(\frac{\A\phi_{m,1}}{\sqrt{\lambda_{m,1}}},\ldots,\frac{\A\phi_{m,\ell}}{\sqrt{\lambda_{m,\ell}}}\right)$ and $$T_{\sigh_m,\ell}=\eu{S}\left(\frac{\A\phih_{m,1}}{\sqrt{\lambdah_{m,1}}},\ldots,\frac{\A\phih_{m,\ell}}{\sqrt{\lambdah_{m,\ell}}}\right).$$
\emph{(i)} By adding and subtracting $\id\widetilde{k}(\cdot,X)$ to the first argument of the inner product in $T_{\sigh,\ell}$, we obtain
\begin{eqnarray*}
T_{\sigh,\ell}&{}\le{}& 2S_{\sigh,\ell}+2\bb{E}\norm{\sum_{i=1}^{\ell}\inner{\id \overline{k}(\cdot,X)-\id \widetilde{k}(\cdot,X)}{\frac{\id\phih_i}{\sqrt{\lambdah_i}}}_{\lp}\frac{\id\phih_i}{\sqrt{\lambdah_i}}}^2_{L^2(\Pb)}\nonumber\\
&{}={}&2S_{\sigh,\ell}+2\norm{\id\sigh^{-1}_\ell\Sigma(m_\bb{P}-\mh_\bb{P})}^2_{\lp}
=2S_{\sigh,\ell}+2\norm{\Sigma^{1/2}\sigh^{-1}_\ell\Sigma(m_\bb{P}-\mh_\bb{P})}^2_{\Cal{H}}\nonumber
\end{eqnarray*}
and the result follows from \eqref{Eq:2ndfinal}. Also note that $$T_{\sigh,\ell}=\left\Vert \Sigma^{1/2}(I-\sigh^{-1}_\ell\Sigma)\Sigma^{1/2}\right\Vert^2_{\HSH}$$ and therefore the lower bound follows from the proof of the lower bound of $S_{\sigh,\ell}$.
\vspace{1mm}\\
\emph{(ii)} As above, adding and subtracting $\A \overline{k}_m(\cdot,X)$ to the first argument of the inner product in $T_{\Sigma_m,\ell}$, we obtain
\begin{eqnarray*}
T_{\Sigma_m,\ell}&{}\le{}& 2S_{\Sigma_m,\ell}+2\bb{E}\norm{\sum_{i=1}^{\ell}\inner{\id \overline{k}(\cdot,X)-\A \overline{k}_m(\cdot,X)}{\frac{\A\phi_{m,i}}{\sqrt{\lambda_{m,i}}}}_{\lp}\frac{\A\phi_{m,i}}{\sqrt{\lambda_{m,i}}}}^2_{L^2(\Pb)}\nonumber
\end{eqnarray*}
\begin{eqnarray*}
&{}={}&2S_{\Sigma_m,\ell}+2\bb{E}\norm{\A\Sigma^{-1}_{m,\ell}\A^*(\id \overline{k}(\cdot,X)-\A \overline{k}_m(\cdot,X))}^2_{\lp}\nonumber\\
&{}\le{}&2S_{\Sigma_m,\ell}+2\norm{\Sigma_m\Sigma^{-1}_{m,\ell}}^2_{\Cal{L}^\infty(\Cal{H}_m)}\bb{E}\norm{\id \overline{k}(\cdot,X)-\A \overline{k}_m(\cdot,X)}^2_{\lp}\nonumber
\end{eqnarray*}
and the result follows by noting that $\Sigma_m\Sigma^{-1}_{m,\ell}=P_\ell(\Sigma_m)$, $\Vert P_\ell(\Sigma_m)\Vert_{\Cal{L}^\infty(\Cal{H}_m)}=1$ and applying Lemma~\ref{lem:rf kernel mean bnd}. For the lower bound, note that 
\begin{eqnarray*}
T_{\Sigma_m,\ell}&{}={}&\bb{E}\Vert \id \overline{k}(\cdot,X)-\A\Sigma^{-1}_{m,\ell}\A^*\id \overline{k}(\cdot,X)\Vert^2_{\lp}\nonumber\\
&{}={}& S_{\Sigma_m,\ell}+\bb{E}\left\Vert\A \Sigma^{-1}_{m,\ell}\A^*\left[\id \overline{k}(\cdot,X)-\A\overline{k}_m(\cdot,X)\right]\right\Vert^2_{\lp}\nonumber\\
&{}{}&\qquad-2\bb{E}\left\langle \id\overline{k}(\cdot,X)-\A \Sigma^{-1}_{m,\ell}\A^*\A\overline{k}_m(\cdot,X),\A \Sigma^{-1}_{m,\ell}\A^*\left[\id \overline{k}(\cdot,X)-\A\overline{k}_m(\cdot,X)\right]\right\rangle_{\lp}\nonumber\\
&{}\ge{}&S_{\Sigma_m,\ell}+\circled{\small{g}}-2\sqrt{S_{\Sigma_m,\ell}}\sqrt{\circled{\small{g}}}\nonumber\\
&{}={}&\left(\sqrt{S_{\Sigma_m,\ell}}-\sqrt{\circled{\small{g}}}\right)^2\ge \left(\sqrt{S_{\Sigma_m,\ell}}-\sqrt{\circled{\small{3}}}\right)^2,\nonumber
\end{eqnarray*}
where we used $\circled{\small{g}}:=\bb{E}\left\Vert\A \Sigma^{-1}_{m,\ell}\A^*\left[\id \overline{k}(\cdot,X)-\A\overline{k}_m(\cdot,X)\right]\right\Vert^2_{\lp}\le \circled{\small{3}}\times\Vert \A \Sigma^{-1}_{m,\ell}\A^*\Vert^2_{\OPL}=\circled{\small{3}}\times \Vert \Sigma_m\Sigma^{-1}_{m,\ell}\Vert^2_{\Cal{L}^\infty(\Cal{H}_m)}=\circled{\small{3}}$. Since $\sqrt{S_{\Sigma_m,\ell}}\ge \sqrt{\circled{\small{4}}}-\sqrt{\circled{\small{3}}}$, we have $$T_{\Sigma_m,\ell}\ge \left(\sqrt{\circled{\small{4}}}-2\times\sqrt{\circled{\small{3}}}\right)^2.$$ Based on the calculation we made for the lower bound of $R_{\Sigma_m,\ell}$, for $\frac{1}{2}\sqrt{\sum_{i>\ell}\lambda^2_{i}}\ge 20\kappa\sqrt{\frac{\log\frac{2}{\delta}}{m}}$, we obtain $T_{\Sigma_m,\ell}\ge \frac{1}{4}\sum_{i>\ell}\lambda^2_i$, which  holds with probability at least $1-3\delta$ over the choice of $(\theta_i)^m_{i=1}$.
\vspace{1mm}\\
\emph{(iii)} Doing as above, we obtain
\begin{eqnarray*}
 T_{\sigh_m,\ell}&{}\le{}& 2S_{\sigh_m,\ell}+2\bb{E}\norm{\sum_{i=1}^{\ell}\inner{\id \overline{k}(\cdot,X)-\A \widetilde{k}_m(\cdot,X)}{\frac{\A\phi_{m,i}}{\sqrt{\lambda_{m,i}}}}_{\lp}\frac{\A\phi_{m,i}}{\sqrt{\lambda_{m,i}}}}^2_{L^2(\Pb)}\nonumber\\
 &{}={}&2S_{\sigh_m,\ell}+2\bb{E}\norm{\A\sigh^{-1}_{m,\ell}\A^*(\id \overline{k}(\cdot,X)-\A \widetilde{k}_m(\cdot,X))}^2_{\lp}\nonumber\\
 &{}\le{}&2S_{\sigh_m,\ell}+4\norm{\A\sigh^{-1}_{m,\ell}\A^*}^2_{\Cal{L}^\infty(\lp)}\bb{E}\norm{\id \overline{k}(\cdot,X)-\A \overline{k}_m(\cdot,X)}^2_{\lp}\nonumber\\
 &{}{}&\qquad\qquad+4\bb{E}\norm{\A\sigh^{-1}_{m,\ell}\A^*(\A\overline{k}(\cdot,X)-\A\widetilde{k}(\cdot,X))}^2_{\lp}\nonumber
 \end{eqnarray*}
 \begin{eqnarray*}
 &{}\le{}& 2S_{\sigh_m,\ell}+4\norm{\Sigma^{1/2}_m\sigh^{-1}_{m,\ell}\Sigma^{1/2}_m}^2_{\Cal{L}^\infty(\Cal{H}_m)}\bb{E}\norm{\id \overline{k}(\cdot,X)-\A \overline{k}_m(\cdot,X)}^2_{\lp}\nonumber\\
 &{}{}&\qquad\qquad+4\norm{\Sigma^{1/2}_m\sigh^{-1}_{m,\ell}\Sigma_m(m_{\bb{P},m}-\widehat{m}_{\Pb,m})}^2_{\Cal{H}_m}.\nonumber
 \end{eqnarray*}
 By applying Lemma~\ref{lem:rf kernel mean bnd}, the result follows from \eqref{eq:rf-ekpca metric 2 F decomp 2}---see the second line in the chain of equations leading to \eqref{eq:rf-ekpca metric 2 F decomp 2}---and \eqref{Eq:f}. For the lower bound, note that
$ T_{\sigh_m,\ell}\ge \left(\sqrt{S_{\sigh_m,\ell}}-\sqrt{\circled{\small{h}}}\right)^2\ge \left(\sqrt{\sum_{i>\ell}\lambda^2_{m,i}}-\sqrt{\circled{\small{3}}}-\sqrt{\circled{\small{h}}}\right)^2$, where \begin{eqnarray*}\circled{\small{h}}&{}:={}&
\bb{E}\left\Vert\A \sigh^{-1}_{m,\ell}\A^*\left[\id \overline{k}(\cdot,X)-\A\overline{k}_m(\cdot,X)\right]\right\Vert^2_{\lp}\\
&{}\le{}&\norm{\A\sigh^{-1}_{m,\ell}\A^*}^2_{\Cal{L}^\infty(\lp)}\times\circled{\small{3}}\nonumber\\ &{}={}&\norm{\Sigma^{1/2}_m\sigh^{-1}_{m,\ell}\Sigma^{1/2}_m}^2_{\Cal{L}^\infty(\Cal{H}_m)}\times\circled{\small{3}}\lesssim_{\Lambda^m}\circled{\small{3}}.\nonumber
\end{eqnarray*}
The result therefore follows by choosing $m$ sufficiently larger than $\frac{1}{\sum_{i>\ell}\lambda^2_i}$.
\subsection{Proof of Corollary \ref{rff poly decay corollary metric 2}}\label{sec:rff poly decay proof2}
\emph{(i)} and \emph{(iii)} are exactly same as that of the proof of Corollary~\ref{rff poly decay corollary}.\vspace{1mm}\\
\emph{(ii)} From Theorem \ref{thm:rff main thm metric 2}\emph{(ii)} we have
$$S_{\sigh,\ell}\lesssim_{\Pb^n}\frac{1}{n}+\frac{1}{n^2t}+\left(\frac{\Cal{N}_\Sigma(t)}{n\sqrt{t}}\wedge \frac{1}{nt}\right)+\Cal{N}_\Sigma(t)(\lambda_\ell+t)^2,$$
with $\frac{\log n}{n}\lesssim t\le\frac{\lambda_\ell}{3}$. Clearly $\frac{1}{n^2t}\lesssim \frac{1}{n}$. Using $\Cal{N}_\Sigma(t)\lesssim t^{-1/\alpha}$ from Lemma~\ref{N(t)}\emph{(i)}, it follows that
$$S_{\sigh,\ell}\lesssim_{\Pb^n}\inf\left\{t^{-1/\alpha}(n^{-\theta}+t)^2+\left(\frac{t^{-(\frac{1}{\alpha}+\frac{1}{2})}}{n}\wedge \frac{1}{nt}\right)+\frac{1}{n}:\frac{\log n}{n}\lesssim t\lesssim n^{-\theta}\right\}.$$
It is clear that both $\frac{t^{-(\frac{1}{\alpha}+\frac{1}{2})}}{n}$ and $\frac{1}{nt}$ dominate $n^{-1}$ and using $t\lesssim n^{-\theta}$ in the first term, we obtain
\begin{eqnarray*}
S_{\sigh,\ell}&{}\lesssim_{\Pb^n}{}&\inf\left\{t^{-1/\alpha}n^{-2\theta}+\frac{t^{-\frac{1}{\alpha'}}}{n}:\frac{\log n}{n}\lesssim t\lesssim n^{-\theta}\right\}\\
&{}={}&n^{-2\theta\left(1-\frac{1}{2\alpha}\right)}+n^{-\left(1-\frac{\theta}{\alpha'}\right)},
\end{eqnarray*}
where $\frac{1}{\alpha'}:=\left(\frac{1}{\alpha}+\frac{1}{2}\right)\wedge 1$
and the result follows.\vspace{1mm}
\\
\emph{(iv)} From Theorem \ref{thm:rff main thm metric 2}\emph{(iv)} we have
$$S_{\sigh_m,\ell}\lesssim_{\Pb^n\times\Lambda^m}\frac{1}{n}+\frac{1}{m}+\frac{1}{n^2t}+\left(\frac{\Cal{A}(t)}{n\sqrt{t}}\wedge \frac{1}{nt}\right)+\Cal{A}(t)(\lambda_\ell+t)^2$$
for $\frac{\log n}{n}\vee\frac{\log m}{m}\lesssim t\lesssim\lambda_\ell$, where $\Cal{A}(t)=\Cal{N}_\Sigma(t)+\sqrt{\frac{\Cal{N}_\Sigma(t)}{tm}}+\frac{1}{tm}$. From the proof of Corollary~\ref{rff poly decay corollary}\emph{(iv)}, we have $\Cal{A}(t)\lesssim t^{-1/\alpha}$. Also it is obvious that $\frac{1}{n^2t}\lesssim \frac{1}{n}$. Therefore, 
\begin{eqnarray*}
S_{\sigh_m,\ell}&{}\lesssim_{\Pb^n\times\Lambda^m}{}&\inf\left\{\frac{1}{n^{\gamma}}+t^{-1/\alpha}n^{-2\theta}+\frac{t^{-\frac{1}{\alpha'}}}{n}:\frac{\log n}{n^\gamma}\lesssim t\lesssim n^{-\theta}\right\}\\
&{}={}&n^{-\gamma}+n^{-2\theta\left(1-\frac{1}{2\alpha}\right)}+n^{-\left(1-\frac{\theta}{\alpha'}\right)}
\end{eqnarray*}
and the result follows by imposing $\theta<\gamma$ to ensure the constraint $\frac{\log n}{n^\gamma}\lesssim t\lesssim n^{-\theta}$ is satisfied.
\subsection{Proof of Proposition~\ref{pro:schatten}}\label{subsec:pro-schatten}
$(i)$ Note that for any $a\ge 1$, $(\id\id^*)^{a}\id=\id(\id^*\id)^{a}$. This follows by observing that for any $f\in\Cal{H}$, \begin{eqnarray}(\id\id^*)^{a}\id f&{}={}&\sum_{i}\lambda^a_i \left(\frac{\id\phi_i}{\sqrt{\lambda_i}}\ol \frac{\id\phi_i}{\sqrt{\lambda_i}}\right)\id f\nonumber\\
&{}={}&\id\sum_i\lambda^{a-1}_i(\phi_i\oh\phi_i)\Sigma f=\id\sum_i\lambda^{a-1}_i\phi_i\langle \phi_i,\Sigma f\rangle_\Cal{H}\nonumber\\
&{}={}&\id\sum_i\lambda^{a}_i\phi_i\langle \phi_i,f\rangle_\Cal{H}=\id\Sigma^{a}f=\id(\id^*\id)^a f.\label{Eq:alpha}
\end{eqnarray}
Therefore, 
\begin{eqnarray*}
\eu{R}_s(\psi_1,\ldots,\psi_\ell)&{}={}&\bb{E}\norm{\left(\id\id^*\right)^{-s/2}\id\left(I-P_\psi\right)\overline{k}(\cdot,X)}^2_{\lp}\\
&{}\stackrel{\eqref{Eq:alpha}}{=}{}&\bb{E}\norm{\id\Sigma^{-s/2}\left(I-P_\psi\right)\overline{k}(\cdot,X)}^2_{\lp}\nonumber\\
&{}\stackrel{(\dagger)}{=}{}&\norm{\Sigma^{(1-s)/2}\left(I-P_\psi\right)\Sigma^{1/2}}^2_{\HSH},\nonumber
\end{eqnarray*}
where $(\dagger)$ follows from Lemma~\ref{rf cov proj rewrite}.\vspace{1mm}\\
\emph{(ii)--(v)} These exactly follow the proof of Proposition~\ref{pro:interpret}\emph{(ii)--(v)} by using $\delta=1-s$ in Lemma~\ref{lem:optima}.
\subsection{Proof of Theorem~\ref{thm:schatten}}\label{subsec:proof-rls}
$(i)$ The result follows from Proposition~\ref{pro:schatten}\emph{(ii)}.\vspace{1mm}\\
$(ii)$ Along the lines of the proof of Theorem~\ref{thm:rff main thm metric 1}\emph{(ii)} and using \eqref{Eq:alpha}, it is easy to show that
\begin{eqnarray*}
R_{\sigh,\ell,s}&{}={}&\left\Vert\Sigma^{(1-s)/2}(I-P_\ell(\sigh))\Sigma^{1/2}\right\Vert^2_{\HSH}+\left\Vert\Sigma^{(1-s)/2}P_\ell(\sigh)(m_\bb{P}-\mph)\right\Vert^2_\Cal{H},\nonumber
\end{eqnarray*}
where the second term can be bounded as $$\Vert \Sigma\Vert^{1-s}_{\OPH}\Vert P_\ell(\sigh)\Vert^2_{\OPH}\Vert m_\bb{P}-\mph\Vert^2_\Cal{H}\lesssim_{\bb{P}^n} \frac{1}{n},$$
through an application of Lemma~\ref{lem:kernel mean bnd}\emph{(i)}. For the first term, employing the strategy used for bounding \circled{\small{1}}, for any $t>0$, we obtain
\begin{eqnarray*}
&{}{}&\norm{\Sigma^{(1-s)/2}(I-P_\ell(\sigh))\Sigma^{1/2}}_{\HSH}^2\nonumber
\\
&{}={}&\left\Vert\Sigma^{(1-s)/2}(\Sigma+tI)^{-1/2}(\Sigma+tI)^{1/2}(I-P_\ell(\sigh))(\Sigma+tI)^{1/2}(\Sigma+tI)^{-1/2}\Sigma^{1/2}\right\Vert_{\HSH}^2\nonumber
\\
&{}\le{}&\norm{(\Sigma+tI)^{-1/2}\Sigma^{1/2}}_{\HSH}^2\norm{\Sigma^{(1-s)/2}(\Sigma+tI)^{-1/2}}_{\OPH}^2\\
&{}{}&\qquad\qquad\qquad\times\norm{(\Sigma+tI)^{1/2}(I-P_\ell(\sigh))(\Sigma+tI)^{1/2}}_{\OPH}^2\nonumber
\\
&{}\le{}&\Cal{N}_\Sigma(t)(\lambdah_{\ell+1}+t)^2\norm{\Sigma^{(1-s)/2}(\Sigma+tI)^{-1/2}}_{\OPH}^2,\nonumber
\end{eqnarray*}
where $$\norm{\Sigma^{(1-s)/2}(\Sigma+tI)^{-1/2}}_{\OPH}^2=\sup_i\frac{\lambda^{1-s}_i}{\lambda_i+t}=\sup_i\left(\frac{\lambda_i}{\lambda_i+t}\right)^{1-s}\frac{1}{(\lambda_i+t)^s}\le\frac{1}{t^s}$$ for $s\le 1$. The result is completed by bounding $(\lambdah_{\ell+1}+t)^2$ as in the proof of the upper bound in Theorem~\ref{thm:rff main thm metric 1}\emph{(ii)}. The lower bound follows from Lemma~\ref{lem:optima} by noting that $R_{\sigh,\ell,s}\ge \Cal{R}^{\Sigma}_{0,1-s,1}(\sigh)\ge \Cal{R}^{\Sigma}_{0,1-s,1}(\Sigma)=R_{\Sigma,\ell,s}$.\vspace{1mm}\\
%
$(iii)$ Note that 
\begin{eqnarray}
R_{\Sigma_m,\ell,s}&{}={}&\bb{E}_{X\sim\Pb}\left\Vert \left(\id\id^*\right)^{-s/2}\id \overline{k}(\cdot,X)-\left(\A\A^*\right)^{-s/2}\A P_\ell(\Sigma_m)\overline{k}_m(\cdot,X)\right\Vert^2_{\lp}\nonumber\\
&{}\le{}& 2\left(\circled{\small{8}}+\small{\circled{9}}\right),\label{Eq:mls}
\end{eqnarray}
where $\circled{\small{8}}:=\bb{E}_{X\sim\Pb}\left\Vert \left(\id\id^*\right)^{-s/2}\id \overline{k}(\cdot,X)-\left(\A\A^*\right)^{-s/2}\A \overline{k}_m(\cdot,X)\right\Vert^2_{\lp}$ and 
\begin{equation}\circled{\small{9}}:=\bb{E}_{X\sim\Pb}\left\Vert\left(\A\A^*\right)^{-s/2}\A (I-P_\ell(\Sigma_m))\overline{k}_m(\cdot,X)\right\Vert^2_{\lp}=\sum_{i>\ell}\lambda^{2-s}_{m,i},\label{Eq:aaa}
\end{equation}
which follows by replicating the analysis in $(i)$ for $\A\A^*$. To bound \circled{\small{8}}, we adapt the proof idea of Lemma~\ref{lem:rf kernel mean bnd}. Similar to \eqref{Eq:eq-diff}, it can be shown that
\begin{eqnarray}
\circled{\small{8}}
&{}={}& \Vert \left(\id\id^*\right)^{(2-s)/2}-\left(\tS\tS^*\right)^{(2-s)/2}\Vert^2_{\HSL}+2\text{tr}\left(\left(\id\id^*\right)^{-s/2}\id\id^*\tS\tS^*\left(\tS\tS^*\right)^{-s/2}\right)\nonumber\\
&{}{}&\qquad\qquad-2\bb{E}\langle \left(\id\id^*\right)^{-s/2}\id\overline{k}, \left(\tS\tS^*\right)^{-s/2}\tS\overline{k}_m(\cdot,X)\rangle_{\lp}.\label{Eq:neww}
\end{eqnarray}
Similar to \eqref{Eq:expe} and \eqref{Eq:tr}, we obtain
\begin{eqnarray*}
\bb{E}\langle \left(\id\id^*\right)^{-s/2}\id\overline{k}(\cdot,X), \left(\tS\tS^*\right)^{-s/2}\tS\overline{k}_m(\cdot,X)\rangle_{\lp}=\int_\Theta\sum^m_{i=1}B_i(\theta)\left\langle \varphi(\cdot,\theta),\varphi_i\right\rangle_{\lp}\,d\Lambda(\theta)\nonumber
\end{eqnarray*}
and
\begin{eqnarray*}
\text{tr}\left(\left(\id\id^*\right)^{-s/2}\id\id^*\tS\tS^*\left(\tS\tS^*\right)^{-s/2}\right)
=\int_\Theta\sum^m_{i=1}B_i(\theta)\left[\left\langle \varphi(\cdot,\theta),\varphi_i\right\rangle_{\lp}-\varphi_\bb{P}(\theta)\varphi_{i,\bb{P}}\right]\,d\Lambda(\theta),\nonumber
\end{eqnarray*}
where $B_i(\theta):=\left\langle \left(\id\id^*\right)^{-s/2}\left(\varphi(\cdot,\theta)-\varphi_{\bb{P}}(\theta)\right),\left(\tS\tS^*\right)^{-s/2}\left(\varphi_i-\varphi_{i,\bb{P}}\right)\right\rangle_{\lp}$.
This implies
\begin{eqnarray*}
&{}{}&
 \text{tr}\left(\left(\id\id^*\right)^{-s/2}\id\id^*\tS\tS^*\left(\tS\tS^*\right)^{-s/2}\right)-\bb{E}\langle \left(\id\id^*\right)^{-s/2}\id\overline{k}(\cdot,X), \left(\tS\tS^*\right)^{-s/2}\tS\overline{k}_m(\cdot,X)\rangle_{\lp}\nonumber\\
 &{}={}&-\left\langle \left(\id\id^*\right)^{-s/2}\int_\Theta A(\theta)\,d\Lambda(\theta),\left(\tS\tS^*\right)^{-s/2}\frac{1}{m}\sum^m_{i=1}A(\theta_i) \right\rangle_{\lp}\nonumber\\
 &{}\stackrel{(*)}{\le}{}&\frac{1}{4}\norm{\left(\id\id^*\right)^{-s/2}\int_\Theta A(\theta)\,d\Lambda(\theta)-\left(\tS\tS^*\right)^{-s/2}\frac{1}{m}\sum^m_{i=1}A(\theta_i)}^2_{\lp}\nonumber\\
 &{}\le{}&\frac{1}{2}\norm{\left(\id\id^*\right)^{-s/2}-\left(\tS\tS^*\right)^{-s/2}}^2_{\Cal{L}^\infty(\lp)}\norm{\int_\Theta A(\theta)\,d\Lambda(\theta)}^2_{\lp}\\
 &{}{}&\qquad+\left[\norm{\left(\id\id^*\right)^{-s/2}-\left(\tS\tS^*\right)^{-s/2}}^2_{\Cal{L}^\infty(\lp)}+\norm{\left(\id\id^*\right)^{-s/2}}^2_{\Cal{L}^\infty(\lp)}\right]\\
 &{}{}&\qquad\qquad \times\norm{\int_\Theta A(\theta)\,d\Lambda(\theta)-\frac{1}{m}\sum^m_{i=1}A(\theta_i)}^2_{\lp}\nonumber
\end{eqnarray*}
where $(*)$ follows from the parallelogram identity with $A(\theta):=\varphi(\cdot,\theta)\varphi_\bb{P}(\theta)-\varphi^2_\bb{P}(\theta)$. It therefore follows from Theorem~\ref{thm:bernstein} that 
\begin{eqnarray}
\qquad\quad&{}{}&
 \text{tr}\left(\left(\id\id^*\right)^{-s/2}\id\id^*\tS\tS^*\left(\tS\tS^*\right)^{-s/2}\right)-\bb{E}\langle \left(\id\id^*\right)^{-s/2}\id\overline{k}(\cdot,X), \left(\tS\tS^*\right)^{-s/2}\tS\overline{k}_m(\cdot,X)\rangle_{\lp}\nonumber\\
\qquad\quad&{}\lesssim_{\Lambda^m}{}& \norm{\left(\id\id^*\right)^{-s/2}-\left(\tS\tS^*\right)^{-s/2}}^2_{\Cal{L}^\infty(\lp)}+\frac{1}{m}\left\Vert (\id\id^*)^{-s/2}\right\Vert^2_{\Cal{L}^\infty(\lp)}.\label{Eq:mm}
\end{eqnarray}
Combining \eqref{Eq:mls}--\eqref{Eq:mm}, we obtain
\begin{eqnarray}
R_{\Sigma_m,\ell,s}&{}\lesssim_{\Lambda^m}{}& \sum_{i>\ell}\lambda^{2-s}_{m,i}+\norm{\left(\id\id^*\right)^{(2-s)/2}-\left(\tS\tS^*\right)^{(2-s)/2}}^2_{\HSL}+\frac{1}{m}\left\Vert (\id\id^*)^{-s/2}\right\Vert^2_{\Cal{L}^\infty(\lp)}
\nonumber\\
&{}{}&\qquad+ \norm{\left(\id\id^*\right)^{-s/2}-\left(\tS\tS^*\right)^{-s/2}}^2_{\Cal{L}^\infty(\lp)},\label{Eq:rls-2}
\end{eqnarray}
where for $s\le 1$, 
$$\sum_{i>\ell}\lambda^{2-s}_{m,i}\lesssim \sum_{i>\ell}|\lambda_{m,i}-\lambda_i|^{2-s}+\sum_{i>\ell}\lambda^{2-s}_{i}\stackrel{(*)}{\lesssim} \left\Vert \id\id^*-\A\A^*\right\Vert^{2-s}_{2-s}+\sum_{i>\ell}\lambda^{2-s}_{i},$$ with $(*)$ following from \citep[Theorem II]{Kato-87}. Here $\Vert \cdot\Vert_{2-s}$ denotes the $(2-s)$-Schatten norm. Since $t\mapsto t^\alpha$ is Lipschitz on a bounded subset of $(0,\infty)$ for $\alpha\ge 1$, it follows from \citep[Lemma 7]{devito-14} that the second term of \eqref{Eq:rls-2} is bounded (up to constants) by $\Vert \id\id^*-\A\A^*\Vert^2_{\HSL}$ for $s\le 0$. Using the fact that $t\mapsto t^\alpha$ is operator monotone on $(0,\infty)$ for $0\le\alpha\le 1$, the third term in \eqref{Eq:rls-2} is bounded by $\Vert \id\id^*-\A\A^*\Vert^{-s}_{\Cal{L}^\infty(\lp)}$ for $-2\le s\le 0$ (follows from \citealp[Theorem X.1.1]{Bhatia-97}). Therefore for $-2\le s\le 0$, \eqref{Eq:rls-2} reduces to
\begin{eqnarray*}
R_{\Sigma_m,\ell,s}&{}\lesssim_{\Lambda^m}{}&\sum_{i>\ell}\lambda^{2-s}_{i}+\left\Vert \id\id^*-\A\A^*\right\Vert^{2-s}_{2-s}+\Vert \id\id^*-\A\A^*\Vert^2_{\HSL}\\
&{}{}&\qquad\qquad+\Vert \id\id^*-\A\A^*\Vert^{-s}_{\Cal{L}^\infty(\lp)}+\frac{1}{m}.\nonumber
\end{eqnarray*}
The result follows by applying Lemma~\ref{lem:perturb} to $\Vert \id\id^*-\A\A^*\Vert_{\HSL}$ and noting that $\left\Vert \id\id^*-\A\A^*\right\Vert_{2-s}\le \left\Vert \id\id^*-\A\A^*\right\Vert_{\HSL}$ since $-2\le s\le 0$.

The lower bound follows the idea in the proof of lower bound on $R_{\Sigma_m,\ell}$ by noticing that $R_{\Sigma_m,\ell,s}\ge \left(\sqrt{\circled{\small{9}}}-\sqrt{\circled{\small{8}}}\right)^2$ where $\circled{\small{8}}\lesssim m^{s/2}$ for $s\in[-2,0)$ and $\circled{\small{8}}\lesssim \frac{1}{m}$ for $s=0$. Considering $\circled{\small{9}}=\sum_{i>\ell}\lambda^{2-s}_{m,i}$, we have
\begin{eqnarray}
\left(\sum_{i>\ell}\lambda^{2-s}_{m,i}\right)^{\frac{1}{2-s}}&{}\ge{}&\left|\left(\sum_{i>\ell}\lambda^{2-s}_{i}\right)^{\frac{1}{2-s}}-\left(\sum_{i>\ell}\left|\lambda_{m,i}-\lambda_i\right|^{2-s}\right)^{\frac{1}{2-s}}\right|\nonumber\\
&{}\ge{}&\left(\sum_{i>\ell}\lambda^{2-s}_{i}\right)^{\frac{1}{2-s}}-\norm{\id\id^*-\tS\tS^*}_{2-s}\nonumber\\
&{}\gtrsim_{\Lambda^m}{}& \left(\sum_{i>\ell}\lambda^{2-s}_{i}\right)^{\frac{1}{2-s}}-\frac{1}{\sqrt{m}}\gtrsim \left(\sum_{i>\ell}\lambda^{2-s}_{i}\right)^{\frac{1}{2-s}},\nonumber
\end{eqnarray}
since $m\gtrsim \left(\sum_{i>\ell}\lambda^{2-s}_{i}\right)^{\frac{2}{s-2}}$. Since $m\gtrsim \left(\sum_{i>\ell}\lambda^{2-s}_{i}\right)^{\frac{2}{s}}$ for $s\in[-2,0)$, it also follows that $\sqrt{\circled{\small{9}}}-\sqrt{\circled{\small{8}}}\gtrsim_{\Lambda^m} \left(\sum_{i>\ell}\lambda^{2-s}_{i}\right)^{\frac{1}{2}}$ and the result follows.\vspace{1mm}\\
$(iv)$ We skip the proof as it follows the ideas in the proof of $R_{\sigh,\ell}$ combined with the bounds on $\circled{\small{8}}$ and \circled{\small{9}}.

\section*{Acknowledgments}
BKS is supported by National Science Foundation (NSF) award DMS-1713011 and CAREER award DMS-1945396. 

\bibliographystyle{plainnat}
\bibliography{Reference}
\begin{appendices}\label{appendix}
\numberwithin{equation}{section}
\section{Additional Results}\label{sec:additional}
In this section, we present corollaries of Theorems~\ref{thm:rff main thm metric 1} and \ref{thm:rff main thm metric 2} assuming exponential decay rates for the eigenvalues of $\Sigma$.

\begin{appxcor}[Exponential decay of eigenvalues]\label{rff exp decay corollary}
Suppose $\underbar{B}e^{-\tau i}\le\lambda_i\le\bar{B}e^{-\tau i}$ for $\tau>0$ and $\underbar{B},\bar{B}\in(0,\infty)$. Let $\ell=\frac{1}{\tau}\log n^\theta$ for $\theta>0$.
Then 
\\\\
(i) $$n^{-2\theta}\lesssim R_{\Sigma,\ell}\lesssim n^{-2\theta}.$$
There exists $\tilde{n}\in\bb{N}$ such that for all $n>\tilde{n}$, the following hold:
\vspace{2mm}\\
(ii) \[ n^{-2\theta}\lesssim R_{\hat{\Sigma},\ell}\lesssim_{\Pb^n}\begin{cases} 
      n^{-2\theta}\log n ,\quad\theta\le\frac{1}{2}\\
      \frac{1}{n},\qquad\qquad\,\,\,\theta>\frac{1}{2}
   \end{cases};
\]
(iii) For $0<\gamma\le1$ and $m=n^\gamma$, 
\[ n^{-2\theta}\mathds{1}_{\{\gamma\ge 2\theta\}}\lesssim_{\Lambda^m} R_{\Sigma_m,\ell}\lesssim_{\Lambda^m}\begin{cases} 
      n^{-2\theta},\qquad\gamma\ge2\theta,\,\theta\le\frac{1}{2}\hspace{2.5mm} \\
      n^{-\gamma},\qquad\,\,\gamma\le 1\wedge 2\theta
   \end{cases};
\]
(iv) For $0<\gamma\le1$ and $m=n^\gamma$, 
\[ n^{-2\theta}\mathds{1}_{\{\gamma\ge 2\theta\}}\lesssim_{\Lambda^m} R_{\hat{\Sigma}_m,\ell}\lesssim_{\Pb^n\times\Lambda^m}\begin{cases} 
      n^{-2\theta}\log n,\quad\gamma\ge 2\theta,\,\theta\le\frac{1}{2} \\
      n^{-\gamma},\qquad\quad\,\,\,\gamma<2\theta,\,\gamma\le 1
   \end{cases}.
\]
\end{appxcor}
We may draw conclusions similar to Remark~\ref{rff:rem1} from Corollary~\ref{rff exp decay corollary}. The behavior of $R_{\Sigma,\ell}$ matches that of $R_{\sigh,\ell}$, up to a $\log n$ factor, if $\ell$ grows slower than $\log\sqrt{n}$. If $m\ge n^{2\theta}$ with $\theta\le\frac{1}{2}$, then RF-EKPCA and EKPCA have similar statistical convergence behavior (i.e., no statistical loss) but with RF-EKPCA enjoying a computational edge if $m<\sqrt{n\log n^\theta}$, i.e., $\theta\le\frac{1}{4}$.
\begin{proof}
\emph{(i)}  From Theorem \ref{thm:rff main thm metric 1}\emph{(i)} we have
$$R_{\Sigma,\ell}=\sum_{i>\ell}\lambda_i^2\lesssim\sum_{i>\ell} e^{-2\tau i}\lesssim\int_\ell^\infty e^{-2\tau x}dx\lesssim e^{-2\tau\ell}=n^{-2\theta}$$
and 
$$R_{\Sigma,\ell}=\sum_{i>\ell}\lambda_i^2\gtrsim\sum_{i>\ell} e^{-2\tau i}\gtrsim\int_{\ell+1}^\infty e^{-2\tau x}dx\gtrsim e^{-2\tau(\ell+1)}=e^{-2\tau}n^{-2\theta}.$$
\\
\emph{(ii)} Using $\Cal{N}_\Sigma(t)\lesssim\log\frac{1}{t}$ from Lemma~\ref{N(t)}\emph{(ii)} in Theorem \ref{thm:rff main thm metric 1}\emph{(ii)}, we have
\begin{eqnarray*}
R_{\hat{\Sigma},\ell}&{}\lesssim_{\Pb^n}{}&\inf\left\{(n^{-\theta}+t)^2\log\frac{1}{t}+n^{-1}:\frac{\log n}{n}\lesssim t\le\frac{\lambda_1}{3}\right\}\nonumber\\
&{}\lesssim{}&
\begin{cases}n^{-2\theta}\log n+\frac{1}{n},\qquad \theta<1\\\frac{\log^3 n}{n^2}+\frac{1}{n},\qquad\qquad \theta\ge1\end{cases},
\end{eqnarray*}
and the result follows.
\vspace{1mm}\\
\emph{(iii)} We obtain $R_{\Sigma_m,\ell}\lesssim_{\Lambda_m}\frac{1}{m}+n^{-2\theta}$ and the result follows.
\vspace{1mm}\\
\emph{(iv)} Theorem \ref{thm:rff main thm metric 1} \emph{(iv)} yields
$$R_{\sigh_m,\ell}\lesssim_{\Pb^n\times\Lambda^m}n^{-\gamma}+\left(\Cal{N}_\Sigma(t)+\sqrt{\frac{\Cal{N}_\Sigma(t)}{tn^\gamma}}+\frac{1}{tn^\gamma}\right)(\lambda_\ell+t)^2,$$
for $\frac{\log n}{n^\gamma}\lesssim t\lesssim\frac{\lambda_1}{3}$. Using $\Cal{N}_\Sigma(t)\lesssim\log\frac{1}{t}$, it is clear that $\frac{1}{tn^\gamma}\lesssim \log\frac{1}{t}$ and $\sqrt{\Cal{N}_\Sigma(t)n^{-\gamma}/t}\lesssim\log\frac{1}{t}$ which follows from the constraint on $t$. Therefore,
\begin{eqnarray*}
R_{\sigh_m,\ell}&{}\lesssim_{\Pb^n\times\Lambda^m}{}&\inf\left\{(n^{-\theta}+t)^2\log\frac{1}{t}+n^{-\gamma}:\frac{\log n}{n^\gamma}\lesssim t\le\frac{\lambda_1}{3}\right\}\nonumber\\
&{}\lesssim{}&
\begin{cases}n^{-2\theta}\log n+n^{-\gamma},\quad\theta<\gamma\\\frac{\log^3 n}{n^2}+n^{-\gamma},\quad\qquad \theta\ge\gamma\end{cases},
\end{eqnarray*}
and the result follows. 
\end{proof}
\begin{appxcor}[Exponential decay of eigenvalues]\label{rff exp decay corollary metric 2}
Suppose $\underbar{B}e^{-\tau i}\le\lambda_i\le\bar{B}e^{-\tau i}$ for $\tau>0$ and $\underbar{B},\bar{B}\in(0,\infty)$. Let $\ell=\frac{1}{\tau}\log n^\theta$ for $\theta>0$.
Then 
\\\\
(i) $$n^{-2\theta}\lesssim S_{\Sigma,\ell}\lesssim n^{-2\theta}.$$
There exists $\tilde{n}\in\bb{N}$ such that for all $n>\tilde{n}$, the following hold:
\vspace{2mm}\\
(ii) \[ n^{-2\theta}\lesssim S_{\widehat{\Sigma},\ell}\lesssim_{\Pb^n}\begin{cases} 
    n^{-2\theta}\log n ,\qquad\qquad\,\,\,\theta\le\frac{2}{5}\\
      n^{-\left(1-\frac{\theta}{2}\right)}\log n,\qquad\,\,\,\frac{2}{5}\le\theta<1
   \end{cases};
\]
(iii) For $0<\gamma\le1$ and $m=n^\gamma$,  \[ n^{-2\theta}\mathds{1}_{\left\{\gamma\ge 2\theta\right\}}\lesssim_{\Lambda^m} S_{\Sigma_m,\ell}\lesssim_{\Lambda^m}\begin{cases} 
      n^{-2\theta},\qquad\gamma\ge2\theta,\,\theta\le\frac{1}{2}\hspace{2.5mm} \\
      n^{-\gamma},\qquad\,\,\gamma\le 1\wedge 2\theta
   \end{cases};
\]
(iv) For $0<\gamma\le1$ and $m=n^\gamma$, \begin{eqnarray*} n^{-2\theta}\mathds{1}_{\left\{\gamma\ge 2\theta\right\}}&{}\lesssim_{\Lambda^m}{}& S_{\widehat{\Sigma}_m,\ell}\lesssim_{\Pb^n\times\Lambda^m}\\&{}{}&\begin{cases} 
n^{-2\theta}\log n ,\qquad\qquad\,\,\,\gamma\ge 2\theta,\,\theta\le\frac{2}{5}\\
      n^{-\left(1-\frac{\theta}{2}\right)}\log n,\quad\gamma\ge 1-\frac{\theta}{2},\,\gamma>\theta,\,\frac{2}{5}\le\theta<1\\
      n^{-\gamma},\qquad\qquad\qquad\theta<\gamma\le\left[1\wedge 2\theta\wedge \left(1-\frac{\theta}{2}\right)\right]
   \end{cases}.
\end{eqnarray*}
\end{appxcor}
\begin{proof}
\emph{(i)} and \emph{(iii)} are exactly same as that of the proof of Corollary~\ref{rff exp decay corollary}.\vspace{1mm}\\
\emph{(ii)} Using $\Cal{N}_\Sigma(t)\lesssim\log\frac{1}{t}$ from Lemma~\ref{N(t)}\emph{(ii)} in Theorem \ref{thm:rff main thm metric 2}\emph{(ii)}, we have $\frac{\Cal{N}_\Sigma(t)}{n\sqrt{t}}\wedge \frac{1}{nt}\lesssim \frac{\Cal{N}_\Sigma(t)}{n\sqrt{t}}\lesssim \frac{\log\frac{1}{t}}{n\sqrt{t}}$, which implies 
\begin{eqnarray*}
S_{\hat{\Sigma},\ell}&{}\lesssim_{\Pb^n}{}&\inf\left\{n^{-2\theta}\log\frac{1}{t}+\frac{\log\frac{1}{t}}{n\sqrt{t}}:\frac{\log n}{n}\lesssim t\lesssim n^{-\theta}\right\}\\
&{}\lesssim{}& \left(n^{-2\theta}+n^{-\left(1-\frac{\theta}{2}\right)}\right)\log n
\end{eqnarray*}
and the result follows.
\vspace{1mm}\\
\emph{(iv)} By noting that $\Cal{A}(t)\lesssim \Cal{N}_\Sigma(t)\lesssim \log\frac{1}{t}$, we obtain
\begin{eqnarray*}
S_{\hat{\Sigma},\ell}&{}\lesssim_{\Pb^n}{}&\inf\left\{n^{-\gamma}+n^{-2\theta}\log\frac{1}{t}+\frac{\log\frac{1}{t}}{n\sqrt{t}}:\frac{\log n}{n}\lesssim t\lesssim n^{-\theta}\right\}\\
&{}\lesssim{}& n^{-\gamma}+\left(n^{-2\theta}+n^{-\left(1-\frac{\theta}{2}\right)}\right)\log n,
\end{eqnarray*}
which yields the result. 
\end{proof}

\section{Technical Results}
In this section, we collect important technical results used to prove the main results of this paper.

\begin{appxlem}\label{lem:optima}
 Let $A:H\rightarrow H$ be a positive self-adjoint Hilbert-Schmidt operator on a separable Hilbert space $H$ with $(\lambda_i,\psi_i)_i$ being its eigenvalues and eigenfunctions. Suppose the eigenvalues are simple and satisfy $\lambda_1>\lambda_2>\cdots$. 
 
 Define  
 $$\Cal{Q}_\ell=\left\{\sum^\ell_{i=1}\tau_i\otimes_H\tau_i\,:(\tau_i)_{i\in[\ell]}\subset H\right\},$$
 \begin{equation}\mathcal{R}^A_{\alpha,\delta,\theta}(Q)=\left\Vert A^{\delta/2}\left(I-QA^{\alpha}\right)A^{\theta/2}\right\Vert^2_{\Cal{L}^2(H)},\,\,Q\in\Cal{Q}_\ell,\label{Eq:def}\end{equation}
 and 
 \begin{equation}\mathcal{S}^A_{\rho}(Q)=\left\Vert \left(I-Q\right)A^{\rho/2}\right\Vert^2_{\Cal{L}^2(H)},\,\,Q\in\Cal{Q}_\ell,\nonumber
 \end{equation}
 where 
 $\alpha,\delta\ge 0$ and $\theta,\rho>0$. 
 Then,
 \begin{itemize}
\item[(i)] $$A^{-\alpha}_\ell=\arg\inf_{Q\in\Cal{Q}_\ell}\Cal{R}^A_{\alpha,\delta,\theta}(Q),$$
 where $A_\ell=\sum^\ell_{i=1}\lambda_i\psi_i\otimes_H \psi_i$ and $\Cal{R}^A_{\alpha,\delta,\theta}(A^{-\alpha}_\ell)=\sum_{i>\ell}\lambda^{\theta+\delta}_i$;
\item[(ii)] $$Q_\ell=\arg\inf_{Q\in\Cal{Q}_\ell}\Cal{S}^A_{\rho}(Q),$$
 where $Q_\ell=\sum^\ell_{i=1}\psi_i\otimes_H \psi_i$ and $\Cal{S}^A_{\rho}(Q_\ell)=\sum_{i>\ell}\lambda^{\rho}_i$.
 \end{itemize}
\end{appxlem}
\begin{proof}
\emph{(i)} Define $$\Al=\sum^\ell_{i=1}\lambda_i \psi_i\otimes_H \psi_i\,\,\text{and}\,\,\Ag=\sum_{i>\ell}\lambda_i \psi_i\otimes_H \psi_i$$ so that $A=\Al+\Ag$. Also since $A^{\delta/2}=\sum_i\lambda^{\delta/2}_i\psi_i\otimes_H\psi_i$, we have $A^{\delta/2}=\Al^{\delta/2}+\Ag^{\delta/2}$. Consider
\begin{eqnarray}
&&\Cal{R}^A_{\alpha,\delta,\theta}(Q)=\norm{A^{(\theta+\delta)/2}-A^{\delta/2}QA^{\alpha+\theta/2}}^2_{\hsh}\nonumber\\
&{}={}&\norm{\Al^{(\theta+\delta)/2}+\Ag^{(\theta+\delta)/2}-\left(\Al^{\delta/2}+\Ag^{\delta/2}\right)Q\left(\Al^{\alpha+\theta/2}+\Ag^{\alpha+\theta/2}\right)}^2_{\hsh}\nonumber\\
&{}\stackrel{(*)}{=}{}&\left\Vert\left(\Al^{(\theta+\delta)/2}-\Al^{\delta/2}Q\Al^{\alpha+\theta/2}\right)+\left(\Ag^{(\theta+\delta)/2}-\Ag^{\delta/2}Q\Ag^{\alpha+\theta/2}\right)\right.\nonumber\\
&{}{}&\qquad\qquad\left.-\left(\Ag^{\delta/2} Q\Al^{\alpha+\theta/2}+\Al^{\delta/2} Q\Ag^{\alpha+\theta/2}\right)\right\Vert^2_{\hsh}\nonumber\\
&{}={}&\norm{\Al^{(\theta+\delta)/2}-\Al^{\delta/2}Q\Al^{\alpha+\theta/2}}^2_{\hsh}+\norm{\Ag^{(\theta+\delta)/2}-\Ag^{\delta/2}Q\Ag^{\alpha+\theta/2}}^2_{\hsh}\nonumber\\
&{}{}&\qquad\qquad+\norm{\Ag^{\delta/2} Q\Al^{\alpha+\theta/2}+\Al^{\delta/2} Q\Ag^{\alpha+\theta/2}}^2_{\hsh},\label{Eq:simple1}
\end{eqnarray}
where \eqref{Eq:simple1} is obtained by expanding the square in $(*)$ and noting that all the inner products are zero since $\text{Tr}(\Al^\beta B \Ag^\gamma)=0$ for any $\beta,\gamma\ge 0$ and any operator $B:H\rightarrow H$. Again expanding the square in the last term of \eqref{Eq:simple1} and noting that the inner product is zero, we obtain
\begin{eqnarray}
\Cal{R}^A_{\alpha,\delta,\theta}(Q)
&{}={}& \norm{\Al^{(\theta+\delta)/2}-\Al^{\delta/2}Q\Al^{\alpha+\theta/2}}^2_{\hsh}+\norm{\Ag^{(\theta+\delta)/2}-\Ag^{\delta/2}Q\Ag^{\alpha+\theta/2}}^2_{\hsh}\nonumber\\
&{}{}&\qquad\qquad +\norm{\Ag^{\delta/2} Q\Al^{\alpha+\theta/2}}^2_{\hsh}+\norm{\Al^{\delta/2} Q\Ag^{\alpha+\theta/2}}^2_{\hsh}.\label{Eq:simple2}
\end{eqnarray}
We now decompose $Q\in\Cal{Q}_\ell$ as $Q=Q_1+Q_2+Q_3$ where $Q_j=\sum_{i\in \Cal{A}_j}\tau_i\otimes_H\tau_i$ with $\Cal{A}_j\subset\{1,\ldots,\ell\}$, $j\in[3]$, $\cup_j\Cal{A}_j=\{1,\ldots,\ell\}$, $\Cal{A}_i\cap\Cal{A}_j=\emptyset$ for all $i,j\in[3]$ such that
$(\tau_i)_{i\in \Cal{A}_1}\subset \text{Ran}(\Al)$, $(\tau_i)_{i\in \Cal{A}_2}\subset \text{Ran}(\Ag)$, and $(\tau_i)_{i\in \Cal{A}_3}\subset \text{Ker}(A)$. This means optimizing over $Q\in\Cal{Q}_\ell$ is equivalent to optimizing over $Q_j\in\Cal{Q}_{|\Cal{A}_j|}$ and $\Cal{A}_j$, $j\in[3]$. 

Using $$\Ag^{\beta} Q\Al^{\gamma}=\Ag^{\beta}(Q_1+Q_2+Q_3)\Al^{\gamma}=0$$ and 
$$\Al^{\beta} Q\Ag^{\gamma}=\Al^{\beta}(Q_1+Q_2+Q_3)\Ag^{\gamma}=0$$
for any $\beta,\gamma\ge 0$ in \eqref{Eq:simple2}, we obtain
\begin{eqnarray}
\Cal{R}^A_{\alpha,\delta,\theta}(Q)
&{}={}&\norm{\Al^{(\theta+\delta)/2}-\Al^{\delta/2}Q\Al^{\alpha+\theta/2}}^2_{\hsh}+\norm{\Ag^{(\theta+\delta)/2}-\Ag^{\delta/2}Q\Ag^{\alpha+\theta/2}}^2_{\hsh}\nonumber\\
&{}={}&\norm{\Al^{\delta/2}\left(\Pl-Q\Al^\alpha\right)\Al^{\theta/2}}^2_{\hsh}+\norm{\Ag^{\delta/2}\left(\Pg-Q\Ag^\alpha\right)\Ag^{\theta/2}}^2_{\hsh}\nonumber\\
&{}={}&\norm{\Al^{\delta/2}\left(\Pl-Q_1\Al^\alpha\right)\Al^{\theta/2}}^2_{\hsh}+\norm{\Ag^{\delta/2}\left(\Pg-Q_2\Ag^\alpha\right)\Ag^{\theta/2}}^2_{\hsh},\label{Eq:simple3}
\end{eqnarray}
where $$\Pl:=\sum^\ell_{i=1}\psi_i\otimes_H\psi_i\,\,\text{and}\,\,\Pg:=\sum_{i>\ell}\psi_i\otimes_H\psi_i.$$
Let $\Cal{B}\subseteq\{1,\ldots,\ell\}$, $\Cal{B}^c:=\{1,\ldots,\ell\}\backslash\Cal{B}$, $\Cal{C}\subseteq\{\ell+1,\ell+2,\ldots\}$, $\Cal{C}^c:=\{\ell+1,\ell+2,\ldots\}\backslash\Cal{C}$ such that 
$\text{span}\{(\psi_i)_{i\in\Cal{B}}\}=\text{span}((\tau_{i})_{i\in\Cal{A}_1})$ and $\text{span}\{(\psi_i)_{i\in\Cal{C}}\}=\text{span}((\tau_{i})_{i\in\Cal{A}_2})$. This means $|\Cal{B}|\le|\Cal{A}_1|$ and $|\Cal{C}|\le |\Cal{A}_2|$ and $|\Cal{B}|+|\Cal{C}|\le \ell$. Note that $A_{\le}=A_{\le,\Cal{B}}+A_{\le,\Cal{B}^c}$ and $A_{>}=A_{>.\Cal{C}}+A_{>,\Cal{C}^c}$, where
\begin{eqnarray*}
A_{\le,\bullet}:=\sum_{i\in\bullet}\lambda_i\psi_i\otimes_H\psi_i,\quad P_{\le,\bullet}=\sum_{i\in\bullet}\psi_i\otimes_H \psi_i,\,\,\,\bullet\in\{\Cal{B},\Cal{B}^c\},\nonumber\\
A_{>,\blacksquare}
:=\sum_{i\in\blacksquare}\lambda_i\psi_i\otimes_H\psi_i,\quad\text{and}\quad P_{>,\blacksquare}=\sum_{i\in\blacksquare}\psi_i\otimes_H \psi_i,\,\,\, \blacksquare\in\{\Cal{C},\Cal{C}^c\}.\nonumber
\end{eqnarray*}
Consider the first term in the r.h.s.~of \eqref{Eq:simple3}, i.e.,
\begin{eqnarray}
&&\norm{\Al^{\delta/2}\left(\Pl-Q_1\Al^\alpha\right)\Al^{\theta/2}}^2_{\hsh} \nonumber\\
&{}={}&\norm{\left(\AlB^{\delta/2}+\AlBc^{\delta/2}\right)\left(\PlB+\PlBc-Q_1\Al^\alpha\right)\Al^{\theta/2}}^2_{\hsh}\nonumber\\
&{}\stackrel{(\dagger)}{=}{}&\norm{\AlB^{\delta/2}\left(\PlB-Q_1\Al^\alpha\right)\Al^{\theta/2}}^2_{\hsh}+\norm{\AlBc^{\delta/2}\left(\PlBc-Q_1\Al^\alpha\right)\Al^{\theta/2}}^2_{\hsh}\nonumber\\
&{}\stackrel{(*)}{=}{}&\norm{\AlB^{\delta/2}\left(\PlB-Q_1\AlB^\alpha-Q_1\AlBc^\alpha\right)\left(\AlB^{\theta/2}+\AlBc^{\theta/2}\right)}^2_{\hsh}\nonumber\\
&{}{}&\qquad +\norm{\AlBc^{\delta/2}\left(\PlBc-Q_1\AlB^\alpha-Q_1\AlBc^\alpha\right)\left(\AlB^{\theta/2}+\AlBc^{\theta/2}\right)}^2_{\hsh}\nonumber
\end{eqnarray}
\begin{eqnarray}
&{}\stackrel{(\ddagger)}{=}{}&\norm{\AlB^{\delta/2}\left(\PlB-Q_1\AlB^\alpha\right)\AlB^{\theta/2}-\AlB^{\delta/2}Q_1\AlBc^{\alpha+\theta/2}}^2_{\hsh}\nonumber\\
&{}{}&\qquad +\norm{\AlBc^{\delta/2}\left(\PlBc-Q_1\AlBc^\alpha\right)\AlBc^{\theta/2}-\AlBc^{\delta/2}Q_1\AlB^{\alpha+\theta/2}}^2_{\hsh},\label{Eq:reduce1}
\end{eqnarray}
where $(\dagger)$ is obtained by expanding the square in the previous line and using \begin{equation}\AlB^{\beta}\AlBc^{\gamma}=\AlBc^{\gamma}\AlB^{\beta}=0\,\,\text{for any}\,\, \beta,\gamma\ge 0.\label{Eq:equa}\end{equation} Again using \eqref{Eq:equa}, $(*)$ reduces to $(\ddagger)$. By noting that
$\AlB^{\delta/2}Q_1\AlBc^{\alpha+\theta/2}=\AlBc^{\delta/2}Q_1\AlB^{\alpha+\theta/2}=\AlBc^{\delta/2}Q_1\AlBc^{\alpha+\theta/2}=0$, \eqref{Eq:reduce1} reduces to
\begin{eqnarray}
\norm{\Al^{\delta/2}\left(\Pl-Q_1\Al^\alpha\right)\Al^{\theta/2}}^2_{\hsh}&{}={}&\norm{\AlB^{\delta/2}\left(\PlB-Q_1\AlB^\alpha\right)\AlB^{\theta/2}}^2_{\hsh}\nonumber\\
&{}{}&\qquad+\norm{\AlBc^{(\delta+\theta)/2}}^2_{\hsh}.\label{Eq:part1}
\end{eqnarray}
Carrying out similar calculation for the second term of \eqref{Eq:simple3}, and combining the result along with \eqref{Eq:part1} in \eqref{Eq:simple3} yields
\begin{eqnarray}
\Cal{R}^A_{\alpha,\delta,\theta}(Q)
&{}={}& \norm{\AlB^{\delta/2}\left(\PlB-Q_1\AlB^\alpha\right)\AlB^{\theta/2}}^2_{\hsh}+\norm{\AlBc^{(\delta+\theta)/2}}^2_{\hsh}\nonumber\\
&{}{}&\qquad+\norm{\AgC^{\delta/2}\left(\PgC-Q_2\AgC^\alpha\right)\AgC^{\theta/2}}^2_{\hsh}+\norm{\AgCc^{(\delta+\theta)/2}}^2_{\hsh}\nonumber\\
&{}={}&\norm{\AlB^{\delta/2}\left(\PlB-Q_1\AlB^\alpha\right)\AlB^{\theta/2}}^2_{\hsh}\nonumber\\
&{}{}&\qquad+\norm{\AgC^{\delta/2}\left(\PgC-Q_2\AgC^\alpha\right)\AgC^{\theta/2}}^2_{\hsh}+\sum^\infty_{i=1}\lambda^{\delta+\theta}_i-\sum_{i\in\Cal{B}\cup\Cal{C}}\lambda^{\delta+\theta}_i,\label{Eq:finalexp}
\end{eqnarray}
where we used $$\norm{\AlBc^{(\delta+\theta)/2}}^2_{\hsh}+\norm{\AgCc^{(\delta+\theta)/2}}^2_{\hsh}=\sum_{i\in\Cal{B}^c\cup\Cal{C}^c}\lambda^{\delta+\theta}_i=\sum^\infty_{i=1}\lambda^{\delta+\theta}_i-\sum_{i\in\Cal{B}\cup\Cal{C}}\lambda^{\delta+\theta}_i.$$ It follows from \eqref{Eq:finalexp} that $\Cal{R}^A_{\alpha,\delta,\theta}$ is minimized only when $Q_1$ and $Q_2$ satisfy $\PlB=Q_1\AlB^\alpha$ (i.e., $Q_1=\AlB^{-\alpha}$), $\PgC=Q_2\AgC^\alpha$ (i.e., $Q_2=\AgC^{-\alpha}$) for $\Cal{B}$ and $\Cal{C}$ such that $\sum_{i\in\Cal{B}\cup\Cal{C}}\lambda^{\delta+\theta}_i$ is maximized. Subject to the constraint $|\Cal{B}|+|\Cal{C}|\le\ell$, clearly $\sum_{i\in\Cal{B}\cup\Cal{C}}\lambda^{\delta+\theta}_i$ is maximized only when $\Cal{B}=\{1,\ldots,\ell\}$, $\Cal{C}=\emptyset$. This yields $Q_1=\Al^{-\alpha}$, $Q_2=Q_3=0$ and the result follows by noting that $Q=Q_1+Q_2+Q_3$ and $\Cal{R}^A_{\alpha,\delta,\theta}(\Al^{-\alpha})=\sum_{i>\ell}\lambda^{\delta+\theta}_i$.  \vspace{2mm}\\
\emph{(ii)} Let $P_{\overline{\text{Ran}}(A)}:=\sum_{i}\psi_i\otimes_H\psi_i$ and $P^\perp$ denote the orthogonal projection operators that project onto $\overline{\text{Ran}}(A)$ and $\text{Ker}(A)$ respectively. Then $I=P_{\overline{\text{Ran}}(A)}+P^\perp$. Therefore, 
\begin{eqnarray}
 \norm{(I-Q)A^{\rho/2}}^2_{\Cal{L}^2(H)}
 &{}={}&\norm{(P_{\overline{\text{Ran}}(A)}+P^\perp)(I-Q)A^{\rho/2}}^2_{\Cal{L}^2(H)}\nonumber\\
 &{}={}&\norm{P_{\overline{\text{Ran}}(A)}(I-Q P_{\overline{\text{Ran}}(A)})A^{\rho/2}+P^\perp(I-Q)A^{\rho/2}}^2_{\Cal{L}^2(H)}\nonumber\\
 &{}={}&\norm{P_{\overline{\text{Ran}}(A)}(I-Q P_{\overline{\text{Ran}}(A)})A^{\rho/2}}^2_{\Cal{L}^2(H)}
 +\norm{P^\perp(I-Q)A^{\rho/2}}^2_{\Cal{L}^2(H)}\nonumber\\
 &{}\stackrel{(*)}{=}{}&\Cal{R}^A_{0,0,\rho}(Q)+\norm{P^\perp(I-Q)A^{\rho/2}}^2_{\Cal{L}^2(H)}
 \ge \Cal{R}^A_{0,0,\rho}(Q)\stackrel{(\dagger)}{\ge} \Cal{R}^A_{0,0,\rho}(Q_\ell),\nonumber
\end{eqnarray}
where $(\dagger)$ follows from Lemma~\ref{lem:optima}\emph{(i)}. Note that, at the minimizer of $\Cal{R}^A_{0,0,\rho}$, which is $Q_\ell$, the second term in $(*)$ is zero, which implies $Q_\ell$ is the minimizer of $\Cal{S}^A_{\rho}(Q)$ over $Q\in\Cal{Q}_\ell$. Therefore $\Cal{S}^A_{\rho}(Q_\ell)=\Cal{R}^A_{0,0,\rho}(Q_\ell)=\sum_{i>\ell}\lambda^\rho_i$.
\end{proof}

The following result extends Lemma 3.6 of \citep{Rudi-13}, which holds for uncentered covariance operators, to centered covariance operators that are estimated using a $U$-statistic.
\begin{appxlem}\label{lem:cent 1}
Let $H$ be a separable Hilbert space and $\Y$ be a separable topological space. Define $$\frak{C}=\frac{1}{2}\int_\Y\int_\Y (s(x)-s(y))\otimes_H (s(x)-s(y))\,dP(x)\,dP(y)$$ where $s:\Cal{Y}\rightarrow H$ is a Bochner-measurable function with $\sup_{x\in\Cal{Y}}\Vert s(x)\Vert^2_H=\kappa$. Given $(Y_i)^r_{i=1}\stackrel{i.i.d.}{\sim}P$ with $r\ge 2$, define 
$$\widehat{\frak{C}}=\frac{1}{2r(r-1)}\sum^r_{i\ne j}(s(Y_i)-s(Y_j))\otimes_H (s(Y_i)-s(Y_j)).$$
Then for any $0\le\delta\le\frac{1}{2}$ and $\frac{140\kappa}{r}\log\frac{16\kappa r}{\delta}\le t\le \norm{\Ck}_\OH$, the following hold:\vspace{2mm}\\
    (i) $P^r\left\{(Y_i)^r_{i=1}:\left\Vert (\Ck+tI)^{-1/2}(\Ch-\Ck)(\Ck+tI)^{-1/2}\right\Vert_\OH\le\frac{1}{2}\right\}\ge 1-2\delta$;\vspace{2mm}\\
    (ii) $P^r\left\{(Y_i)^r_{i=1}:\sqrt{\frac{2}{3}}\le\norm{(\Ck+tI)^{1/2}(\Ch+tI)^{-1/2}}_\OH\le\sqrt{2}\right\}\ge 1-2\delta;$\vspace{2mm}\\
    (iii) $P^r\left\{(Y_i)^r_{i=1}:\norm{(\Ck+tI)^{-1/2}(\Ch+tI)^{1/2}}_\OH\le\sqrt{\frac{3}{2}}\right\}\ge 1-2\delta;$\vspace{2mm}\\
    (iv) $P^r\left\{(Y_i)^r_{i=1}:\lambda_\ell(\Ch)+t\le\frac{3}{2}(\lambda_\ell(\Ck)+t)\right\}\ge 1-2\delta$ for all $\ell\ge 1$;\vspace{2mm}\\
    (v) $P^r\left\{(Y_i)^r_{i=1}:\lambda_\ell(\Ck)+t\le2(\lambda_\ell(\Ch)+t)\right\}\ge 1-2\delta$ for all $\ell\ge 1$.\vspace{2mm}\\
In addition, for any $0<t\le\Vert\frak{C}\Vert_{\Cal{L}^\infty(H)}$,
\begin{eqnarray}&{}{}&P^r\left\{(Y_i)_{i=1}^r:\norm{\frak{C}_t^{-1/2}(\widehat{\frak{C}}-\frak{C})\frak{C}^{1/2}}_{\Cal{L}^2(H)}\le\sqrt{\frac{64\kappa^{5/2}\Cal{N}_{\frak{C}}(t)\log\frac{2}{\delta}}{r\sqrt{t}}}+\frac{32\sqrt{2}\kappa^{3/2}\log\frac{3}{\delta}}{r\sqrt{t}}
\right\}\nonumber\\
&{}{}&\qquad\qquad\qquad\qquad\ge1-2\delta,\label{Eq:hs-op}
\end{eqnarray}
where $\frak{C}_t:=(\frak{C}+t I)$ and $\Cal{N}_\frak{C}(t)=\emph{tr}(\frak{C}_t^{-1}\frak{C})$.
\end{appxlem}
\begin{proof}
\emph{(i)} Define $A(x,y):=\frac{1}{\sqrt{2}}(s(x)-s(y))$, $U(x,y):=(\Ck+tI)^{-1/2} A(x,y)\in H$ and $Z(x,y):=U(x,y)\otimes_H U(x,y)$. Clearly $Z(x,y)=Z(y,x)$ and $$(\Ck+tI)^{-1/2} (\Ch-\Ck)(\Ck+tI)^{-1/2}=\frac{1}{r(r-1)}\sum^r_{i\ne j}Z(Y_i,Y_j)-\E[Z(X,Y)].$$ Also 
\begin{eqnarray*}\sup_{x,y\in\Y}\Vert Z(x,y)\Vert_{\Cal{L}^2(H)}&{}\stackrel{(*)}{=}{}&\sup_{x,y\in\Y}\Vert U(x,y)\Vert^2_H
\le \frac{1}{2}\sup_{x,y\in\Y}\Vert (\Ck+tI)^{-1/2} (s(x)-s(y))\Vert^2_H=\frac{2\kappa}{t},\end{eqnarray*}
where $(*)$ follows from Lemma~\ref{lem:rank}. Define
$$\psi(x):=\E_Y[Z(x,Y)]=\E_Y[U(x,Y)\otimes_H U(x,Y)].$$ Clearly, 
\begin{eqnarray*}
\sup_{x\in\Cal{Y}}\Vert \psi(x)\Vert_{\Cal{L}^\infty(H)}\le \sup_{x,y\in\Y}\Vert U(x,y)\otimes_H U(x,y)\Vert_{\Cal{L}^\infty(H)}\stackrel{(\dagger)}{=}\sup_{x,y\in\Y}\Vert U(x,y)\Vert^2_H\le\frac{2\kappa}{t},
\end{eqnarray*}
where $(\dagger)$ follows from Lemma~\ref{lem:rank}. Since $\E[\psi(X)]=\E[Z(X,Y)]$, $
\E[(\psi(X)-\E[Z(X,Y)])^2]=\E[\psi^2(X)]-\E^2[Z(X,Y)]\preceq\E[\psi^2(X)]$. By defining $\Ck_t=\Ck+tI$, we have
\begin{eqnarray*}
&&\E[\psi^2(X)]=\E[\E^2_Y[U(X,Y)\otimes_H U(X,Y)]]\nonumber\\
&{}={}&\E\left[\Ck_t^{-1/2}\E_Y[A(X,Y)\otimes_H A(X,Y)]\Ck_t^{-1}\E_Y[A(X,Y)\otimes_H A(X,Y)]\Ck_t^{-1/2}\right]\nonumber
\\
&{}\preceq{}&\sup_{x\in\Y}\Vert \Ck_t^{-1/2} \E_Y[A(x,Y)\otimes_H A(x,Y)]\Ck^{-1/2}_t\Vert_{\Cal{L}^\infty(H)}\\
&{}{}&\qquad\qquad\times\E\left[\Ck^{-1/2}_t\E_Y[A(X,Y)\otimes_H A(X,Y)]\Ck_t^{-1/2}\right]\nonumber
\\
&{}\preceq{}&\frac{2\kappa}{t}(\Ck+tI)^{-1/2}\Ck (\Ck+tI)^{-1/2}=:S.\nonumber
\end{eqnarray*}
Note that $\Vert S\Vert_{\Cal{L}^\infty(H)}\le \frac{2\kappa}{t}$ and $$d:=\frac{\Vert S\Vert_{\Cal{L}^1(H)}}{\Vert S\Vert_{\Cal{L}^\infty(H)}}=\frac{\text{tr}(\Ck_t^{-1}\Ck)}{\Vert\Ck_t^{-1}\Ck\Vert_{\Cal{L}^\infty(H)}}\le \frac{(\Vert \Ck\Vert_{\Cal{L}^\infty(H)}+t)\text{tr}(\Ck^{-1}_t\Ck)}{\Vert\Ck\Vert_{\Cal{L}^\infty(H)}}.$$ Therefore, applying Theorem~\ref{thm:bernstein U-stat} yields that for $0<\delta\le d$ with probability at least $1-2\delta$,
\begin{eqnarray}
\left\Vert(\Ck+tI)^{-1/2}(\Ch-\Ck)(\Ck+tI)^{-1/2}\right\Vert_{\Cal{L}^\infty(H)}&{}\le{}& \frac{4\kappa\beta}{rt}+\sqrt{\frac{24\kappa\beta}{rt}}+\frac{16\kappa\log\frac{3}{\delta}}{rt}\nonumber
\\
&{}\le{}& \frac{4\kappa\beta}{rt}+\sqrt{\frac{24\kappa\beta}{rt}}+\frac{24\kappa\beta}{rt}\nonumber
\\
&=&\frac{28\kappa\beta}{rt}+\sqrt{\frac{24\kappa\beta}{rt}},\label{Eq:final}
\end{eqnarray}
where $\beta=\frac{2}{3}\log\frac{4d}{\delta}$ and we used that fact that $d>1$ in the second line. Since $t\ge \frac{140\kappa}{r}\log\frac{16\kappa r}{\delta}$, it follows that $t\ge \frac{140\kappa}{r}\log\frac{4d}{\delta}$ as $\frac{140\kappa}{r}\log\frac{16\kappa r}{\delta}\ge \frac{140\kappa}{r}\log\frac{16\kappa}{t\delta}\ge \frac{140\kappa}{r}\log\frac{4d}{\delta}$ where we use the fact that $d\le\frac{4\kappa}{t}$ which follows from $t\le \Vert\Ck\Vert_{\Cal{L}^\infty(H)}$ and $\text{tr}(\Ck^{-1}_t\Ck)\le \frac{\text{tr}(\Ck)}{t}\le \frac{2\kappa}{t}$. This implies $t\ge \frac{210\kappa\beta}{r}$ or $\frac{\kappa\beta}{rt}\le \frac{1}{210}$. Using this in \eqref{Eq:final} yields the result.\vspace{1mm}\\
\emph{(ii)} By defining $B_n=(\Ck+tI)^{-1/2}(\Ck-\Ch)(\Ck+tI)^{-1/2}$, we have
\begin{eqnarray*}
    \norm{(\Ck+tI)^{1/2}(\Ch+tI)^{-1/2}}_{\Cal{L}^\infty(H)}
    &{}={}&\norm{(\Ch+tI)^{-1/2}(\Ck+tI)(\Ch+tI)^{-1/2}}^{1/2}_{\Cal{L}^\infty(H)}\nonumber\\
    &{}={}&\norm{(\Ck+tI)^{1/2}(\Ch+tI)^{-1}(\Ck+tI)^{1/2}}^{1/2}_{\Cal{L}^\infty(H)}\nonumber\\
    &{}={}&\norm{(I-B_n)^{-1}}_{\Cal{L}^\infty(H)}^{1/2}
    \le (1-\norm{B_n}_{\Cal{L}^\infty(H)})^{-1/2},\nonumber
\end{eqnarray*}
where the last inequality holds whenever $\norm{B_n}_{\Cal{L}^\infty(H)}<1$. Similarly,
\begin{eqnarray*}
    \norm{(\Ck+tI)^{1/2}(\Ch+tI)^{-1/2}}_{\Cal{L}^\infty(H)}&{}={}&\norm{(I+(-B_n))^{-1}}_{\Cal{L}^\infty(H)}^{1/2}\ge(1+\norm{B_n}_{\Cal{L}^\infty(H)})^{-1/2}.\nonumber
\end{eqnarray*}
The result therefore follows from $(i)$.\vspace{1mm}
%
\\
$(iii)$ Since
$$\norm{(\Ck+tI)^{-1/2}(\Ch+tI)^{1/2}}_{\Cal{L}^\infty(H)}=\norm{I-B_n}_{\Cal{L}^\infty(H)}^{1/2}\le(1+\norm{B_n}_{\Cal{L}^\infty(H)})^{1/2},$$
the result follows from $(i)$.\vspace{1mm}
\\
$(iv)$ Since $\sqrt{\frac{2}{3}}\le\norm{(\Ck+tI)^{1/2}(\Ch+tI)^{-1/2}}_{\Cal{L}^\infty(H)}\le\sqrt{2}$ as obtained in $(i)$, it follows that $\Ch+tI\preceq\frac{3}{2}(\Ck+tI)$ (see \citealp[Lemmas B.2 and 3.5]{Rudi-13}). This implies (see \citealp{gohberg}) that $\lambda_\ell(\Ch)+t\le \frac{3}{2}(\lambda_\ell(\Ck)+t)$ for all $\ell\ge 1$. $(v)$ follows similarly.\vspace{1mm}\\
\emph{Proof of \eqref{Eq:hs-op}:} Define $Z(x,y):=\frak{C}_t^{-1/2}(A(x,y)\otimes_H A(x,y))\frak{C}^{1/2}$ so that $Z(x,y)=Z(y,x)$ and
$$\frak{C}_t^{-1/2}(\hat{\frak{C}}-\frak{C})\frak{C}^{1/2}=\frac{1}{r(r-1)}\sum_{i\neq j}^rZ(X_i,X_j)-\E[Z(X,Y)].$$
We have
$\sup_{x,y\in\X}\norm{Z(x,y)}_{\Cal{L}^2(H)}\le\Vert\frak{C}_t^{-1/2}\Vert_{\Cal{L}^\infty(H)}\norm{\frak{C}^{1/2}}_{\Cal{L}^\infty(H)}\norm{A(x,y)}_H^2\le\frac{(2\kappa)^{3/2}}{\sqrt{t}}:=M.$

By defining $\psi(x):=\E_Y[Z(x,Y)]$, we have
\begin{eqnarray*}
&&\E\Vert\psi(X)-\frak{C}\Vert^2_{\Cal{L}^2(H)}=\E\Vert\psi(X)\Vert^2_{\Cal{L}^2(H)}-\Vert\frak{C}\Vert^2_{\Cal{L}^2(H)}\le\E\Vert\psi(X)\Vert^2_{\Cal{L}^2(H)}\nonumber\\
&{}={}&\E\left\Vert\frak{C}_t^{-1/2}\E_Y[A(X,Y)\otimes_H A(X,Y)]\frak{C}^{1/2}\right\Vert^2_{\Cal{L}^2(H)}\nonumber\\
&{}={}&\E\,\text{tr}\left[\frak{C}^{1/2}\E_Y[A(X,Y)\otimes_H A(X,Y)]\frak{C}_t^{-1}\E_Y[A(X,Y)\otimes_H A(X,Y)]\frak{C}^{1/2}\right]\nonumber\\
&{}={}&\E\,\text{tr}\left[\frak{C}^{-1/2}_t \E_Y[A(X,Y)\otimes_H A(X,Y)]\frak{C}_t^{-1}\E_Y[A(X,Y)\otimes_H A(X,Y)]\frak{C}\frak{C}^{1/2}_t\right]\nonumber
\\
&{}\le{}&\sup_{x\in\X}\norm{\frak{C}_t^{-1/2}\E_Y[A(X,Y)\otimes_H A(X,Y)]\frak{C}\frak{C}^{1/2}_t}_{\Cal{L}^\infty(H)}\\
&{}{}&\qquad\qquad\times \E\,\text{tr}\left[\frak{C}_t^{-1/2}\E_Y[A(X,Y)\otimes_H A(X,Y)]\frak{C}_t^{-1/2}\right]\nonumber
\\
&{}\le{}&\norm{\frak{C}_t^{-1/2}}_{\Cal{L}^\infty(H)}\norm{\frak{C}}_{\Cal{L}^\infty(H)}\norm{\frak{C}^{1/2}_t}_{\Cal{L}^\infty(H)}\text{tr}\left[\frak{C}^{-1}_t\frak{C}\right] \sup_{x,y\in\Cal{X}}\norm{A(x,y)\otimes_H A(x,y)}_{\Cal{L}^\infty(H)}\\
&{}\le{}&\sqrt{2\kappa+t}(2\kappa)^2\frac{\Cal{N}_{\frak{C}}(t)}{\sqrt{t}}\le \sqrt{2}(2\kappa)^{5/2}\frac{\Cal{N}_{\frak{C}}(t)}{\sqrt{t}},\nonumber
\end{eqnarray*}
where we used $t\le\Vert\frak{C}\Vert_{\Cal{L}^\infty(H)}\le 2\kappa$ in the last inequality. The result follows by applying Theorem \ref{thm:bernstein U-stat}~\emph{(ii)}.
%
%
\end{proof}
\begin{appxlem}\label{lem:1 rff}
Suppose $(A_1)$, $(A_2)$, $(A_4)$ and $(A_5)$ hold. For $t>0$, define $\Cal{N}_\Sigma(t)=\emph{tr}(\Sigma(\Sigma+tI)^{-1})$ and $\Cal{N}_{\Sigma_m}(t)=\emph{tr}(\Sigma_m(\Sigma_m+tI)^{-1})$. For $\delta>0$ and $\frac{86\kappa}{m}\log\frac{16\kappa m}{\delta}\le t\le \norm{\Sigma}_\OPH$, the following hold:\vspace{2mm}\\ 
    (i) 
    $\Lambda^m\left\{(\theta_i)^m_{i=1}:\sqrt{\frac{2}{3}}\le\norm{(\id\id^*+tI)^{1/2}(\tS\tS^*+tI)^{-1/2}}_{\Cal{L}^\infty(\lp)}\le\sqrt{2}\right\}\ge 1-\delta;$\vspace{2mm}\\
    (ii) $\Lambda^m\left\{(\theta_i)^m_{i=1}:\lambda_{m,j}+t\le\frac{3}{2}(\lambda_j+t)\right\}\ge 1-\delta$ for all $j\ge 1$;\vspace{2mm}\\
    (iii) $\Lambda^m\left\{(\theta_i)^m_{i=1}:\frac{1}{2}(\lambda_j+t)\le \lambda_{m,j}+t\right\}\ge 1-\delta$ for all $j\ge 1$;\vspace{2mm}\\
    (iv) $\Lambda^m\left\{(\theta_i)^m_{i=1}:\Cal{N}_{\Sigma_m}(t)\le\frac{32\kappa\log\frac{2}{\delta}}{tm}+\sqrt{\frac{32\kappa\Cal{N}_\Sigma(t)\log\frac{2}{\delta}}{tm}}+2\Cal{N}_\Sigma(t)\right\}\ge 1-2\delta$.
\end{appxlem}
\begin{proof}
$(i,ii,iii)$ Define $A_i:=\varphi(\cdot,\theta_i)-(1\ol 1)\varphi(\cdot,\theta_i)$ and $D_i:=A_i\ol A_i$. Then it follows from Propositions~\ref{pro:id} and \ref{pro:approx} that $\tS\tS^*=\frac{1}{m}\sum^m_{i=1}D_i$ and $\id\id^*=\bb{E}[\tS\tS^*]$. Define $E_m:=(\id\id^*+tI)^{-1/2}(\id\id^*-\tS\tS^*)(\id\id^*+tI)^{-1/2}$.
By mimicking the strategy of Lemma \ref{lem:cent 1}\emph{(ii, iii)}, we obtain
\begin{eqnarray}\label{lem 3: Bm upper/lower}
   \qquad\quad(1+\norm{E_m}_{\Cal{L}^\infty(\lp)})^{-1/2} &{}\le{}&\norm{(\id\id^*+tI)^{1/2}(\tS\tS^*+tI)^{-1/2}}_{\Cal{L}^\infty(\lp)}\\
   &{}\le{}&(1-\norm{E_m}_{\Cal{L}^\infty(\lp)})^{-1/2}\nonumber
\end{eqnarray}
provided $\norm{E_m}_{\Cal{L}^\infty(\lp)}<1$. We will now apply Theorem \ref{thm: tropp} to bound $\norm{E_m}_{\Cal{L}^\infty(\lp)}$. By defining
$Z_i:=(\id\id^*+tI)^{-1/2}A_i$
%
and $U_i:=Z_i\ol Z_i$, we obtain $E_m=\frac{1}{m}\sum^m_{i=1}U_i-\E_\Lambda[U_i]$. Note that
\begin{eqnarray*}
\norm{U_i}_{\Cal{L}^\infty(\lp)}&{}={}&\norm{Z_i}^2_{\lp}\le \frac{1}{t}\norm{A_i}^2_{\lp}\\
&{}\le{}& \frac{2\norm{\varphi(\cdot,\theta_i)}^2_{\lp}}{t}\left(1+\norm{1\ol1}_{\Cal{L}^\infty(\lp)}\right)
\le \frac{4\kappa}{t}.\nonumber
\end{eqnarray*}
%
%
Define $T:=E_\Lambda[U_i]$. Then $$\E_\Lambda[(U_i-T)^2]=\E_\Lambda[\norm{Z_i}_{L^2(\Pb)}^2U_i-T^2]\preceq\E_\Lambda[\norm{Z_i}_{L^2(\Pb)}^2U_i]\preceq\frac{4\kappa}{t}T.$$
Now we set
$$
    \sigma^2=\norm{\frac{4\kappa}{t}T}_{\Cal{L}^\infty(\lp)}\le \frac{4\kappa}{t}$$ and $$d=\frac{\norm{T}_{\Cal{L}^1(\lp)}}{\norm{T}_{\Cal{L}^\infty(\lp)}}\le\frac{(\lambda_1+t)\norm{T}_{\Cal{L}^1(\lp)}}{\lambda_1},
$$
where $\lambda_1=\Vert\Sigma\Vert_\OPH=\Vert\id\id^*\Vert_\OPL$. Then Theorem \ref{thm: tropp} yields
\begin{equation}\label{eq:end}
    \Lambda^m\left\{\norm{B_m}_{\Cal{L}^\infty(\lp)}\le\frac{8\beta \kappa}{3tm}+\sqrt{\frac{8\kappa\beta }{tm}}\right\}\le1-\delta,
\end{equation}
where $\beta=\log\frac{4d}{\delta}$. Since $t\ge\frac{86\kappa}{m}\log\frac{16\kappa m}{\delta}$, it follows that $t\ge \frac{86\kappa}{m}\log\frac{4d}{\delta}$ as $\frac{86\kappa}{m}\log\frac{16\kappa m}{\delta}\ge\frac{86\kappa}{m}\log\frac{16\kappa}{t\delta}\ge\frac{86\kappa}{m}\log\frac{4d}{\delta}$, where we have used $d\le\frac{4\kappa}{t}$ which follows from $t\le\norm{\Sigma}_\OPH$ and $\text{tr}(T)\le\frac{\text{tr}(\id\id^*)}{t}=\frac{\text{tr}(\Sigma)}{t}\le\frac{2\kappa}{t}.$  This implies $t\ge\frac{86\beta\kappa}{m}$. Combining this with (\ref{eq:end}) yields that with probability at least $1-\delta$, $\Vert B_m\Vert_\OPL\le \frac{1}{2}$. $(i)$ follows by using this in (\ref{lem 3: Bm upper/lower}).  $(ii),\,(iii)$ are implied as in $(iv)$, $(v)$ of Lemma \ref{lem:cent 1}.\vspace{1mm}\\
$(iv)$ Observe that $\Cal{N}_{\Sigma_m}(t)=\text{tr}(\Sigma_m(\Sigma_m+tI)^{-1})=\text{tr}(\tS^*\tS(\tS^*\tS+tI)^{-1})=\text{tr}(\tS(\tS^*\tS+tI)^{-1}\tS^*)=\text{tr}((\tS\tS^*+tI)^{-1}\tS\tS^*),$
where we have used the fact that $\tS(\tS^*\tS+tI)^{-1}=(\tS\tS^*+tI)^{-1}\tS$ and the invariance of trace under cyclic permutations. Similarly, it can be shown that $\Cal{N}_\Sigma(t)=\text{tr}((\id\id^*+t I)^{-1}\id\id^*)$. For the ease of notation, define $A:=\tS\tS^*$, $B:=\id\id^*$, $A_t:=A+tI$ and $B_t:=B+tI$. Then
\begin{eqnarray*}
A+tI&{}={}&(A-B)+(B+tI)\\
&{}={}&(B+tI)^{1/2}\left(I+(B+tI)^{-1/2}(A-B)(B+tI)^{-1/2}\right)(B+tI)^{1/2}
\end{eqnarray*}
implying,
$$A_t^{-1}=B_t^{-1/2}\left(I+B_t^{-1/2}(A-B)B_t^{-1/2}\right)^{-1}B_t^{-1/2}.
$$
Therefore,
\begin{eqnarray*}
\Cal{N}_m(t)&{}={}&\text{tr}(AA^{-1}_t)=\text{tr}\left[AB_t^{-1/2}\left(I+B_t^{-1/2}(A-B)B_t^{-1/2}\right)^{-1}B_t^{-1/2}\right]\nonumber
\\
&{}={}&\text{tr}\left[B_t^{-1/2}AB_t^{-1/2}\left(I+B_t^{-1/2}(A-B)B_t^{-1/2}\right)^{-1}\right]\nonumber
\\
&{}\le{}&\norm{(I+B_t^{-1/2}(A-B)B_t^{-1/2})^{-1}}_{\Cal{L}^\infty(\lp)} \text{tr}(B_t^{-1/2}AB_t^{-1/2})\nonumber\\
&{}={}&\norm{(I-E_m)^{-1}}_{\Cal{L}^\infty(\lp)}\text{tr}(B_t^{-1/2}AB_t^{-1/2}),\nonumber
\end{eqnarray*}
where
$E_m:=B_t^{-1/2}(B-A)B_t^{-1/2}=(\id\id^*+tI)^{-1/2}(\id\id^*-\tS\tS^*)(\id\id^*+tI)^{-1/2}$. Since 
we showed in the proof of $(i)$ that with probability at least $1-\delta$, $\norm{E_m}_{\Cal{L}^\infty(\lp)}\le\frac{1}{2}$, we obtain 
\begin{equation}\Cal{N}_m(t)\le 2\,\text{tr}(B_t^{-1/2}AB_t^{-1/2}),\label{Eq:Nm}\end{equation}
where we use $\norm{(I-E_m)^{-1}}_{\Cal{L}^\infty(\lp)}\le \frac{1}{1-\norm{E_m}_{\Cal{L}^\infty(\lp)}}$.
Next, consider
\begin{eqnarray}
\text{tr}(B_t^{-1/2}AB_t^{-1/2})&{}={}&\text{tr}(B_t^{-1}(A-B+B))\nonumber\\
&{}={}&\inner{B_t^{-1}}{A-B}_\HSL+\Cal{N}_\Sigma(t),\label{Eq:Nsigma}
\end{eqnarray}
where $\inner{B_t^{-1}}{A-B}_\HSL=\inner{(\id\id^*+tI)^{-1}}{\tS\tS^*-\id\id^*}_\HSL.$ We now bound this term as follows.
Let
$$\zeta_i=\left(\varphi(\cdot,\theta_i)-(1\ol 1)\varphi(\cdot,\theta_i)\right)\ol\left(\varphi(\cdot,\theta_i)-(1\ol 1)\varphi(\cdot,\theta_i)\right)$$
so that $\bb{E}_\Lambda[\zeta_1]=\id\id^*$, $\frac{1}{m}\sum^m_{i=1}\zeta_i=\tS\tS^*$ and $$\inner{B_t^{-1}}{\tS\tS^*-\id\id^*}_\HSL=\frac{1}{m}\sum_{i=1}^m\inner{(\id\id^*+tI)^{-1}}{(\zeta_i-\id\id^*)}_\HSL.$$
%
We will now apply Bernstein's inequality (Theorem~\ref{thm:bernstein}). To this end, note that
\begin{eqnarray*}
&&\left|\inner{(\id\id^*+tI)^{-1}}{\zeta_1-\id\id^*}_\HSL\right|\\
&{}{}&\qquad\qquad\le\inner{(\id\id^*+tI)^{-1}}{\id\id^*}_\HSL+\inner{(\id\id^*+tI)^{-1}}{\zeta_i}_\HSL\nonumber\\
&{}={}&\Cal{N}_\Sigma(t)+\frac{1}{t}\text{tr}\left[\tau_i\ol\tau_i\right]\nonumber
\\
&{}={}&\Cal{N}_\Sigma(t)+\frac{1}{t}\norm{\tau_i}_{L^2(\Pb)}^2\le\frac{\norm{\Sigma}_\TCH+4\kappa}{t}\le\frac{8\kappa}{t},\nonumber
\end{eqnarray*}
where we use $\Vert\Sigma\Vert_\TCH\le \bb{E}\norm{\kbar\oh\kbar}_\TCH=\bb{E}\norm{\kbar}^2_\Cal{H}\le4\kappa$ and $\tau_i:=\varphi(\cdot,\theta_i)-(1\ol 1)\varphi(\cdot,\theta_i)$. Also
\begin{eqnarray*}
\bb{E}_\Lambda\inner{(\id\id^*+tI)^{-1}}{\zeta_1-\id\id^*}_\HSL^2
=\bb{E}_\Lambda\left[\left(\inner{(\id\id^*+tI)^{-1}}{\zeta_1}_\HSL-\Cal{N}_\Sigma(t)\right)^2\right]\nonumber
\end{eqnarray*}
\begin{eqnarray*}
&{}{}&\qquad=\bb{E}_\Lambda\left[\inner{(\id\id^*+tI)^{-1}}{\zeta_1}_\HSL^2-\Cal{N}_\Sigma^2(t)\right]\\
&{}{}&\qquad\le\bb{E}_\Lambda\inner{(\id\id^*+tI)^{-1}}{\zeta_1}_\HSL^2\nonumber\\
&{}{}&\qquad=\bb{E}_\Lambda\text{tr}\left[(\id\id^*+tI)^{-1}\zeta_1(\id\id^*+tI)^{-1}\zeta_1\right]\nonumber
\\
&{}{}&\qquad\le\sup_{\theta_1}\norm{(\id\id^*+tI)^{-1/2}\zeta_1(\id\id^*+tI)^{-1/2}}_{\Cal{L}^\infty(\lp)}\bb{E}_\Lambda \left[\text{tr}\left((\id\id^*+tI)^{-1/2}\zeta_1(\id\id^*+tI)^{-1/2}\right)\right]\nonumber\\
&{}{}&\qquad\le\frac{\Cal{N}_\Sigma(t)}{t}\sup_{\theta_1}\norm{\varphi(\cdot,\theta_1)-(1\ol 1)\varphi(\cdot,\theta_1)}^2_{\lp}\le\frac{4\kappa\Cal{N}_\Sigma(t)}{t}.\nonumber
\end{eqnarray*}
The result follows by applying Theorem~\ref{thm:bernstein} to $\inner{B_t^{-1}}{(A-B)}_\HSL$ and combining \eqref{Eq:Nm} and \eqref{Eq:Nsigma}.
\end{proof}
\begin{appxlem}\label{lem:perturb}
Suppose $(A_1)$ and $(A_4)$ hold. Then for any $0<\delta<1$ and $m\ge 2\log\frac{2}{\delta}$, 
\begin{equation*}\Lambda^m\left\{(\theta_i)^m_{i=1}:\Vert \frak{A}\frak{A}^*-\id\id^*\Vert_{\HSl}\le 4\kappa\sqrt{\frac{2\log\frac{2}{\delta}}{m}}\right\}\ge 1-\delta.\nonumber
\end{equation*}
\end{appxlem}
\begin{proof}
From Proposition~\ref{pro:id}\emph{(iv)}, Lemma~\ref{lem:theta} and Proposition~\ref{pro:approx}\emph{(iv)}, we have $$\id\id^*=\Upsilon-(1\ol 1)\Upsilon-\Upsilon(1\ol 1)+(1\ol 1)\Upsilon(1\ol 1)$$ 
and $$\frak{A}\frak{A}^*=\Pi -(1\ol 1)\Pi -\Pi(1\ol 1)+(1\ol 1)\Pi (1\ol 1)$$ where $$\Upsilon:=\intt \vp(\cdot,\theta)\ol \vp(\cdot,\theta)\,d\Lambda(\theta)$$ and $$
\Pi:=\sum^m_{i=1}\vp_i\ol\vp_i=\frac{1}{m}\sum^m_{i=1}\vp(\cdot,\theta_i)\ol \vp(\cdot,\theta_i).$$
Define $A_i:=\varphi(\cdot,\theta_i)-(1\ol 1)\varphi(\cdot,\theta_i)$ and $D_i:=A_i\ol A_i$. Then it follows that $\tS\tS^*=\frac{1}{m}\sum^m_{i=1}D_i$ and $\id\id^*=\bb{E}[\tS\tS^*]$. The result follows by applying Theorem~\ref{thm:bernstein} with $B=\theta=\sup_{\theta_1}\norm{A_1\ol A_1}_{\HSl}=\sup_{\theta_1}\norm{A_1}^2_{\lp}\le 2\sup_{\theta_1}\norm{\varphi(\cdot,\theta_1)}^2_{\lp}\le 2\kappa$ and noting that $\HSl$ is a separable Hilbert space since $L^2(\bb{P})$ is separable. 
\end{proof}
\begin{appxlem}\label{lem:kernel mean bnd}
Suppose $(A_1)$ and $(A_4)$ hold. For any $0<\delta<1$ with $n\ge 2\log\frac{2}{\delta}$, then the following hold:
\begin{itemize}
 \item[(i)] $\Pb^n\left\{(X_i)_{i=1}^n:\norm{m_\Pb-\widehat{m}_\Pb}_\Hk^2\le\frac{32\kappa\log\frac{2}{\delta}}{n}\right\}\ge1-\delta$;
 \item[(ii)] $\Pb^n\left\{(X_i)_{i=1}^n:  \norm{m_{\Pb,m}-\widehat{m}_{\Pb,m}}_{\Hk_m}^2\le\frac{32\kappa\log\frac{2}{\delta}}{n}\Big{|}(\theta_i)^m_{i=1}\right\}\ge1-\delta$.
\end{itemize}
\end{appxlem}
\begin{proof}
Define $\xi_i=k(\cdot,X_i)-\int_\X k(\cdot,x)d\Pb(x)$. Clearly $\frac{1}{n}\sum_{i=1}^n\xi_i=\widehat{m}_\Pb-m_\Pb$. Note that $\norm{\xi_i}_\Cal{H}\le 2\sqrt{\kappa}$ for all $i$. The result therefore follows by applying Theorem~\ref{thm:bernstein} with $B=\theta=2\sqrt{\kappa}$. Conditioned on $(\theta_i)^m_{i=1}$, the second result follows exactly the first one with $k$ replaced by $k_m$. 
%
\end{proof}
\begin{appxlem}\label{lem:rf kernel mean bnd}
Suppose $(A_1)$ and $(A_4)$ hold. For any $\delta>0$ with $m\ge 2\log\frac{2}{\delta}$,
$$\Lambda^m\left\{\bb{E}\norm{\id\kbar-\tS\overline{k}_m(\cdot,X)}_{\lp}^2\le\frac{64\kappa^2\log\frac{2}{\delta}}{m}\right\}\ge1-2\delta.$$
\end{appxlem}
\begin{proof}
Note that 
\begin{eqnarray}
&&\bb{E}\norm{\id\kbar-\tS\overline{k}_m(\cdot,X)}_{\lp}^2\nonumber\\
&{}={}&\bb{E}\norm{\id\kbar}_{\lp}^2+\bb{E}\norm{\tS\overline{k}_m(\cdot,X)}_{\lp}^2-2\bb{E}\langle \id\kbar, \tS\overline{k}_m(\cdot,X)\rangle_{\lp}\nonumber\\
&{}\stackrel{(\dagger)}{=}{}&\Vert \Sigma\Vert^2_{\HSH}+\Vert \Sigma_m\Vert^2_{\Cal{L}^2(\Cal{H}_m)}-2\bb{E}\langle \id\kbar, \tS\overline{k}_m(\cdot,X)\rangle_{\lp}\nonumber\\
&{}={}&\Vert \id\id^*\Vert^2_{\HSL}+\Vert \tS\tS^*\Vert^2_{\HSL}-2\bb{E}\langle \id\kbar, \tS\overline{k}_m(\cdot,X)\rangle_{\lp}\nonumber\\
&{}={}&\Vert \id\id^*-\tS\tS^*\Vert^2_{\HSL}
+2\left[\text{tr}(\id\id^*\tS\tS^*)-\bb{E}\langle \id\kbar, \tS\overline{k}_m(\cdot,X)\rangle_{\lp}\right],\label{Eq:eq-diff}
\end{eqnarray}
where we used Lemma~\ref{rf cov proj rewrite} in $(\dagger)$. We will now focus on computing $\bb{E}\langle \id\kbar, \tS\overline{k}_m(\cdot,X)\rangle_{\lp}$ and $\text{tr}(\id\id^*\tS\tS^*)$. Note that $k(\cdot,x)=\int_\Theta \varphi(\cdot,\theta)\varphi(x,\theta)\,d\Lambda(\theta)$ and $k_m(\cdot,x)=\sum^m_{i=1}\varphi_i(x)\varphi_i$. Define $\varphi_\bb{P}(\theta):=\int_\Cal{X}\varphi(x,\theta)\,d\bb{P}(x)$, $\varphi_{i,\bb{P}}:=\int_\Cal{X} \varphi_i(x)\,d\bb{P}(x)$, $\mu(\cdot,\theta)=\varphi(\cdot,\theta)-\varphi_{\bb{P}}(\theta)$ and $\mu_i:=\varphi_i-\varphi_{i,\bb{P}}$. Therefore, 
\begin{eqnarray}
\bb{E}\langle \id\kbar, \tS\overline{k}_m(\cdot,X)\rangle_{\lp}
&{}={}&\int_\Cal{X} \langle \id\overline{k}(\cdot,x), \tS\overline{k}_m(\cdot,x)\rangle_{\lp}\,d\bb{P}(x)\nonumber\\
&{}\stackrel{(*)}{=}{}&\int_\Cal{X} \left\langle \int_\Theta \mu(\cdot,\theta)\varphi(x,\theta)\,d\Lambda(\theta), \sum^m_{i=1}\varphi_i(x)\mu_i\right\rangle_{\lp}\,d\bb{P}(x)\nonumber\\
&{}={}&\int_\Cal{X}\int_\Theta\sum^m_{i=1}\left\langle \mu(\cdot,\theta),\mu_i\right\rangle_{\lp}\varphi_i(x)\varphi(x,\theta)\,d\Lambda(\theta)\,d\bb{P}(x)\nonumber\\
&{}={}&\int_\Theta\sum^m_{i=1}\left\langle \mu(\cdot,\theta),\mu_i\right\rangle_{\lp}\left\langle \varphi_i,\varphi(\cdot,\theta)\right\rangle_{\lp}\,d\Lambda(\theta),\label{Eq:expe}
\end{eqnarray}
where the penultimate and last equalities follow by employing Fubini's theorem and $(*)$ follows from Propositions~\ref{pro:id} and \ref{pro:approx}. On the other hand, by defining $\tau(\cdot,\theta):=\varphi(\cdot,\theta)-\left(1\otimes_{\lp}1\right)\varphi(\cdot,\theta)$ and $\tau_i:=\varphi_i-\left(1\otimes_{\lp}1\right)\varphi_i$, we have
\begin{eqnarray}
\text{tr}(\id\id^*\tS\tS^*)
&{}\stackrel{(\ddagger)}{=}{}&\text{tr}\left[\int_\Theta \tau(\cdot,\theta)\otimes_{\lp}\tau(\cdot,\theta)\,d\Lambda(\theta)\sum^m_{i=1}\tau_i\otimes_{\lp}\tau_i\right]\nonumber\\
&{}={}&\text{tr}\left[\int_\Theta\sum^m_{i=1}\left\langle \mu(\cdot,\theta),\mu_i\right\rangle_{\lp}\mu(\cdot,\theta)\otimes_{\lp}\mu_i\,d\Lambda(\theta)\right]\nonumber\\
&{}={}&\int_\Theta\sum^m_{i=1}\left\langle \mu(\cdot,\theta),\mu_i\right\rangle_{\lp}\left\langle \mu(\cdot,\theta),\mu_i\right\rangle_{\lp}\,d\Lambda(\theta)\nonumber\\
&{}={}&\int_\Theta\sum^m_{i=1}\left\langle \mu(\cdot,\theta),\mu_i\right\rangle_{\lp}\left[\left\langle \varphi(\cdot,\theta),\varphi_i\right\rangle_{\lp}-\varphi_\bb{P}(\theta)\varphi_{i,\bb{P}}\right]\,d\Lambda(\theta),\label{Eq:tr}
\end{eqnarray}
where we used Propositions~\ref{pro:id}\emph{(iv)} and \ref{pro:approx}\emph{(iv)} in $(\ddagger)$. It follows from \eqref{Eq:expe} and \eqref{Eq:tr} that
\begin{eqnarray}
\text{tr}(\id\id^*\tS\tS^*)&{}={}&\bb{E}\langle \id\kbar, \tS\overline{k}_m(\cdot,X)\rangle_{\lp}-\left\langle \int_\Theta A(\theta)\,d\Lambda(\theta),\frac{1}{m}\sum^m_{i=1}A(\theta_i) \right\rangle_{\lp},\label{Eq:redu}
\end{eqnarray}
where $A(\theta)=\varphi(\cdot,\theta)\varphi_\bb{P}(\theta)-\varphi^2_\bb{P}(\theta)$. We remind the reader that $\varphi_i=\frac{1}{\sqrt{m}}\varphi(\cdot,\theta_i)$ with $(\theta_i)^m_{i=1}\stackrel{i.i.d.}{\sim}\Lambda$. Define $\Lambda_m$ to be the empirical measure based on $(\theta_i)^m_{i=1}$. Then, \eqref{Eq:redu} can be written as 
\begin{eqnarray}
\text{tr}(\id\id^*\tS\tS^*)
&{}={}&\bb{E}\langle \id\kbar, \tS\overline{k}_m(\cdot,X)\rangle_{\lp}+\frac{1}{2}\left\Vert \int_\Theta A(\theta)\,d(\Lambda_m-\Lambda)(\theta)\right\Vert^2_{\lp}\nonumber\\
&{}{}&\qquad\qquad-\frac{1}{2}\left\Vert \int_\Theta A(\theta)\,d\Lambda(\theta)\right\Vert^2_{\lp}-\frac{1}{2}\left\Vert \int_\Theta A(\theta)\,d\Lambda_m(\theta)\right\Vert^2_{\lp}\nonumber\\
&{}\le{}&\bb{E}\langle \id\kbar, \tS\overline{k}_m(\cdot,X)\rangle_{\lp}+\frac{1}{2}\left\Vert \int_\Theta A(\theta)\,d(\Lambda_m-\Lambda)(\theta)\right\Vert^2_{\lp},\label{Eq:fin-equiv}
\end{eqnarray}
which holds $\Lambda$-a.s. Using \eqref{Eq:fin-equiv} in \eqref{Eq:eq-diff}, we obtain
\begin{eqnarray}
\bb{E}\norm{\id\kbar-\tS\overline{k}_m(\cdot,X)}_{\lp}^2&{}\le{}& \Vert \id\id^*-\tS\tS^*\Vert^2_{\HSL}\nonumber\\
&{}{}&\qquad+ \left\Vert \int_\Theta A(\theta)\,d(\Lambda_m-\Lambda)(\theta)\right\Vert^2_{\lp},\label{Eq:finfin}
\end{eqnarray}
which holds $\Lambda$-a.s. The result follows by applying Lemma~\ref{lem:perturb} and Theorem~\ref{thm:bernstein} to \eqref{Eq:finfin}
by noting that $\sup_{\theta\in\Theta}\norm{A(\theta)}_{L^2(\Pb)}\le 2\kappa$ and $\E_{\theta\sim\Lambda}\norm{A(\theta)}^2_{L^2(\Pb)}\le 4\kappa^2$.
\end{proof}

\begin{appxlem}\label{rf cov proj rewrite}
Let $X$ be a separable topological space, $H$ be a separable Hilbert space and $\rho$ be a probability measure on $X$. Suppose $v:X\rightarrow H$ is Bochner-measurable and $\E_\rho\norm{v}^2_H:=\int_X \norm{v(x)}^2_H\,d\rho(x)<\infty$. Define $A=B^*B=\int_X v(x)\otimes_H v(x)\,d\rho(x)=:\E_{\rho}[v\otimes_H v]$ where $B:H\rightarrow G$ and $G$ is a separable Hilbert space. Then for any $Q:H\rightarrow H$,
$$\E_{\rho}\norm{BQv}_G^2=\norm{A^{1/2}QA^{1/2}}_{\Cal{L}^2(H)}^2.$$
%
\end{appxlem}
\begin{proof}
Note that\begin{eqnarray}
\E_\rho\norm{BQv}_G^2=\E_\rho\inner{BQv}{BQv}_G
=\E_\rho\inner{Q^*AQv}{v}_H
=\E_\rho\inner{Q^*AQ}{v\otimes_H v}_{\Cal{L}^2(H)}.\nonumber
\end{eqnarray}
Since $v$ is Bochner-measurable and $\E_\rho\norm{v}^2_H<\infty$, it is Bochner integrable, which yields
\begin{equation}
\E_\rho\inner{Q^*AQ}{v\otimes_H v}_{\Cal{L}^2(H)}
=\inner{Q^*AQ}{\E_\rho[v\otimes_H v]}_{\Cal{L}^2(H)}=\inner{Q^*AQ}{A}_{\Cal{L}^2(H)}.\nonumber
\end{equation}
The result follows by noting that
\begin{eqnarray*}
    \inner{Q^*AQ}{A}_{\Cal{L}^2(H)}=\text{tr}\left(Q^*AQA\right)=\text{tr}\left(A^{1/2}Q^*A^{1/2}A^{1/2}QA^{1/2}\right)
=\norm{A^{1/2}QA^{1/2}}_{\Cal{L}^2(H)}^2,\nonumber
\end{eqnarray*}
where we have used invariance of the trace under cyclic permutations.
\end{proof}
\begin{appxlem}\label{lem:rank}
Define $B=f\otimes_H f$ where $H$ is a separable Hilbert space and $f\in H$. Then $\Vert B\Vert_{\op(H)}=\Vert B\Vert_{\HS(H)}=\Vert B\Vert_{\Tr(H)}=\Vert f\Vert^2_H$.
\end{appxlem}
\begin{proof}
Since $B$ is self-adjoint, $$\Vert B\Vert_{\op(H)}=\lambda_1(B)=\sup_{\Vert g\Vert_H=1}\langle g,Bg\rangle_H= \sup_{\Vert g\Vert_H=1}\langle f,g\rangle^2_H= \Vert f\Vert^2_H.$$ Note that
$\Vert B\Vert_{\Tr(H)}=\sum_{j}\langle e_j,(f\otimes_H f)e_j\rangle_H=\sum_j \langle f,e_j\rangle^2_H=\Vert f\Vert^2_H$ for any orthonormal basis $(e_j)_j$ in $H$.
\end{proof}
\begin{appxlem}\label{N(t)}
For any trace class self-adjoint operator $C$, the following hold:\vspace{2mm}\\
 (i) Suppose $ai^{-\alpha}\le\lambda_i(C)\le Ai^{-\alpha}$ for $\alpha>1$ and $a,A\in(0,\infty)$. Then $t^{-1/\alpha}\lesssim \Cal{N}_C(t)\lesssim t^{-1/\alpha}.$\vspace{2mm}\\
 (ii) Suppose $be^{-\tau i}\le\lambda_i(C)\le Be^{-\tau i}$ for $\tau>0$ and $b,B\in(0,\infty)$. Then
$\log\frac{1}{t}\lesssim\Cal{N}_C(t)\lesssim\log\frac{1}{t}.$
\end{appxlem}
\begin{proof}
\emph{(i)} Define $\lambda_i:=\lambda_i(C)$. We have
\begin{eqnarray}
\Cal{N}_C(t)&{}={}&\text{tr}\left((C+tI)^{-1}C\right)=\sum_{i\ge 1}\frac{\lambda_i}{\lambda_i+t}\le\sum_{i\ge 1}\frac{Ai^{-\alpha}}{a i^{-\alpha}+t}=\frac{A}{a}\sum_{i\ge 1}\frac{i^{-\alpha}}{i^{-\alpha}+ta^{-1}}\nonumber\\
&{}\le{}& \frac{A}{a}\int^\infty_0 \frac{x^{-\alpha}}{x^{-\alpha}+ta^{-1}}\,dx\le \frac{A}{a}\left(\frac{a}{t}\right)^{1/\alpha}\int_0^\infty\frac{1}{1+x^\alpha}dx,\nonumber
\end{eqnarray}
where clearly the integral is finite for $\alpha>1$, thereby yielding $\Cal{N}_C(t)\lesssim t^{-1/\alpha}$. \emph{(ii)} follows by carrying out a similar calculation as in \emph{(i)}.
\end{proof}
\begin{appxlem}\label{lem:supp}
Let $X$ and $Y$ be $H$-valued random elements where $H$ is a separable Hilbert space. Then,
\begin{equation} \left(\sqrt{\bb{E}\norm{X}^2_H}-\sqrt{\bb{E}\norm{Y}^2_H}\right)^2\le\bb{E}\norm{X-Y}^2_H\le 2\bb{E}\norm{X}^2_H+2\bb{E}\norm{Y}^2_H.
\label{Eq:supp}
\end{equation}
\end{appxlem}
\begin{proof}
Note that $\bb{E}\norm{X-Y}^2_H=\bb{E}\norm{X}^2_H+\bb{E}\norm{Y}^2_H-2\bb{E}\inner{X}{Y}_H$. Using $-2\bb{E}\inner{X}{Y}_H\le \bb{E}\norm{X}^2_H+\bb{E}\norm{Y}^2_H$ yields the upper bound. Using $-\bb{E}\inner{X}{Y}_H\ge -\sqrt{\bb{E}\norm{X}^2_H}\sqrt{\bb{E}\norm{Y}^2_H}$ gives the lower bonud.
\end{proof}

\section{Sampling, Inclusion and Approximation Operators}\label{app:conc}
In this appendix, we present some technical results related to the properties of sampling, inclusion and approximation operators.
\subsection{Properties of the sampling operator}
The following result presents the properties of the sampling operator, $S$ and its adjoint. While these results are known in the literature (e.g., see \citealp{Smale-07}), we present it here for completeness. 
\begin{appxpro}
\label{pro:sampling}
Let $\Cal{H}$ be an RKHS of real-valued functions on a non-empty set $\Cal{X}$ with $k$ as the reproducing kernel. Define $S:\Cal{H}\rightarrow\bb{R}^n$, $f\mapsto\frac{1}{\sqrt{n}}(f(X_1),\ldots,f(X_n))^\top$ where $(X_i)_i\subset\Cal{X}$. Then the following hold:
\begin{itemize}
 \item[(i)] $S^*:\bb{R}^n\rightarrow\Cal{H}$, $\bm{\alpha}\mapsto\frac{1}{\sqrt{n}}\sum^n_{i=1}\alpha_i k(\cdot,X_i)$;
 \item[(ii)] $\sigh=\frac{n}{n-1}S^*\bm{H}_nS$ where $\sigh$ is defined in \eqref{Eq:emp-sig};
 \item[(iii)] $\bm{K}=nSS^*$, where $[\bm{K}]_{ij}=k(X_i,X_j)$.
\end{itemize}
\end{appxpro}
\begin{proof}
\emph{(i)} For any $g\in\Cal{H}$ and $\bm{\alpha}\in\bb{R}^n$, we have
$\langle S^*\bm{\alpha},g\rangle_\Cal{H}=\langle \bm{\alpha},Sg\rangle_2=\frac{1}{\sqrt{n}}\sum^n_{i=1}\alpha_i g(X_i)=\left\langle\frac{1}{\sqrt{n}}\sum^n_{i=1}\alpha_i k(\cdot,X_i),g\right\rangle_\Cal{H},$
where the last equality follows from the reproducing property.\vspace{1mm}\\
\emph{(ii)} For any $f\in\Cal{H}$, 
\begin{eqnarray}
\langle f,\sigh f\rangle_\Cal{H}&{}={}&\frac{1}{2n(n-1)}\sum^n_{i\ne j}\left(f(X_i)-f(X_j)\right)^2\nonumber\\
&{}={}&\frac{1}{n}\sum^n_{i=1}f^2(X_i)-\frac{1}{n(n-1)}\sum_{i\ne j}f(X_i)f(X_j)\nonumber\\
&{}={}&\frac{1}{n-1}\sum^n_{i=1}f^2(X_i)-\frac{1}{n-1}\left(\frac{1}{\sqrt{n}}\sum^n_{i=1}f(X_i)\right)^2\label{Eq:temmm}\\
&{}={}&\frac{n}{n-1}
\langle Sf,Sf\rangle_2-\frac{1}{n-1}\langle \bm{1}_n,Sf\rangle^2_2\nonumber\\
&{}={}&\frac{n}{n-1}\langle f,S^*Sf\rangle_\Cal{H}-\frac{1}{n-1}\langle S^*\bm{1}_n,f\rangle^2_\Cal{H}\nonumber\\
&{}={}&\frac{n}{n-1}\langle f,S^*Sf\rangle_\Cal{H}-\frac{1}{n-1}\langle f,S^*(\bm{1}_n\otimes_2 \bm{1}_n)Sf\rangle_\Cal{H}\nonumber\\
&{}={}&\frac{n}{n-1}\langle f,S^*\bm{H}_n Sf\rangle_\Cal{H}.\nonumber
\end{eqnarray}
\emph{(iii)} For any $\bm{\alpha}\in\bb{R}^n$,
$$SS^*\bm{\alpha}=S\left(\frac{1}{\sqrt{n}}\sum^n_{i=1}\alpha_i k(\cdot,X_i)\right)=\frac{1}{\sqrt{n}}\sum^n_{i=1}\alpha_i Sk(\cdot,X_i)=\frac{1}{n}\bm{K}\bm{\alpha},$$
where in the second equality, we used the fact $S$ is a linear operator.
\end{proof}
\subsection{Properties of the inclusion operator}
The following result captures the properties of the inclusion operator $\id$. A variation of the result is known in the literature (e.g., see \citealp[Theorem 4.26]{Steinwart-08}).
\begin{appxpro}\label{pro:id}
Suppose $(A_1)$ holds. Define $\ide:\Cal{H}\rightarrow\lp,\,f\mapsto f-f_\bb{P}$, where $f_\bb{P}:=\int f(x)\,d\bb{P}(x)$. Then the following hold:\vspace{2mm}\\
 (i) $\ide^*:\lp\rightarrow\Cal{H},
 \,f\mapsto \int_\Cal{X} k(\cdot,x)f(x)\,d\bb{P}(x)-m_\bb{P}f_\bb{P}$ where $$m_\bb{P}:=\int_\Cal{X} k(\cdot,x)\,d\bb{P}(x).$$
 (ii) $\ide$ and $\ide^*$ are Hilbert-Schmidt.\vspace{2mm}\\
 (iii) $\Sigma=\ide^*\ide$ is trace-class, where $\Sigma$ is defined in \eqref{Eq:cov}.\vspace{2mm}\\
 (iv) $\ide\ide^*=\Upsilon-(1\ol 1)\Upsilon-\Upsilon(1\ol 1)+(1\ol 1)\Upsilon(1\ol 1)$ is trace-class where $\Upsilon:\lp\rightarrow\lp,\,f\mapsto\intx k(\cdot,x) f(x)\,d\bb{P}(x)$.
\end{appxpro}
\begin{proof}
\emph{(i)} For any $f\in\lp$ and $g\in\Cal{H}$, 
\begin{eqnarray}
\langle \id^*f,g\rangle_\Cal{H}&{}={}&\langle f,\id g\rangle_{\lp}=\int_\Cal{X} f(x)(\id g)(x)\,d\bb{P}(x)=\int_\Cal{X} f(x)[g(x)-g_\bb{P}]\,d\bb{P}(x) \nonumber\\
&{}={}& \int_\Cal{X} f(x)\langle k(\cdot,x),g\rangle_\Cal{H}\,d\bb{P}(x)-\langle m_\bb{P},g\rangle_\Cal{H}f_\bb{P}\nonumber\\
&{}={}&\left\langle \int_\Cal{X} k(\cdot,x)f(x)\,d\bb{P}(x),g\right\rangle_\Cal{H}-\langle m_\bb{P}f_\bb{P},g\rangle_\Cal{H}.\nonumber
\end{eqnarray}
Clearly $f_\bb{P}$ is well defined as for any $f\in \lp$, $f_\bb{P}\le \int |f(x)|\,d\Pb(x) \le \norm{f}_{\lp}<\infty$ and for $f\in \Hk$, $f_\Pb=\langle f,m_\Pb\rangle_\Hk\le\norm{f}_\Hk\int \sqrt{k(x,x)}\,d\Pb(x)<\infty$ and the result therefore follows.\vspace{1mm}\\ 
\emph{(ii)} For any orthonormal basis $(e_j)_j$ in $\Cal{H}$,
\begin{eqnarray}
\Vert \id\Vert^2_{\HS(\Cal{H},\lp)}&{}={}&\sum_j\Vert \id e_j\Vert^2_{\lp}=\sum_j \Vert e_j-e_{j,\bb{P}}\Vert^2_{\lp}=\sum_j \Vert e_j\Vert^2_{\lp}-e^2_{j,\bb{P}}\nonumber\\ 
&{}\le{}& \sum_j \Vert e_j\Vert^2_{\lp}=\sum_j \intx \langle e_j,k(\cdot,x)\rangle^2_\Cal{H}\,d\bb{P}(x)\nonumber\\
&{}\stackrel{(\star)}{=}{}&\intx \sum_j \langle e_j,k(\cdot,x)\rangle^2_\Cal{H}\,d\bb{P}(x)=\intx k(x,x)\,d\bb{P}(x)<\infty,\nonumber
\end{eqnarray}
where $(\star)$ follows from monotone convergence theorem. Since $\Vert \id\Vert_{\HS(\Cal{H},\lp)}=\Vert \id^*\Vert_{\HS(\lp,\Cal{H})}$, the result follows.
\vspace{1mm}\\
\emph{(iii)} For any $f\in\Cal{H}$, 
$(\id^*\id)f=\id^*(f-f_\bb{P})=\id^*f-\id^*f_\bb{P}=\id^*f$, where we use the fact that $\id^*f_{\bb{P}}=0$ since $f_\bb{P}$ is a constant function. By using the reproducing property,
\begin{eqnarray}
\id^*\id f=\id^*f&{}={}&\int_\Cal{X} f(x)k(\cdot,x)\,d\bb{P}(x)-m_\bb{P}f_\bb{P}\nonumber\\
&{}={}&\intx  k(\cdot,x)\langle k(\cdot,x),f\rangle_\Cal{H}\,d\bb{P}-m_\bb{P}\langle m_\bb{P},f\rangle_\Cal{H}\nonumber\\
&{}={}&\intx (k(\cdot,x)\oh k(\cdot,x))f \,d\bb{P}(x)-(m_\bb{P}\oh m_\bb{P})f=\Sigma f\nonumber
\end{eqnarray}
and the result follows. Since $\Vert \id\Vert^2_{\HS(\Cal{H},\lp)}=\Vert\id^*\id\Vert_{\Tr(\Cal{H})}$, $\Sigma$ is trace-class.\vspace{1mm}\\
\emph{(iv)} For any $f\in\lp$,
\begin{eqnarray}
(\id\id^*)f&{}={}&\id(\id^*f)=\id\left(\intx k(\cdot,x)f(x)\,d\bb{P}(x)-m_\bb{P}f_\bb{P}\right)\nonumber\\
&{}={}& \intx k(\cdot,x)f(x)\,d\bb{P}(x)-m_\bb{P}f_\bb{P}- \intx\intx k(y,x)f(x)\,d\bb{P}(x)\,d\bb{P}(y)\nonumber\\
&{}{}&\qquad\qquad+f_\bb{P}\intx\intx k(y,x)\,d\bb{P}(x)\,d\bb{P}(y)\nonumber\\
&{}={}&\Upsilon f-\Upsilon 1\langle 1,f\rangle_{\lp}-1\langle \Upsilon 1,f\rangle_{\lp}+1\langle 1,\Upsilon 1\rangle_{\lp}\langle 1,f\rangle_{\lp}\nonumber\\
&{}={}&\Upsilon f-\Upsilon (1\ol 1)f-(1\ol 1)\Upsilon f+(1\ol 1)\Upsilon (1\ol 1)f\nonumber
\end{eqnarray}
and the result follows, where in the last line we use the fact that $\Upsilon$ is self-adjoint, which follows from \citep[Theorem 4.27]{Steinwart-08}. Since $\Vert \id^*\Vert^2_{\HS(\lp,\Cal{H})}=\Vert \id\id^*\Vert_{\Tr(\lp)}$, it follows that $\id\id^*$ is trace-class.
\end{proof}
The following result presents a representation for $\Upsilon$ if $k$ satisfies $(A_4)$.
\begin{appxlem}\label{lem:theta}
Suppose $(A_4)$ holds. Then $\Upsilon=\intt \vp(\cdot,\theta)\ol \vp(\cdot,\theta)\,d\Lambda(\theta)$.
\end{appxlem}
\begin{proof}
Since $k(x,y)=\intt \vp(x,\theta) \vp(y,\theta)\,d\Lambda(\theta)$, for any $f\in\lp$,
\begin{eqnarray}
\Upsilon f&{}={}&\intx k(\cdot,x)f(x)\,d\bb{P}(x)=\intx \intt \vp(\cdot,\theta)\vp(x,\theta)\,d\Lambda(\theta)f(x)\,d\bb{P}(x) \nonumber\\
&{}\stackrel{(*)}{=}{}&\intt \vp(\cdot,\theta)\left(\intx \vp(x,\theta)f(x)\,d\bb{P}(x)\right)\,d\Lambda(\theta)\nonumber
\end{eqnarray}
\begin{eqnarray}
&{}={}&\intt \vp(\cdot,\theta)\langle \vp(\cdot,\theta),f\rangle_{\lp}\,d\Lambda(\theta)
= \intt \left(\vp(\cdot,\theta)\ol \vp(\cdot,\theta)\right)f\,d\Lambda(\theta)\nonumber\\
&{}={}&\left(\intt \vp(\cdot,\theta)\ol \vp(\cdot,\theta)\,d\Lambda(\theta)\right)f,\nonumber
\end{eqnarray}
where Fubini's theorem is applied in $(*)$.
\end{proof}

\subsection{Properties of the approximation operator}
The following result presents the properties of the approximation operator, $\frak{A}$.
\begin{appxpro}\label{pro:approx}
Define $$\frak{A}:\Hk_m\rightarrow\lp,\,f=\sum_{i=1}^m\beta_i\vp_i\mapsto\sum^m_{i=1}\beta_i(\vp_i-\vp_{i,\bb{P}})$$ where $\vp_{i,\bb{P}}:=\intx \vp_i(x)\,d\bb{P}(x)$ and $\sup_{x\in\X}|\vp_i(x)|\le\sqrt{\frac{\kappa}{m}}$
for all $i\in[m]$ with $\kappa<\infty$. Then the following hold:\vspace{2mm}\\
 (i) $\frak{A}^*:\lp\rightarrow\Hk_m$,\,\,$f\mapsto\sum_{i=1}^m\left(\inner{f}{\vp_i}_{\lp}-f_\Pb\vp_{i,\Pb}\right)\vp_i$.\vspace{2mm}\\
 (ii) $\frak{A}$ and $\frak{A}^*$ are Hilbert-Schmidt.\vspace{2mm}\\
 (iii) $\Sigma_m=\frak{A}^*\frak{A}$ is trace-class.\vspace{2mm}\\
 (iv) $\frak{A}\frak{A}^*=\Pi -(1\ol 1)\Pi -\Pi(1\ol 1)+(1\ol 1)\Pi (1\ol 1)$ is trace-class where $\Pi:=\sum^m_{i=1}\vp_i\ol\vp_i:\lp\rightarrow\lp$.
\end{appxpro}
\begin{proof} The proof is similar to that of Proposition~\ref{pro:id}.\vspace{1mm}\\
$(i)$ For any $g=\sum_{i=1}^m\beta_i\vp_i\in\Hk_m$ and $f\in\lp$,
\begin{eqnarray*}
    \inner{\frak{A}^*f}{g}_{\Hk_m}&{}={}&\inner{f}{\frak{A} g}_{\lp}=\int_\X\left(\sum^m_{i=1}\beta_i(\vp_i(x)-\vp_{i,\bb{P}})\right)f(x)d\Pb(x)\\
    &{}={}&\sum_{i=1}^m\beta_i(\inner{f}{\vp_i}_{\lp}-f_\Pb\vp_{i,\Pb}),\nonumber
\end{eqnarray*}
and the result follows from the definition of $\inner{\cdot}{\cdot}_{\Hk_m}$.\vspace{1mm}
\\
\emph{(ii)} For any orthonormal basis $(e_j)_j$ in $\lp$,
\begin{eqnarray}
\Vert\frak{A}^*\Vert^2_{\HS(\lp,\Hk_m)}
&{}={}&\sum_j\Vert\frak{A}^*e_j\Vert^2_{\Hk_m}= \sum_j\sum^m_{i=1}\left(\langle e_j,\vp_i\rangle_{\lp}-e_{j,\bb{P}}\vp_{i,\bb{P}}\right)^2\nonumber\\
&{}={}&\sum_j\sum^m_{i=1}\langle e_j,\vp_i\rangle^2_{\lp}+e^2_{j,\bb{P}}\vp^2_{i,\bb{P}}-2e_{j,\bb{P}}\vp_{i,\bb{P}}\langle e_j,\vp_i\rangle_{\lp}\nonumber\\
&{}={}&\sum^m_{i=1}\Vert\vp_i\Vert^2_{\lp}+\sum^m_{i=1}\vp^2_{i,\bb{P}}\sum_j \langle e_j,1\rangle^2_{\lp}\nonumber\\
&{}{}&\qquad\qquad-\sum^m_{i=1}\vp_{i,\bb{P}}\sum_j\langle e_j,(\vp_i\ol 1+1\ol \vp_i)e_j\rangle_{\lp}\nonumber\\
&{}={}&\sum^m_{i=1}\Vert\vp_i\Vert^2_{\lp}+\sum^m_{i=1}\vp^2_{i,\bb{P}}-2\sum^m_{i=1}\vp_{i,\bb{P}}\langle \vp_i,1\rangle_{\lp}\nonumber
\le \sum^m_{i=1}\Vert\vp_i\Vert^2_{\lp}\le\kappa<\infty,\nonumber
\end{eqnarray}
and so $\frak{A}$ and $\frak{A}^*$ are Hilbert-Schmidt.\vspace{1mm}
\\
\emph{(iii)} For any $f=\sum_{i=1}^m\beta_i\vp_i\in\Hk_m$, 
\begin{eqnarray}
\frak{A}^*\frak{A} f&{}={}&\frak{A}^*\left(\sum^m_{i=1}\beta_i(\vp_i-\vp_{i,\Pb})\right)=\sum^m_{i=1}\beta_i\frak{A}^*(\vp_i-\vp_{i,\Pb})\hspace{2cm}\nonumber
\\
&{}={}&\sum_{i=1}^m\beta_i\sum_{j=1}^m(\inner{\vp_i}{\vp_j}_{\lp}-\vp_{i,\Pb}\vp_{j,\Pb})\vp_j\nonumber\\
&{}={}&\sum_{j=1}^m\inner{\sum_{i=1}^m\beta_i\vp_i}{\vp_j}_{\lp}\vp_j-\left(\intx\sum_{i=1}^m\beta_i\vp_i(x)d\Pb(x)\right)\left(\intx\sum_{j=1}^m\vp_j(x)\vp_jd\Pb(x)\right)\nonumber\\
&{}={}&\intx\left(\sum_{i=1}^m\beta_i\vp_i(x)\right)\left(\sum_{j=1}^m\vp_j(x)\vp_j\right)\,d\Pb(x)\nonumber\\
&{}{}&\qquad\qquad-\left(\intx f(x)\,d\Pb(x)\right)\left(\intx k_m(\cdot,x)\,d\Pb(x)\right)\nonumber
\\
&{}={}&\intx f(x)k_m(\cdot,x)\,d\Pb(x)-\left(\intx f(x)\,d\Pb(x)\right)\left(\intx k_m(\cdot,x)\,d\Pb(x)\right)
=\Sigma_mf.\nonumber
\end{eqnarray}
That $\Sigma_m$ is trace class is implied by \emph{(ii)}.
\\
\emph{(iv)}  For any $f\in\lp$,
\begin{eqnarray}
\frak{A}\frak{A}^*f
&{}={}& \sum^m_{i=1}(\langle f,\vp_i\rangle_{\lp}-f_\bb{P}\vp_{i,\bb{P}})(\vp_i-\vp_{i,\Pb})\nonumber\\
&{}={}&\sum^m_{i=1}(\langle f,\vp_i\rangle_{\lp}-\langle f,1\rangle_{\lp}\langle \vp_i,1\rangle_{\lp})(\vp_i-\langle \vp_i,1\rangle_{\lp})\nonumber\\
&{}={}&\Pi f-\langle \Pi 1,f\rangle_{\lp}-\Pi(1\ol 1)f+\langle (1\ol 1)\Pi 1,f\rangle_{\lp}\nonumber\\
&{}={}&\Pi f-(1\ol 1)\Pi f-\Pi(1\ol 1)f+(1\ol 1)\Pi (1\ol 1)f\nonumber
\end{eqnarray}
and the result follows. $\frak{A}\frak{A}^*$ is trace-class since $\frak{A}^*$ is Hilbert-Schmidt.
\end{proof}

\section{Supplementary Results}\label{app:sup}
In this appendix, we collect Bernstein's inequality for Hilbert-valued random elements (quoted from \citealp{Yurinsky-95}) and Tropp's inequality for operator-valued random elements (quoted from \citealp[Theorem A.1]{Rudi-13}), that are used to prove the results of this paper. Based on these two results, Theorem~\ref{thm:bernstein U-stat} presents a Bernstein-type inequality for the operator and Hilbert-Schmidt norms of a operator-valued U-statistics.
\begin{appxthm}[Bernstein's inequality in separable Hilbert spaces]\label{thm:bernstein}
Let $(\Omega,\Cal{A},P)$ be a probability space, $H$ be a separable Hilbert space, $B>0$ and $\theta>0$. Furthermore, let $\xi_1,\ldots,\xi_n:\Omega\rightarrow H$ be zero mean i.i.d.~random variables satisfying 
\begin{equation*}\bb{E}\Vert \xi_1\Vert^r_{H}\le \frac{r!}{2}\theta^2B^{r-2},\,\,\forall\,\,r>2.\nonumber
\end{equation*}
Then for any $0<\delta<1$,
$$P^n\left\{(\xi_i)^n_{i=1}:\left\Vert\frac{1}{n} \sum^n_{i=1}\xi_i\right\Vert_{H}\ge \frac{2B\log\frac{2}{\delta}}{n}+\sqrt{\frac{2\theta^2\log\frac{2}{\delta}}{n}}\right\}\le \delta.$$
\end{appxthm}
%
\begin{appxthm}[Tropp's inequality for operators]\label{thm: tropp}
Let $(Z_i)_{i=1}^n$ be independent copies of the random variable $Z$ with law $P$ taking values in the space of bounded self-adjoint operators for a separable Hilbert space $H$. Suppose there exists $S\in\Cal{L}^2(H)$ such that $\E[(Z-\E[Z])^2]\preceq S$ and $0<M<\infty$ such that $\norm{Z}_{\Cal{L}^\infty(H)}\le M$ almost everywhere.  Let $d:=\frac{\norm{S}_{\Cal{L}^1(H)}}{\norm{S}_{\Cal{L}^\infty(H)}}$ and $\sigma^2:=\norm{S}_{\Cal{L}^\infty(H)}$.  Then for $0<\delta\le d$, 
$$P^n\left\{(Z_i)^n_{i=1}:\norm{\frac{1}{n}\sum_{i=1}^nZ_i-\E[Z]}_{\Cal{L}^\infty(H)}\ge\frac{\beta M}{n}+\sqrt{\frac{3\beta\sigma^2}{n}}\right\}\le\delta,$$
where $\beta:=\frac{2}{3}\log\frac{4d}{\delta}$.
\end{appxthm}{}
%
%
\begin{appxthm}\label{thm:bernstein U-stat}
Let $(\Cal{X},P)$ be a measurable space and $Z:\Cal{X}\times\Cal{X}\rightarrow \Cal{L}^2(H)$ with $Z(x,y)=Z(y,x)$ for all $x,y\in H$, where $H$ is a separable Hilbert space. Let $$D=\int\int Z(x,y)\,dP(x)\,dP(y)$$ with 
$$\widehat{D}=\frac{1}{n(n-1)}\sum^n_{i\ne j}Z(X_i,X_j)$$ being its $U$-statistic estimator, where $(X_i)^n_{i=1}\stackrel{i.i.d.}{\sim}P$ and $n\ge 2$. Define $\psi(x)=\bb{E}_Y\left[Z(x,Y)\right]$ and let $\sup_{x,y\in\Cal{X}}\Vert Z(x,y)\Vert_{\Cal{L}^2(H)}\le M$. Then the following hold:\vspace{2mm}\\
 (i) Suppose $Z:\Cal{X}\times\Cal{X}\rightarrow \Cal{S}(H)$, where $\Cal{S}(H)$ is the space of self-adjoint Hilbert-Schmidt operators on $H$, $\E[(\psi(X)-D)^2]\preceq S$, $\sigma^2:=\Vert S\Vert_{\Cal{L}^\infty(H)}$, and $\sup_{x\in\Cal{X}}\Vert \psi(x)\Vert_{\Cal{L}^\infty(H)}\le R$. Then for $0< \delta\le d$,
 $$P^n\left\{(X_i)^n_{i=1}:\left\Vert\widehat{D}-D\right\Vert_{\Cal{L}^\infty(H)}\le\frac{2\beta R}{n}+\sqrt{\frac{12\beta\sigma^2}{n}}+\frac{8M\log\frac{3}{\delta}}{n}\right\}\ge 1-2\delta,$$
where $\beta:=\frac{2}{3}\log\frac{4d}{\delta}$ and $d:=\frac{\Vert S\Vert_{\Cal{L}^1(H)}}{\Vert S\Vert_{\Cal{L}^\infty(H)}}$.\vspace{2mm}\\
(ii) Suppose $\bb{E}\norm{\psi(X)-D}^2_{\Cal{L}^2(H)}\le \sigma^2_1$. Then 
$$P^n\left\{(X_i)^n_{i=1}:\norm{\widehat{D}-D}_{\Cal{L}^2(H)}\le\frac{16M\log\frac{3}{\delta}}{n}+\sqrt{\frac{8\sigma^2_1\log\frac{2}{\delta}}{n}}
\right\}\ge 1-2\delta.$$
\end{appxthm}
\begin{proof}
$(i)$ The Hoeffding decomposition of $\widehat{D}$ yields
\begin{eqnarray*}
\widehat{D}-D&{}={}&2\left[\frac{1}{n}\sum^n_{i=1}\psi(X_i)-D\right]
+\left[\frac{1}{n(n-1)}\sum^n_{i\ne j}Z(X_i,X_j)-\psi(X_i)-\psi(X_j)+D\right],\nonumber
\end{eqnarray*}
which implies
\begin{eqnarray}
\left\Vert\widehat{D}-D\right\Vert_{\Cal{L}^\infty(H)} &{}\le{}&\left\Vert\frac{1}{n(n-1)}\sum^n_{i\ne j}Z(X_i,X_j)-\psi(X_i)-\psi(X_j)+D\right\Vert_{\Cal{L}^\infty(H)}\nonumber\\
&{}{}&\qquad+2\left\Vert\frac{1}{n}\sum^n_{i=1}\psi(X_i)-D\right\Vert_{\Cal{L}^\infty(H)}.\label{Eq:Hoeff}
\end{eqnarray}
The first term can be bounded by applying Theorem \ref{thm: tropp} since $\E[\psi(X)]=D$. We now bound the second term as follows. Define $h(X_i,X_j):=Z(X_i,X_j)-\psi(X_i)-\psi(X_j)+D$. Applying Markov's inequality to the second term, we obtain that for any $\epsilon>0$ and $t>0$,
\begin{eqnarray}&{}{}&P^n\left\{(X_i)^n_{i=1}:\left\Vert \frac{1}{n(n-1)}\sum^n_{i\ne j}h(X_i,X_j)\right\Vert_{\Cal{L}^\infty(H)}\ge \epsilon\right\}\nonumber\\
&{}{}&\qquad\quad\qquad\qquad\le e^{-t\epsilon}\E\exp\left\Vert t'\sum^n_{i\ne j}h(X_i,X_j)\right\Vert_{\Cal{L}^\infty(H)},\label{Eq:Markov}
\end{eqnarray}
where $t':=\frac{t}{n(n-1)}$. Consider
\begin{eqnarray*}
&{}{}&\E\exp\left\Vert t'\sum^n_{i\ne j}h(X_i,X_j)\right\Vert_{\Cal{L}^\infty(H)}\nonumber\\
&{}={}&\E\exp\left\Vert t'\sum^n_{i\ne j}\left[Z(X_i,X_j)-\E_{X'_j}Z(X_i,X'_j)-\E_{X'_i}Z(X_j,X'_i)+\E_{X'_i,X'_j}Z(X'_i,X'_j)\right] \right\Vert_{\Cal{L}^\infty(H)}\nonumber\\
&{}={}& 
\E\exp\left\Vert t'\E_{(X'_i)^n_{i=1}|(X_i)^n_{i=1}}\sum^n_{i\ne j}\left[Z(X_i,X_j)-Z(X_i,X'_j)-Z(X_j,X'_i)+Z(X'_i,X'_j)\right] \right\Vert_{\Cal{L}^\infty(H)}\nonumber\\
&{}\le{}& 
\E\exp \left[\E_{(X'_i)^n_{i=1}|(X_i)^n_{i=1}}\left\Vert t'\sum^n_{i\ne j}\left[Z(X_i,X_j)-Z(X_i,X'_j)-Z(X'_i,X_j)+Z(X'_i,X'_j)\right] \right\Vert_{\Cal{L}^\infty(H)}\right]\nonumber\\
&{}\stackrel{(*)}{=}{}& 
\E\exp \left[\E_{(X'_i)^n_{i=1}|(X_i)^n_{i=1}}\left\Vert t'\sum^n_{i\ne j}\left(\delta_{X_i}-\delta_{X'_i}\right)\left(\delta_{X_j}-\delta_{X'_j}\right)
Z\right\Vert_{\Cal{L}^\infty(H)}\right]\nonumber\\
&{}={}& 
\E \exp\left[ \E_{\epsilon^{(1)}}\E_{\epsilon^{(2)}} \E_{(X'_i)^n_{i=1}|(X_i)^n_{i=1}} \left\Vert t'\sum^n_{i\ne j}\epsilon^{(1)}_i\left(\delta_{X_i}-\delta_{X'_i}\right)\epsilon^{(2)}_j\left(\delta_{X_j}-\delta_{X'_j}\right)Z \right\Vert_{\Cal{L}^\infty(H)}\right]\nonumber\\
&{}\stackrel{(\dagger)}{\le}{}& 
\E \exp \left[\E_{\epsilon^{(1)}}\E_{\epsilon^{(2)}}\left\Vert t'\sum^n_{i\ne j}\epsilon^{(1)}_i\epsilon^{(2)}_j\left(\delta_{X_i}-\delta_{X'_i}\right)\left(\delta_{X_j}-\delta_{X'_j}\right)Z \right\Vert_{\Cal{L}^\infty(H)}\right],\nonumber
\end{eqnarray*}
where $(\epsilon^{(1)}_i)_i$ and $(\epsilon^{(2)}_i)_i$ are independent Rademacher random variables. In $(*)$, $\delta_{x}$ denotes a Dirac measure supported on $x$ and we use the notation $Qf:=\int f(x)\,dQ(x)$ with $Q$ being a Dirac measure. In $(\dagger)$, the expectation is jointly over $(X_i,X'_i)^n_{i=1}$ which is obtained through an application of Jensen's inequality. Therefore, 
\begin{eqnarray*}
\E\exp\left\Vert t'\sum^n_{i\ne j}h(X_i,X_j)\right\Vert_{\Cal{L}^\infty(H)}\le \E \exp \left[\circled{\tiny{A}}+\circled{\tiny{B}} +\circled{\tiny{C}}+\circled{\tiny{D}}\right],
\end{eqnarray*}
where $$\circled{\tiny{A}}:=\E_{\epsilon^{(1)}}\E_{\epsilon^{(2)}}\left\Vert t'\sum^n_{i\ne j}\epsilon^{(1)}_i\epsilon^{(2)}_jZ(X_i,X_j)\right\Vert_{\Cal{L}^\infty(H)},\quad\circled{\tiny{B}}:=\E_{\epsilon^{(1)}}\E_{\epsilon^{(2)}}\left\Vert t'\sum^n_{i\ne j}\epsilon^{(1)}_i\epsilon^{(2)}_jZ(X_i,X'_j)\right\Vert_{\Cal{L}^\infty(H)},$$ 
$$\circled{\tiny{C}}:=\E_{\epsilon^{(1)}}\E_{\epsilon^{(2)}}\left\Vert t'\sum^n_{i\ne j}\epsilon^{(1)}_i\epsilon^{(2)}_jZ(X'_i,X_j)\right\Vert_{\Cal{L}^\infty(H)}$$
and 
$$\circled{\tiny{D}}:=\E_{\epsilon^{(1)}}\E_{\epsilon^{(2)}}\left\Vert t'\sum^n_{i\ne j}\epsilon^{(1)}_i\epsilon^{(2)}_jZ(X'_i,X'_j)\right\Vert_{\Cal{L}^\infty(H)}.$$
Since $(X_i)^n_{i=1}$ and $(X'_i)^n_{i=1}$ are i.i.d., we have
\begin{eqnarray}
\E\exp\left\Vert t'\sum^n_{i\ne j}h(X_i,X_j)\right\Vert_{\Cal{L}^\infty(H)}
&{}\le{}& \E \exp \left[4t' \E_{\epsilon^{(1)}}\E_{\epsilon^{(2)}}\left\Vert \sum^n_{i\ne j}\epsilon^{(1)}_i\epsilon^{(2)}_jZ(X_i,X_j)\right\Vert_{\Cal{L}^\infty(H)}\right]\label{Eq:U-stat proof 1}\\
&{}\le{}&\E \exp \left[4t' \E_{\epsilon^{(1)}}\E_{\epsilon^{(2)}}\left\Vert \sum^n_{i\ne j}\epsilon^{(1)}_i\epsilon^{(2)}_jZ(X_i,X_j)\right\Vert_{\Cal{L}^2(H)}\right]\nonumber\\
&{}\le{}& \E \exp \left[4t' \sqrt{\E_{\epsilon^{(1)}}\E_{\epsilon^{(2)}}\left\Vert \sum^n_{i\ne j}\epsilon^{(1)}_i\epsilon^{(2)}_jZ(X_i,X_j)\right\Vert^2_{\Cal{L}^2(H)}}\right],\nonumber
\end{eqnarray}
where the last inequality follows from Jensen's inequality. We will now bound 
\begin{eqnarray*}
&{}{}&\E_{\epsilon^{(1)}}\E_{\epsilon^{(2)}}\left\Vert \sum^n_{i\ne j}\epsilon^{(1)}_i\epsilon^{(2)}_jZ(X_i,X_j)\right\Vert^2_{\Cal{L}^2(H)}\nonumber\\
&{}{}&\qquad=\E_{\epsilon^{(1)}}\E_{\epsilon^{(2)}}\sum^n_{i\ne j}\sum^n_{k\ne l}\epsilon^{(1)}_i\epsilon^{(2)}_j\epsilon^{(1)}_k\epsilon^{(2)}_l\langle Z(X_i,X_j),Z(X_k,X_l)\rangle_{\Cal{L}^2(H)}.\nonumber
\end{eqnarray*}
We consider the following cases.\vspace{1mm}\\
\textbf{Case 1:} $i=k, j=l$\vspace{1mm}\\
\begin{eqnarray*}
\E_{\epsilon^{(1)}}\E_{\epsilon^{(2)}}\left\Vert \sum^n_{i\ne j}\epsilon^{(1)}_i\epsilon^{(2)}_jZ(X_i,X_j)\right\Vert^2_{\Cal{L}^2(H)}
=\sum^n_{i\ne j}\Vert Z(X_i,X_j)\Vert^2_{\Cal{L}^2(H)}
\le n(n-1)M^2.\nonumber
\end{eqnarray*}
\textbf{Case 2:} $i=k,j\ne l$
\begin{eqnarray*}
\E_{\epsilon^{(1)}}\E_{\epsilon^{(2)}}\left\Vert \sum^n_{i\ne j}\epsilon^{(1)}_i\epsilon^{(2)}_jZ(X_i,X_j)\right\Vert^2_{\Cal{L}^2(H)}
=\E_{\epsilon^{(2)}}\sum^n_{i\ne j\ne l}\epsilon^{(2)}_j\epsilon^{(2)}_l\langle Z(X_i,X_j),Z(X_i,X_l)\rangle_{\Cal{L}^2(H)}=0.\nonumber
\end{eqnarray*}
\textbf{Case 3:} $i\ne k,j=l$
\begin{eqnarray*}
\E_{\epsilon^{(1)}}\E_{\epsilon^{(2)}}\left\Vert \sum^n_{i\ne j}\epsilon^{(1)}_i\epsilon^{(2)}_jZ(X_i,X_j)\right\Vert^2_{\Cal{L}^2(H)}
=\E_{\epsilon^{(1)}}\sum^n_{i\ne j\ne k}\epsilon^{(1)}_i\epsilon^{(1)}_k\langle Z(X_i,X_j),Z(X_k,X_j)\rangle_{\Cal{L}^2(H)}=0.\nonumber
\end{eqnarray*}
\textbf{Case 4:} $i\ne k, j\ne l$
\begin{eqnarray*}
\E_{\epsilon^{(1)}}\E_{\epsilon^{(2)}}\left\Vert \sum^n_{i\ne j}\epsilon^{(1)}_i\epsilon^{(2)}_jZ(X_i,X_j)\right\Vert^2_{\Cal{L}^2(H)}
=0.\nonumber
\end{eqnarray*}
Therefore,
\begin{eqnarray*}
\E\exp\left\Vert t'\sum^n_{i\ne j}h(X_i,X_j)\right\Vert_{\Cal{L}^\infty(H)}\le \exp \left[\frac{4tM}{\sqrt{n(n-1)}}\right]\le \exp\left[\frac{8M t}{n}\right]\nonumber
\end{eqnarray*}
for $n\ge 2$ as $n-1\ge \frac{n}{4}$. Using this in \eqref{Eq:Markov} and choosing 
$t=\frac{n}{8M}$, we obtain
\begin{eqnarray*}P^n\left\{(X_i)^n_{i=1}:\left\Vert \frac{1}{n(n-1)}\sum^n_{i\ne j}h(X_i,X_j)\right\Vert_{\Cal{L}^\infty(H)}\ge \epsilon\right\}\le 3\exp\left(-\frac{n\epsilon}{8M}\right),\nonumber
\end{eqnarray*}
which is equivalent to
\begin{eqnarray}P^n\left\{(X_i)^n_{i=1}:\left\Vert \frac{1}{n(n-1)}\sum^n_{i\ne j}h(X_i,X_j)\right\Vert_{\Cal{L}^\infty(H)}\ge \frac{8M}{n}\log\frac{3}{\delta}\right\}\le \delta.\label{Eq:U-stat proof 2}
\end{eqnarray}
Combining \eqref{Eq:U-stat proof 2} with the bound on the first term in \eqref{Eq:Hoeff} yields the result.\vspace{1mm}
\\
$(ii)$ As in $(i)$, we first write 
\begin{eqnarray}\norm{\widehat{D}-D}_{\Cal{L}^2(H)}&{}\le{}&\norm{\frac{1}{n(n-1)}\sum_{i\neq j}Z(X_i,X_j)-\psi(X_i)-\psi(X_j)+D}_{\Cal{L}^2(H)}\nonumber\\
&{}{}&\qquad\qquad+2\norm{\frac{1}{n}\sum_{i=1}^n\psi(X_i)-D}_{\Cal{L}^2(H)}.\label{Eq:fff}\end{eqnarray}
The first term in \eqref{Eq:fff} is bounded through an application of Theorem~\ref{thm:bernstein}. For the second term, we replicate the analysis between (\ref{Eq:Markov}) and (\ref{Eq:U-stat proof 1}) with the operator norm being replaced by the Hilbert-Schmidt norm. Since the analysis between (\ref{Eq:U-stat proof 1}) and (\ref{Eq:U-stat proof 2}) anyway relies only on the Hilbert-Schmidt norm, the result follows by employing that $\log\frac{2}{\delta}<\log\frac{3}{\delta}$.
\end{proof}

\end{appendices}


\end{document}